\documentclass[11pt]{article}
\usepackage[utf8]{inputenc} 
\usepackage[T1]{fontenc}    
\usepackage{wrapfig}

\usepackage{setspace}
\usepackage{cprotect}
\usepackage{amsmath,amssymb,amsthm}
\usepackage[noend]{algorithmic}
\usepackage[ruled,vlined]{algorithm2e}
\usepackage{hyperref}
\usepackage{fullpage}
\usepackage{makeidx}
\usepackage{enumerate}
\usepackage{graphicx,float,psfrag,epsfig}
\usepackage{epstopdf}
\usepackage{color}
\usepackage{enumitem}
\usepackage{caption}
\usepackage{bigints}
\usepackage{mathtools}
\usepackage[mathscr]{euscript}
\usepackage{natbib}

\newcommand{\ip}[1]{\langle {#1} \rangle }

\usepackage{amsmath}
\usepackage{amssymb}
\usepackage{mathtools}
\usepackage{amsthm}
\usepackage{amsfonts}
\usepackage{soul}
\usepackage{csquotes}
\usepackage{amsmath}
\usepackage{amsthm}
\usepackage{amsfonts}
\usepackage{hyperref}
\usepackage{comment}
\usepackage{graphicx}
\usepackage{xcolor}
\usepackage{parskip}
\usepackage{hyperref}       
\usepackage{url}            
\usepackage{booktabs}       
\usepackage{amsfonts}       
\usepackage{nicefrac}       
\usepackage{microtype}      
\usepackage{xcolor}         
\usepackage{microtype}
\usepackage{graphicx}
\usepackage{subfigure}  
\usepackage{amsfonts}
\usepackage{amsmath}
\usepackage{amssymb}
\usepackage{bbm}
\usepackage{multirow}
\usepackage{mathtools}
\usepackage{relsize}

\newcommand{\clipc}{\textup{clip}_c}
\newcommand{\clip}{C_{\textup{clip}}}

\newcommand{\diff}{\mathrm{d}}

\newcommand{\erf}{\mathrm{erf}}

\def\P{\mathbb{P}}
\def\Cov{\mathrm{Cov}}
\def\Var{\mathrm{Var}}

\def\tr{\mathrm{tr}}

\def\R{\mathbb{R}}

\def\cF{\mathcal{F}}

\def\cK{\mathcal{K}}

\newcommand{\opnorm}[1]{\left\lVert#1\right\rVert_{\textup{op}}}

\def\b0{{0}}

\def\>{\rangle}

\newcommand{\E}{\mathbb{E}}

\newcommand{\norm}[1]{\left\|#1\right\|}

\newcommand{\evmax}[1]{\lambda_{\rm max}\left(#1\right)}
\newcommand{\evmin}[1]{\lambda_{\rm min}\left(#1\right)}

\def\op{\mathrm{op}}

\def\tr{\mathop{\rm tr}\nolimits}

\def\argmax{\mathop{\rm arg\,max}\limits}

\def\min{\mathop{\rm min}\nolimits}
\def\max{\mathop{\rm max}\nolimits}

\def\tdeta{\tilde{\eta}}

\numberwithin{equation}{section}

\newtheorem{claim}{Claim}[section]
\newtheorem{lemma}[claim]{Lemma}

\newtheorem{remark}[claim]{Remark} 
\newtheorem{assumption}{Assumption}

\newtheorem{theorem}{Theorem}
\newtheorem{proposition}[claim]{Proposition}
\newtheorem{corollary}[claim]{Corollary}
\newtheorem{definition}[claim]{Definition}

\author{Simone Bombari\thanks{Institute of Science and Technology Austria (ISTA). Emails: \texttt{\{simone.bombari, marco.mondelli\}@ist.ac.at}.}\;,
\;\;Jialei Luo\thanks{University of Oxford. Email: \texttt{jialei.luo@queens.ox.ac.uk}.}\;,\;\;Inbar Seroussi\thanks{Tel Aviv University. Email: \texttt{inbarser@tauex.tau.ac.il}.}\;,
\;\;Marco Mondelli\footnotemark[1]}

\title{High-Dimensional Private Linear Regression \\ with Optimal Rates}

\begin{document}

\newtheorem*{innercustomthm}{\customthmname}
\newcommand{\customthmname}{}
\newenvironment{customthm}[1]
  {\renewcommand{\customthmname}{#1}\begin{innercustomthm}}
  {\end{innercustomthm}}

\maketitle

\begin{abstract}
    Differentially private (DP) linear regression has received significant attention in the recent theoretical literature, with several approaches proposed to improve error rates.
    Our work considers the popular high-dimensional regime with random data, where the number of training samples $n$ and the input dimension $d$ grow at a proportional rate $d / n \to \gamma$, and it studies a family of one-pass DP gradient descent (DP-GD) algorithms satisfying $\rho^2 / 2$ zero concentrated DP. In this setting, we establish a deterministic equivalent for the DP-GD trajectory in terms of a system of ordinary differential equations. This 
    allows to analyze the effect of gradient clipping constants that are smaller than the typical norm of the per-sample gradients --- a setup 
    shown to improve performance in practice. For well-conditioned data, we show that DP-GD, upon properly choosing clipping constant and learning rate, achieves the non-asymptotic risk of $O(\gamma + \gamma^2 / \rho^2)$, and we establish that this rate is minimax optimal.
    Then, we consider the ill-conditioned case where the data covariance spectrum follows a power-law distribution, and we show that the risk displays a power-like scaling law in $\gamma$, highlighting the change in the exponent as a function of the privacy parameter $\rho$. Overall, our analysis demonstrates the benefits of practical algorithmic design choices, including aggressive gradient clipping and decaying learning rate schedules. 
\end{abstract}

\section{Introduction}\label{sec:intro}
    Differential privacy (DP) \citep{dwork2006} has consolidated as the standard framework for privacy guarantees and data protection in machine learning. This has motivated an extensive research effort on theoretical problems connected to statistics and optimization \citep{chaudhuri11a, bassily2014differentially, bassily19, kamath19a, kamath20a, cai21, cai2025scoreattacklowerbound}, on its implementation in training large scale deep learning architectures \citep{li2022large, de2022, mckenna2025}, and on its adoption for the release of public data \citep{Hawes2020Implementing, desfontainesblog20211001}.
    This framework quantifies privacy through numerical parameters (e.g.\ $\rho$ in Definition \ref{def:zcdp}), which upper bound the impact a single data point can have on the output of the algorithm. Then, guaranteeing a fixed level of DP imposes specific constraints on the learning algorithm, which come with a performance cost: the more stringent the privacy requirements (corresponding to smaller values of $\rho$), the higher is such cost. 
    For example, in first-order optimization, these constraints are typically implemented through \emph{clipping} the per-sample gradients if the loss is not Lipschitz, limiting the number of data access via \emph{early stopping}, and adding independent \emph{noise} at every update \citep{song13dpsgd, bassily2014differentially, Abadi2016}.

    Understanding the cost of privacy in different learning settings and how to achieve optimal performance has been at the center of a rich line of work spanning over a decade, see e.g. \cite{chaudhuri11a, bassily2014differentially, bassily19, song2021evading, cai21, varshney2022nearly}. On the one hand, this involves establishing statistical lower bounds on the expected risk 
    any algorithm must incur when learning with differential privacy; on the other hand, this involves  defining efficient algorithms that provably achieve this risk. 
    A notorious example is the problem of \emph{DP linear regression}, where one is given $n$ input-label pairs $(x_i, y_i)$, sampled i.i.d.\ from a joint distribution $P_{XY}$ such that $y_i = x_i^\top \theta^* + z_i$, where $\theta^* \in \R^{d}$, $x_i \in \R^d$ has mean-0 and covariance $\Sigma \in \R^{d\times d}$, and $z_i \in \R$ is independent label noise with mean-0 and variance $\zeta^2$.
    Informally, the goal is to provide an algorithm $\mathcal M$ that, given the training set, outputs a parameter $\theta^p$ guaranteeing a required privacy budget and minimizing the \emph{test risk}:
    \begin{equation}\label{eq:Pintro}
        \mathcal R(\mathcal M) = \frac{1}{2} \E_{(x, y) \sim P_{XY}} \left[ \left( x^\top \theta^p - y \right)^2 \right] - \frac{\zeta^2}{2} = \frac{\norm{\Sigma^{1/2} \left( \theta^p - \theta^* \right)}_2^2}{2}.
    \end{equation}
On the side of statistical lower bounds, 
\cite{cai21} study the maximum risk over a set $\mathcal B$ of possible $\theta^*$, i.e. $\sup_{\theta^* \in \mathcal B} \mathcal R(\mathcal M)$, and Theorem 4.1 therein shows that, if $\mathbb M$ represents the set of all $\rho^2 / 2$ zCDP algorithms\footnote{\cite{cai21} give their result in terms of $(\varepsilon, \delta)$-DP, and \eqref{eq:lowerboundintro} can be obtained through the conversion from zCDP to approximated DP stated in Proposition \ref{prop:zCDP}.} and $\mathcal B$ is the unit Euclidean ball, the \emph{minimax rates} of this problem satisfy  
    \begin{equation}\label{eq:lowerboundintro}
        \inf_{\mathcal M \in \mathbb M} \sup_{\theta^* \in \mathcal B} \E \left[ \mathcal R (\mathcal M) \right] = \Omega \left( \frac{d}{n} + \frac{d^2}{\rho^2 n^2 \log n } \right),
    \end{equation}
    where the expectation is taken with respect to the randomness of the data and of the private algorithm. 

On the algorithmic side, a number of recent papers gives efficient algorithms with theoretical guarantees on the test risk \citep{wang2018revisiting, cai21, milionis22a, varshney2022nearly, liu2023near, brown2024private}. In particular, \cite{varshney2022nearly, liu2023near, brown2024private} propose gradient-based algorithms for which $n = \tilde \Omega(d)$ samples are sufficient to achieve non trivial performance when $\rho$ is of constant order (here, $d$ denotes the input dimension and the $\sim$ hides logarithmic factors). For example, \cite{brown2024private} show that a $\rho^2/2$ zero concentrated DP (zCDP) gradient descent (DP-GD) algorithm $\mathcal M$ achieves an error of
    \begin{equation}\label{eq:upperboundintro}
        \mathcal R(\mathcal M) = O \left( \frac{d}{n} + \frac{d^2}{n^2 \rho^2} \, \mathrm{poly} \log n \right).
    \end{equation}
We note that %
 \cite{varshney2022nearly, liu2023near, brown2024private} consider gradient updates on the square loss $(x_i^\top \theta - y_i)^2$. As this loss is non Lipschitz, in order to bound the \emph{sensitivity} of the parameter updates with respect to any single training data point, the gradient is \emph{clipped}. The clipping then ensures that the $\ell_2$ norm of the gradient does not exceed a custom hyper-parameter $\clip$, which in turn defines the amount of private noise introduced at every iteration (see, e.g., Algorithm 1 in \cite{brown2024private}). Notably, a common approach in these works is to set the \emph{clipping constant} $\clip$ to be sufficiently large so that, with high probability, gradient clipping effectively \emph{does not take place} throughout the algorithm. This simplifies the analysis, as the optimization becomes more easily comparable to a quadratic problem with noisy gradient updates, but it introduces a larger amount of private noise at every iteration, which gives the additional poly-logarithmic factor in \eqref{eq:upperboundintro}.
    In fact, the practical benefit of avoiding gradient clipping is unclear, 
    and it is empirically pointed out in \cite{brown2024private} that the lowest error occurs under \emph{significant clipping}, see Figure 3 therein. This last conclusion is also in agreement with 
    experimental evidence for deep learning optimization \citep{Kurakin2022, li2022large, de2022}, which supports setting a sufficiently small $\clip$, with the caveat of properly rescaling either the learning rate or the number of training iterations. At the same time, other empirical work suggests an inherent trade-off in the choice of $\clip$ \citep{amin2019a, andrew2021differentially}, due to the bias induced by small clipping constants \citep{Xiangyi2020, song2021evading}.
    In short, this picture remains poorly understood, also on the theoretical side: to the best of our knowledge, there is no analysis of the dynamics of DP-GD when $\clip$ is of the order of the per-sample gradients, nor precise results quantifying the potential benefits of small clipping constants. Hence, practitioners are left without clear guidance, highlighting the need for a theoretical understanding of the impact of this hyper-parameter. 

    In this work, we consider a high-dimensional setting with random data, where the number of training samples $n$ grows proportionally with the number of input dimensions $d$ so that $d / n \to \gamma \in (0, +\infty)$, with constant-order privacy parameter $\rho$. This \emph{proportional regime} has been object of extensive study in the context of high-dimensional statistics and machine learning \citep{hastie2019surprises, sur2019modern, mei2022generalization}, owing its popularity also to the fact that, in modern data science, the feature dimension of the dataset is often comparable to the number of samples.
We note that, in this regime,
    \emph{(i)} the lower bound in \eqref{eq:lowerboundintro} does not capture the cost of privacy, since the second term depending on $\rho$ is always negligible with respect to the first term; and \emph{(ii)} the upper bound in \eqref{eq:upperboundintro} gives vacuous guarantees, due to the extra poly-logarithmic factor. Non-trivial test risk guarantees in the proportional regime are established by \cite{dwork2025} for the 
    solution of minimization problems with output and objective perturbation. However, the analysis is limited to isotropic data covariance and it does not characterize how the risk behaves as a function of $\gamma$. As concerns gradient-based methods, \cite{dwork2025} provide a characterization in terms of dynamical mean-field theory (DMFT) equations \citep{celentano2021high, gerbelot2024rigorous, han2024entrywise}. However, this characterization is hard to interpret, as it does not provide an explicit trajectory or rate, but rather a self consistent, non-Markovian dynamical system (see, e.g., Theorem 7.1 in \cite{dwork2025}). In contrast, this work considers a DP-GD algorithm, deriving the non-asymptotic behavior of its test risk as $\gamma \to 0$. Our analysis allows to characterize the role of the hyper-parameters of the algorithm, such as the learning rate and the clipping constant, also in the regime where $\clip$ is smaller than the typical norm of the per-sample gradients. Then, upon optimally choosing these hyper-parameters, we show that DP-GD achieves the rate $O(\gamma + \gamma^2 / \rho^2)$, which we prove to be minimax optimal for algorithms respecting $\rho^2/2$-zCDP.

\subsection{Main results and contributions}

    More precisely, we focus on a one-pass DP-GD\footnote{We use the term GD rather than stochastic gradient descent (SGD) to emphasize that our algorithm does not incorporate techniques such as privacy amplification via sub-sampling or shuffling.} algorithm (see Algorithm \ref{alg:dp-sgd}), with privacy guarantees on its final output expressed in terms of zero concentrated DP (Proposition \ref{prop:DPguarantees}). 
    Our main technical contributions are summarized below.
    \begin{enumerate}[leftmargin=*]
        \item Following recent progress in high-dimensional optimization \citep{paquette24homogenization, collins2024hitting, collins2024high, marshall2025clip}, we show that, as $d, n \to \infty$ such that $d / n \to \gamma \in (0, +\infty)$, the test risk of DP-GD is well described by the solution of a deterministic family of $d$ coupled ordinary differential equations (ODEs), see 
        Theorem \ref{thm:deteq}.
        Importantly, this approach allows to study the setting where the clipping constant $\clip$ is smaller than the typical size of the per-sample gradients, 
        characterizing the error rates of DP-GD via the non-asymptotic behavior of the system of ODEs as $\gamma \to 0$.

        \item Next, we focus on the well-conditioned case, i.e., where the eigenvalues of the data covariance do not scale with $\gamma$. In this setting, we demonstrate the benefits of using small clipping constants and decaying the learning rate, see Proposition \ref{prop:alpha012}. Then, in Theorem \ref{thm:harmonicbody}, we show that, through a harmonically decreasing learning rate schedule and a properly set clipping constant, Algorithm \ref{alg:dp-sgd} achieves a test risk of
        \begin{equation}\label{eq:tightupperboundintro}
            \mathcal R(\mathcal M) = O \left( \gamma + \frac{\gamma^2}{\rho^2} \right),
        \end{equation}
        both in high probability and in expectation.
        These are the \emph{optimal rates} for this task: we give a matching minimax lower bound in Theorem \ref{thm:lowerbound}. 
        More precisely, we show that, if $\mathbb M$ represents the set of all $\rho^2 / 2$ zCDP algorithms and $\mathcal B$ is the unit Euclidean ball, then
        \begin{equation}\label{eq:tightlowerboundintro}
            \inf_{\mathcal M \in \mathbb M} \sup_{\theta^* \in \mathcal B} \E \left[ \mathcal R( \mathcal M) \right] = \Omega \left( \gamma + \frac{\gamma^2}{\rho^2} \right).
        \end{equation}

        \item Finally, we study the ill-conditioned case, i.e., where the covariance spectrum follows a power-law distribution. 
In this setting, we show that the 
test risk also follows a power law of the form $\gamma^{h}$, see Theorem \ref{thm:scalinglawsbody}. More precisely, via a tight analysis of the system of coupled ODEs, we characterize the exponent $h$ in terms of the privacy level $\rho$, the learning rate schedule and the decays of the covariance spectrum and of the target regression coefficients $\theta^*$. 
                
    \end{enumerate}

    Our results are illustrated via experiments on synthetically generated data as well as on standard datasets (MNIST, California housing prices). 
    These showcase \emph{(i)} the accuracy of the deterministic ODEs (Figure \ref{fig:ODElikeAlg}), \emph{(ii)} the predictive behavior of our proposed scaling laws (Figures \ref{fig:scaling_laws} and \ref{fig:scaling_laws_new}), \emph{(iii)} the benefits of taking a small clipping constant and the connection between clipping and an appropriate choice for the learning rate (Figures \ref{fig:heatmaps}, \ref{fig:heatmap_mnist}, \ref{fig:heatmap_housing}), and \emph{(iv)} the impact on performance of decaying the learning rate during training (Figures \ref{fig:schedules} and \ref{fig:schedules_real}).
    We finally remark that our analysis is fueled by a methodology that is both innovative compared to earlier work in the DP optimization literature \citep{varshney2022nearly, liu2023near, brown2024private} and rather general, thus laying the foundations for the study of DP algorithms in the high-dimensional setting.

\section{Related work}\label{sec:rel}

\paragraph{DP optimization.}
Since its introduction in \cite{dwork2006}, DP has provided the gold standard in the field of private data analysis, and different methods have been proposed with the purpose of learning with DP, such as objective and output perturbation \citep{chauduri2008, chaudhuri11a, kifer12} or different variants of DP-GD \citep{song13dpsgd, bassily2014differentially, Abadi2016}. In the last decade, a popular line of work has theoretically investigated private learning in various settings, such as Lipschitz and bounded optimization \citep{kifer12, bassily2014differentially, bassily19}, generalized linear models with Lipschitz loss \citep{jain14, song2021evading}, models in the over-parameterized regime \citep{bombari2024privacy}, and M-estimators \citep{avella2023differentially}.
In DP linear regression,
\cite{wang2018revisiting, milionis22a} require $n = \tilde \Omega(d^{3/2})$ samples, where 
the privacy budget is assumed of constant order, while \cite{anderson2025sample} require $n = \Omega(d^2)$ samples to also achieve robustness under a corrupted data model. 
DP-GD with adaptive clipping is shown to achieve a sample complexity of $n = \tilde \Omega(d)$ in \cite{varshney2022nearly}. This is ``nearly optimal'', in the sense that it matches, up to logarithmic factors, the minimax lower bound in \cite{cai21}. Similar nearly optimal rates were previously obtained by \cite{liu22b} (however with a computationally inefficient method), and more recently by \cite{liu2023near} and \cite{brown2024private}.
Notably, in the proportional regime $n = \Theta(d)$, if the privacy budget is of constant order ($\varepsilon / \sqrt{\ln(1 / \delta)}=\Theta(1)$ in \cite{varshney2022nearly, liu2023near}, or $\rho=\Theta(1)$ in \cite{brown2024private}), 
these bounds on the test risk diverge logarithmically either in $n$ \citep{varshney2022nearly} or in the failure probability \citep{liu2023near, brown2024private}\footnote{More precisely, in \cite{liu2023near} the number of samples would not be sufficient to achieve Eq.\ (4) with high probability.}.
Precise constant-order values for the test risk in the proportional regime have been provided by the recent work of \cite{dwork2025} for private algorithms based on output and objective perturbation in the case of isotropic data covariance, without focusing on the non-asymptotic rates in $\gamma$.



\paragraph{Minimax lower bounds.}

Minimax lower bounds in statistical estimation are classically derived via tools such as Le Cam’s method, Fano’s inequality, and Assouad lemma \citep{Tsybakov2009nonparametric, wainwright2019high}. 
In a linear model with Gaussian data, a convenient way to prove minimax lower bounds is to choose a prior on $\theta^*$ and compute the corresponding Bayes risk. 
Under DP, early minimax lower bounds were developed by \cite{barber2014privacy}, with \cite{acharya2021differentially} later formulating private analogues of Le Cam, Fano and Assouad. Fingerprinting codes and tracing attack arguments were introduced by \cite{bun2014fingerprinting, dwork2015robust}, and those were used to provide statistical lower bounds for high-dimensional mean estimation and linear regression  by \cite{cai21}. More recently, tracing attacks have been generalized to the score attack technique discussed in \cite{cai2025scoreattacklowerbound}, which allows to treat more general statistical models with a well defined score statistic.
In the context of local DP, where each sample is privatized prior to observation, the statistical cost of privacy has been studied in \cite{duchi2013local, Duchi2018minimax, rohde2020geometrizing, duchi2024right}.

%



\paragraph{Learning in high dimensions.}

The setting where the input dimension (or model size) $d$ scales with the number of training samples $n$ gained popularity for its ability to explain several empirical phenomena observed in modern data analysis and deep learning \citep{hastie2019surprises, belkin2019, bartlett2020benign, bartlett21deep}.
Early investigations of this regime appear in the work of \cite{Huber1973robust} in the context of robust regression. More recently, \cite{hastie2019surprises} studied linear regression in the proportional asymptotic regime $d / n \to \gamma$, analyzing the risk of ridge(-less) regression in both the under-parameterized ($\gamma < 1$) and over-parameterized ($\gamma > 1$) regime.
The proportional asymptotic in ridge regression had been previously considered by \cite{dicker2016ridge, dobriban2018high}, while \cite{cheng2024dimension} went beyond this regime establishing non-asymptotic guarantees with multiplicative error bounds.
\cite{sur2019modern} studied the maximum likelihood estimator in logistic regression in the proportional regime, highlighting its qualitative differences with respect to the case where $d = o(n)$; \cite{deng2022model, montanari2025generalization} studied logistic models, considering the effects of over-parameterization and benign overfitting.
Another line of work, closer to our technical approach, analyzes the behavior of one-pass GD algorithms in terms of high-dimensional stochastic differential equations or coupled ODEs \citep{paquette2022implicit, paquette24homogenization, collins2024hitting, marshall2025clip}. This strategy gives a deterministic equivalent for the gradient dynamics, which in turn leads to insights on optimization stability and the role of stochastic batching.
Another common analytical tool to analyze the high-dimensional regime is the convex Gaussian minimax theorem
\citep{stojnic2013framework, thrampoulidis2015regularized}, which constitutes the main technical tool in \cite{dwork2025}. 


%

\paragraph{Gradient clipping.}
In the context of private optimization with a non-Lipschitz loss, the role of clipping and the magnitude of the corresponding clipping constant $\clip$ has attracted attention due to its nuanced implications: while a small $\clip$ significantly affects the gradients, larger values force the addition of more private noise, 
suggesting that the choice of $\clip$ induces a bias-variance trade-off \citep{brendan2018learning, amin2019a, andrew2021differentially, das2023, brown2024private}. Prior work has argued that the bias induced by small clipping constants can prevent convergence \citep{amin2019a, Xiangyi2020, song2021evading}, which motivates an adaptive selection of $\clip$ based on (private) statistics of the magnitude of the gradients \citep{Abadi2016, andrew2021differentially, pichapati2019, Golatkar2022}. Recent experimental studies have given evidence that the best performance is achieved with a sufficiently small $\clip$ \citep{Kurakin2022, li2022large, de2022}, but it has also been shown that overly-aggressive clipping can be damaging in the context of model calibration \citep{bu2023on, brown2024private}. Theoretical insights on the benefits of small clipping constants have been provided in \cite{das2023, Xiangyi2020}, with \cite{Xiangyi2020} proving optimization guarantees when the gradient distribution is sufficiently symmetric and \cite{das2023} considering the setting where the Lipschitz constant of the loss is sample-dependent.
Recent work on DP linear regression \citep{varshney2022nearly, liu2023near, brown2024private} shares the common feature of setting the (possibly adaptive) $\clip$ a poly-logarithmic factor larger than the expected norm of the per-sample gradient. This ensures that, with high probability, at most a few gradients are clipped during training, and we note that these logarithmic factors are strongly tied to the consequent logarithmic divergence of the test risk guarantees discussed in the previous paragraph.  \cite{marshall2025clip} provide a precise analysis of a version of clipped-GD, although they do not  focus on privacy and, therefore, they do not add private noise. In a parallel work, \cite{lin2025high} add private noise but do not consider clipping, formulating instead a notion of privacy in terms of a diffusion surrogate of the algorithm which is different from a DP guarantee on the algorithm itself.



\paragraph{Scaling laws.}
Large scale experiments on modern AI systems revealed how the risk improves according to a power-law in the model size and number of training data \citep{kaplan2020scaling, hoffmann2022an}.
This evidence sparked a recent line of work establishing provable scaling laws in simplified settings such as linear regression and shallow neural networks \citep{paquette20244+, defilippis2024dimension, lin2024scaling, lin2025improved, ren2025emergence, ferbach2025dimensionadapted,wu2026improved}. These results generally achieve power-law type scaling laws assuming the data covariance spectrum itself follows a power-law distribution. 
Most similar to our setting (see Assumption \ref{ass:power_law_1}) is the work by \cite{collins2024high}, which establishes convergence rates for the risk in the context of one-pass GD with adaptive learning rates. 





\section{Preliminaries}

\paragraph{Notation.} Given a vector $v$, we denote by $\norm{v}_2$ its Euclidean norm. Given a matrix $A$, we denote by $\tr(A)$ and $\opnorm{A}$ its trace and operator (spectral) norm. Given a symmetric matrix $A$, we denote by $\evmin{A}$ ($\evmax{A}$) its smallest (largest) eigenvalue. The complexity notations $\Omega(\cdot)$, $O(\cdot)$, $\omega(\cdot)$, $o(\cdot)$ and $\Theta(\cdot)$ are understood for large data size $n$ and input dimension $d$, while the notation $\Omega_\gamma(\cdot), O_\gamma(\cdot), \Theta_\gamma(\cdot), \omega_\gamma(\cdot), o_\gamma(\cdot)$ is intended for sufficiently small $\gamma$.
We indicate with $C > 0$ a numerical constant independent of $n$ and $d$, whose value may change from line to line, and we say that an event holds with overwhelming probability if it holds with probability at least $1-e^{-\omega(\log d)}$.


\paragraph{Linear regression.} Let $(X, Y)$ be a labeled training dataset, where $X=[x_1, \ldots, x_n]^\top \in \R^{n \times d}$ contains the training data on its rows and $Y=[y_1, \dots, y_n]^\top \in \R^{n}$ contains the corresponding labels, such that the input-label pairs are sampled i.i.d.\ from a joint distribution $P_{XY}$.
We consider a 
\emph{linear regression} model, where the labels are defined as
\begin{equation}\label{eq:datamodel}
y_i = x_i^\top \theta^* + z_i,    
\end{equation}
where $\theta^* \in \R^{d}$, $x_i$ has mean-0 and covariance $\Sigma\in \R^{d\times d}$, and $z_i$ is independent label noise with mean $0$ and variance $\zeta^2$. We use the shorthand $(\omega_i, \lambda_i)$ to denote an eigenvector-eigenvalue pair of the data covariance $\Sigma$, and let $\lambda_{\max} = \evmax{\Sigma}, \lambda_{\min} = \evmin{\Sigma}$ denote its largest and smallest eigenvalue respectively.

The goal of a DP linear regression algorithm is to output a parameter $\theta^p$ that guarantees a required privacy budget and minimizes the \emph{test risk}:
\begin{equation}\label{eq:P}
    \mathcal P(\theta^p) = \frac{1}{2} \E_{(x, y) \sim P_{XY}} \left[ \left( x^\top \theta^p - y \right)^2 \right] = \frac{\norm{\Sigma^{1/2} \left( \theta^p - \theta^* \right)}_2^2 + \zeta^2}{2}.
\end{equation}
We also use the notation $\mathcal R(\theta^p) = \mathcal P(\theta^p) - \zeta^2 / 2$ to denote the excess test risk s.t.\ $\mathcal R(\theta^*) = 0$.
This notation slightly differs from the simplified presentation in Section \ref{sec:intro}, as $\mathcal R$ here is a function defined in parameter space. We will stick to this notation for the rest of the paper.


\paragraph{Differential privacy (DP).} The definition of DP builds on the notion of \emph{adjacent datasets}: a dataset $D'$ is said to be adjacent to a dataset $D$ if they differ by only one sample.
In this work, we will frame privacy in terms of zero-concentrated DP (zCDP) \citep{bun2016concentrated}, which is defined below.


\begin{definition}[Zero concentrated DP \citep{bun2016concentrated}]\label{def:zcdp}
Given $\alpha \in (1, +\infty)$ and two random variables $X$ and $X'$ with laws $p_X$ and $p_{X'}$,
their $\alpha$-Rényi Divergence \citep{Renyi61} is defined as
\begin{equation}
    D_{\alpha} \left( X \,\|\, X' \right) = \frac{1}{\alpha - 1} \ln \int \left( \frac{p_X (\theta)}{p_{X'} (\theta)} \right)^\alpha p_{X'} (\theta) \diff \theta.
\end{equation}
Then, a randomized algorithm $\mathcal A$ satisfies $\rho^2/2$-zero concentrated DP ($\rho^2/2$-zCDP) if, for any pair of adjacent datasets $D, D'$ and any $\alpha \in (1, +\infty)$, we have $D_{\alpha} \left( \mathcal{A}(D) \,\|\, \mathcal{A}(D') \right) \leq \alpha \rho^2 / 2$.
\end{definition}


Guarantees for zCDP can be translated to other 
formulations, such as $(\varepsilon, \delta)$-DP, and we provide conversion formulas in Appendix \ref{app:DP}.

\section{DP-GD and deterministic equivalent}\label{sec:DP-GD}


We consider a DP gradient descent algorithm performing a single pass on the $n$ training samples (Algorithm \ref{alg:dp-sgd}). 
In this formulation, both the learning rate $\eta_k$ and the magnitude of the private noise $\sigma_k$ at the $k$-th iteration follow the 
schedules $\{\eta_k\}_{k=1}^n$, $\{\sigma_k\}_{k=1}^n$ given as input, and 
$\eta_k$ is modified adaptively as a function of $\norm{x_k}_2$.
The output of the algorithm is 
the final parameter $\theta_n$, which is the \emph{only} object 
for which we provide privacy guarantees, due to our approach based on privacy amplification by iteration \citep{feldman2018iteration, feldman2020linear}. This differs from prior work where any intermediate step can be privately released \citep{varshney2022nearly, liu2023near, brown2024private}, which relies on advanced composition theorems \citep{Dwork2014}.

\setlength{\textfloatsep}{5pt} 
\begin{algorithm}[t]
\caption{DP-GD}
\label{alg:dp-sgd}
\begin{algorithmic}
\REQUIRE Training data $(X, Y)$, learning rate schedule $\{ \eta_k \}_{k=1}^n$, clipping constant $\clip$, noise multiplier schedule $\{ \sigma_k \}_{k=1}^n$, initialization $\theta_0 = 0$. \\
\vspace{0.1cm}

\FOR{$k \in \{1, \ldots, n\}$}
    \STATE Compute the gradient $g_k = \nabla_{\theta} \left(x_{k}^\top \theta_{k-1} - y_{k} \right)^2 / 2 = x_k \left( x_{k}^\top \theta_{k-1} - y_k \right)$.
    \vspace{0.1cm}
    
    \STATE Clip the gradient $\bar g_{k} = g_{k} \min \left(1, \frac{\clip}{\norm{g_k}_2} \right)$.
    \vspace{0.1cm}   
    
    \STATE Set the learning rate adaptively $\bar \eta_k = \min \left( \eta_k, \frac{2}{\norm{x_k}_2^2} \right)$.
    \vspace{0.1cm}
    
    \STATE Sample independent Gaussian noise $b_k \sim \mathcal{N}(0, I)$.
    \vspace{0.1cm}

    \STATE Update the model parameters $\theta_k = \theta_{k-1} - \bar \eta_k \bar g_{k} + 2 \clip \sigma_k b_k$.
\ENDFOR
\vspace{0.1cm}
\ENSURE Model parameters $\theta^p = \theta_n$.
\end{algorithmic}
\end{algorithm}


\begin{proposition}\label{prop:DPguarantees}
    Algorithm \ref{alg:dp-sgd} satisfies $(\rho^2 / 2)$-zCDP, where
    \begin{equation}
        \rho = \max_{k \in [n]} \frac{\eta_k}{\sqrt{\sum_{j = k}^n \sigma_j^2}}.
    \end{equation}
\end{proposition}

Proposition \ref{prop:DPguarantees} (whose proof is deferred to Appendix \ref{app:DP}) states that each sample $x_k$ is ``protected'' by the overall noise introduced in the following updates $\sum_{j = k}^n \sigma_j^2$. For an assigned privacy guarantee, we can minimize the noise introduced by the algorithm $\sum_{j = 1}^n \sigma_j^2$ (and, therefore, optimize its performance) via the schedule below:
\begin{equation}\label{eq:schedulesdiff}
    \eta_k = \rho \sqrt{\sum_{j = k}^n \sigma_j^2}, \quad \textup{ or, equivalently, } \quad
    \rho^2 \sigma^2_k =
    \begin{cases}
    \eta_k^2 - \eta_{k+1}^2, \qquad k \in \{1, \ldots, n - 1\},\\
    \eta_k^2, \qquad\qquad \hspace{1.2em} k = n.
    \end{cases}
\end{equation}

Thus, from now on, we will restrict our study to the noise schedules defined in \eqref{eq:schedulesdiff}, making Algorithm \ref{alg:dp-sgd} uniquely defined given the privacy parameter $\rho$, the clipping constant $\clip$, and a non-increasing learning rate schedule $\{ \eta_k \}_{k = 1}^n$.
We now make two assumptions on data distribution and hyper-parameter scaling.

\begin{assumption}[Data distribution]\label{ass:data}
    $\{x_i\}_{i=1}^n$ are $n$ i.i.d.\ samples from the multivariate, mean-0, Gaussian distribution $\mathcal P_{X}$, with covariance $\Sigma := \E \left[xx^\top \right] \in \R^{d \times d}$. Furthermore, the noise $z_i$ is mean-0, Gaussian, with variance $\zeta^2 = \Theta(1) > 0$, and $\norm{\theta^*}_2 = \Theta(1)$, $\lambda_{\max} = O(1)$, and $\tr(\Sigma) = d$. 
\end{assumption}
Gaussian data was 
considered 
in \cite{dwork2025}, while \cite{varshney2022nearly, liu2023near, brown2024private} assume bounds on the tail of the data distributions. We focus on the Gaussian case for simplicity, 
but we expect to be possible to extend some of our results to data with sufficiently well-behaved tails (see Appendix A in \cite{marshall2025clip}).
The normalization $\tr(\Sigma) = d$ is also chosen to simplify the presentation of the results, and is not strictly necessary for our derivation. 

\begin{assumption}[Hyper-parameter scaling]\label{ass:learning_rate_schedules}
Let the clipping constant in Algorithm \ref{alg:dp-sgd} be 
\begin{equation}\label{eq:clippingconstant}
    \clip = c \sqrt{d},
\end{equation}
where $c > 0$ is a constant independent of $n$ and $d$. Furthermore, the learning rate schedule is given by
\begin{equation}\label{eq:tildeeta}
    \eta_k = \frac{\tilde \eta \left(k / n \right)}{n},
\end{equation}
where $\tilde \eta : [0, 1] \to \R$ is a function such that both $\tilde \eta^2 \left( \cdot \right)$ and the absolute value of its first and second derivatives are uniformly bounded from above by a constant independent of $n$ and $d$.
\end{assumption}

The norm of the gradient in Algorithm \ref{alg:dp-sgd} is $\norm{g_k}_2 = \norm{x_k}_2 | x_{k}^\top \theta_{k-1} - y_k | = \norm{x_k}_2 \sqrt{2 \mathcal P(\theta_{k-1})}$. Then, if $\mathcal P(\theta_{k - 1}) = \Theta(1)$, we have that $\norm{g_k}_2 = \Theta(\sqrt d)$, as $\norm{x_k}_2 = \Theta(\sqrt d)$ with high probability by Assumption \ref{ass:data}.
Note that the condition $\clip = c \sqrt d$ considers asymptotically smaller clipping constants than the ones considered e.g.\ in \cite{brown2024private}, which in this setting are of order $\clip = \Omega(\sqrt d \log^3 n)$ (see their Theorem 2.7).
Let us also introduce
\begin{equation}
    r(\theta, x, y) = x^\top \theta - y, \qquad r_c(\theta, x, y) = r(\theta, x, y) \min \left(1 , \frac{c}{\left| r(\theta, x, y) \right|} \right),
\end{equation}
where $r(\theta, x, y)$ represents the residual in $\theta$, and $r_c(\theta, x, y)$ is a clipped version of it. Then, as done in \cite{marshall2025clip}, we define the \emph{descent reduction factor} and the \emph{variance reduction factor}
\begin{equation}\label{eq:munu}
    \mu_c(\theta) = \frac{\norm{\E_{(x, y) \sim P_{XY}} \left[ r_c(\theta, x, y) \, x \right]}_2}{\norm{\E_{(x, y) \sim P_{XY}} \left[ r(\theta, x, y) \, x \right]}_2}, \qquad \nu_c(\theta) = \frac{\E_{(x, y) \sim P_{XY}} \left[ r_c(\theta, x, y)^2 \right]}{\E_{(x, y) \sim P_{XY}} \left[r(\theta, x, y)^2 \right]}.
\end{equation}

In Lemmas \ref{lemma:munu} and \ref{lemma:munubounds}, we show that, due to Assumption \ref{ass:data}, the functions $\mu_c(\theta)$ and $\nu_c(\theta)$ only depend on the risk $\mathcal R(\theta)$, i.e., $\mu_c(\theta_1) = \mu_c(\theta_2)$ if $\mathcal R(\theta_1) = \mathcal R(\theta_2)$ (and same for $\nu_c$). Then, with a slight abuse of notation, we will regard $\mu_c$ and $\nu_c$ as functions from $\R_{\geq 0}$ to $(0, 1)$ taking as argument directly the value of the risk. Furthermore, in the aforementioned Lemmas, we show that, for any fixed value of the risk $R$, $\mu_c(R)$ and $\nu_c(R)$ are monotonically increasing in $c$, going from $0$ (as $c \to 0$) to $1$ (as $c \to + \infty$), see Figure \ref{fig:munu}.

\begin{definition}[Deterministic equivalent]\label{def:deteq}
For any $t \in [0, 1)$, define the system of $d$ coupled ODEs
\begin{equation}\label{eq:ODEi}
    \diff D_i(t) = - 2 \lambda_i \bar{\eta}(t) \mu_c(R(t)) D_i \diff t + \lambda_i \bar{\eta}^2(t) \nu_c(R(t)) (R(t) + \zeta^2/2) \gamma_n \diff t + 2 c^2 \tilde \sigma^2(t) \gamma_n^2 \diff t,
\end{equation}
with $\gamma_n = d / n$, $\bar \eta(t) = \min(\tilde \eta(t), 2 / \gamma_n)$, 
$\tilde \sigma(t)$ such that $\rho^2 \tilde \sigma ^2(t) = - \diff \tilde \eta^2(t) / \diff t$,
\begin{equation}\label{eq:deteq}
    R(t) = \frac{1}{d} \sum_{i=1}^d \lambda_i D_i(t),
\end{equation}
and initial condition $D_i(0) = d \left(\omega_i^\top \theta^* \right)^2 / 2$.
\end{definition}

We remark that the previous system of ODEs has a unique solution, due to Picard–Lindel\"{o}f theorem (see Chapter 2 in \cite{teschl2012}), since the RHS of \eqref{eq:ODEi} is continuous in $t$, locally Lipschitz in $D_i$ ($\mu_c(R)$ and $\nu_c(R)$ are bounded and Lipschitz due to Lemma \ref{lemma:munu}), and since it respects the linear growth condition in Theorem 2.17 in \cite{teschl2012}.

\begin{theorem}\label{thm:deteq}
    Let Assumptions \ref{ass:data} and \ref{ass:learning_rate_schedules} hold. Let $\rho = \Theta(1)$, and $n, d \to \infty$ s.t.\ $\gamma_n = d / n \to \gamma \in (0,\infty)$. Denote by $\theta_k$ a realization of Algorithm \ref{alg:dp-sgd}, and by $R(t)$ the solution to the system of ODEs described by \eqref{eq:ODEi} and \eqref{eq:deteq}. Then, with overwhelming probability, we have
    \begin{equation}\label{eq:deteq1}
        \sup_{t\in[0,1)} \left| \mathcal{R}(\theta_{\lfloor tn\rfloor}) - R(t) \right| = O \left( \frac{\log^2 n}{\sqrt n}  \right).
    \end{equation}
    Furthermore, we also have
    \begin{equation}
        \sup_{t\in[0,1)} \left| \E \left[ \mathcal{R}(\theta_{\lfloor tn\rfloor}) \right] - R(t) \right| = O \left( \frac{\log^2 n}{\sqrt n} \right),
    \end{equation}
    where the expectation is taken with respect to both the data and the algorithm.
\end{theorem}
In words, Theorem \ref{thm:deteq} states that the risk of Algorithm \ref{alg:dp-sgd} can be well approximated in terms of the solution of a system of $d$ coupled deterministic ODEs. 
This system of ODEs then provides a deterministic equivalent for the dynamics of DP-GD, since it approximates sharply its risk without depending on the stochasticity of the algorithm and of the data.
Definition \ref{def:deteq} and Theorem \ref{thm:deteq} also implicitly introduce the continuous time variable $t \in [0, 1]$, which is mapped to the iterations via $k = \lfloor tn \rfloor$. 
%

Notice that, in the isotropic case where $\Sigma = I$, the system in \eqref{eq:ODEi} reduces to the single ODE
\begin{equation}\label{eq:isotropicODE}
    \diff R(t) = - 2 \bar{\eta}(t) \mu_c(R(t)) R(t) \diff t + \bar{\eta}^2(t) \nu_c(R(t)) (R(t) + \zeta^2/2) \gamma_n \diff t + 2 c^2 \tilde \sigma^2(t) \gamma_n^2 \diff t,
\end{equation}
with $R(0) = \norm{\theta^*}_2^2 / 2$, $\bar \eta(t) = \min(\tilde \eta(t), 2 / \gamma_n)$, and $\rho^2 \tilde \sigma ^2(t) = - \diff \tilde \eta^2(t) / \diff t$.
The RHS of \eqref{eq:isotropicODE} has three terms:
\emph{(i)} a negative term proportional to $R(t)$, which captures the descent towards the minimizer of $R$, \emph{(ii)} a positive term proportional to $R(t) + \zeta^2 / 2$, which captures the fact that Algorithm \ref{alg:dp-sgd} at step $k$ does not optimize $R$ directly, but rather an estimate based on $(x_k, y_k)$, and \emph{(iii)} another positive term proportional to $\tilde \sigma^2(t)$, which models the private noise in Algorithm \ref{alg:dp-sgd}.
We also remark that the first two terms on the RHS of \eqref{eq:isotropicODE} are proportional to $\bar \eta(t) = \min(\tilde \eta(t), 2 / \gamma_n)$, in analogy to the adaptive learning rate step that defines $\bar \eta_k$ in Algorithm \ref{alg:dp-sgd}.

For general covariance, it is possible to upper and lower bound the value of $R(t)$ as a function of the largest ($\lambda_{\max}$) and smallest ($\lambda_{\min}$) eigenvalues of the covariance matrix $\Sigma$. In particular, defining
\begin{equation}\label{eq:upperlowerODEsbodynew}
\begin{aligned}
    \diff \overline {R}(t) &= - 2 \lambda_{\min} \tilde{\eta}(t) \mu_c(\overline R) \overline{R} \diff t +  \lambda_{\max} \tilde{\eta}^2(t) \nu_c(\overline R) (\overline {R} + \zeta^2/2) \gamma_n \diff t + 2 c^2 \tilde \sigma^2(t) \gamma_n^2 \diff t, \\
    \diff \underline {R}(t) &= - 2 \lambda_{\max} \tilde{\eta}(t) \mu_c(\underline R) \underline{R} \diff t + \tilde{\eta}^2(t) \nu_c(\underline R) (\underline {R} + \zeta^2/2) \gamma_n \diff t + 2 c^2 \tilde \sigma^2(t) \gamma_n^2 \diff t,
\end{aligned}
\end{equation}
with $\overline R(0) = \underline R(0) = R(0) = \|\Sigma^{1 /2} \theta^*\|_2^2 /2$, we have, for every $t \in [0,1]$,
\begin{equation}
    \underline R(t) \leq R(t) \leq \overline R(t).
\end{equation}
We refer to Proposition \ref{prop:conditionsandwhich} in Appendix \ref{app:ODEs} for the formal statement and proof. Compared to the isotropic case, the upper bound $\overline R(t)$ reduces the descent term by a factor $\lambda_{\min} \leq 1$ and increases the second term in the RHS of \eqref{eq:isotropicODE} by a factor $\lambda_{\max} \geq 1$. Instead, the lower bound $\underline R(t)$ just increases the descent term by a factor $\lambda_{\max} \geq 1$.
In Section \ref{sec:optimalrates} we use these bounds to \emph{(i)} quantify the benefit of clipping, \emph{(ii)} establish the minimax optimality of DP-GD with a properly chosen learning rate schedule and, more generally, to \emph{(iii)} give convergence rates for a wide class of schedules. In Section \ref{sec:scalinglaws} we consider directly the system of $d$ ODEs in \eqref{eq:ODEi} to \emph{(iv)} derive scaling laws. 

\begin{remark}\label{remark:H} 
The risk of Algorithm \ref{alg:dp-sgd} can be compactly characterized also through the solution of the following stochastic differential equation
\begin{equation}\label{eq:HDPSGD}
    \diff \Theta_t = - 2 \bar \eta(t) {\mu_c(\Theta_t)}\nabla \mathcal{P}(\Theta_t) \diff t 
    + \bar \eta(t) \sqrt{\frac{2 \nu_c (\Theta_t) \mathcal{P}(\Theta_t) \Sigma}{n }} \diff B^s_t + 2 \frac{\sqrt d}{n} c \tilde \sigma(t) \diff B^p_t,
\end{equation}
where $\Theta_0 = \theta_0 = 0$, $B^s_t$ and $B^p_t$ are two independent standard Brownian motions in $\R^d$, $\bar \eta(t) = \min(\tilde \eta(t), 2n / d)$, and $\tilde \sigma(t)$ is such that $\rho^2 \tilde \sigma ^2(t) = - \diff \tilde \eta^2(t) / \diff t$. More precisely, in Appendix \ref{app:deteq}, 
we show that, with overwhelming probability,  
$$
    \sup_{t\in[0,1)} \left| \mathcal{R}(\theta_{\lfloor tn\rfloor}) - \mathcal{R}(\Theta_t) \right| = O \left( \frac{\log^2 n}{\sqrt n}  \right).
$$
This is a description of the dynamics of DP-GD through a single stochastic differential equation, rather than through a system of $d$ coupled ODEs.
\end{remark}


\paragraph{Proof ideas for Theorem \ref{thm:deteq}.}
The argument follows a similar strategy as the one proving Theorem 1 in \cite{marshall2025clip}. In particular, let $u_k = \theta_k - \theta^*$ and $V_t = \Theta_t - \theta^*$, where
$\Theta_t$ is defined in \eqref{eq:HDPSGD}. Define the set of quadratic functions
\begin{equation}
    q_z(v) = \frac{1}{2} v^\top (\Sigma - z I)^{-1} v, \qquad z \in \Omega = \{ z \in \mathbb C : |z| = 2 \opnorm{\Sigma} \}.
\end{equation}
On the one hand, we can write the Doob's decomposition of the update
\begin{equation}\label{eq:doobbody}
    q_z(u_{k+1}) - q_z(u_k) = \frac{1}{n} \mathcal F_z \left( u_k, R(u_k), k/n \right) + \Delta M_k(z) + \Delta E_k(z).
\end{equation}
Here, $\mathcal F_z$ denotes the predictable drift, with contributions coming from the deterministic descent term, the sampling-variance term, and the private-noise term (see \eqref{eq:SGD_q_Doob}):
\begin{equation}
\begin{aligned}
    \mathcal F_z(u, R, t) = & - \bar \eta(t) \mu_c(R) \left( \|u\|_2^2 + z q_z(u) \right)
    + \bar \eta(t)^2 \nu_c(R) \left(R + \zeta^2 / 2 \right) \gamma_n
    \frac{\tr \left( \Sigma (\Sigma - z I)^{-1} \right)}{d} \\
    & \qquad + 2 c^2 \tilde \sigma^2(t) \gamma_n^2
\frac{\tr  \left( (\Sigma - z I)^{-1} \right)}{d}.
\end{aligned}
\end{equation}
Combining Lemma 2 in \cite{marshall2025clip} with Lemmas \ref{lem:mar_noise} and \ref{lemma:step}, we show that the martingale and error terms in \eqref{eq:doobbody} ($\sum_{j \le k} \Delta M_j(z)$ and $\sum_{j\le k}\Delta E_j(z)$) are $\tilde O(n^{-1/2})$. 

On the other hand, applying It\^o's formula to \eqref{eq:HDPSGD} yields
\begin{equation}\label{eq:sdeevolbody}
    \diff q_z(V_t) = \mathcal F_z \left(V_t, R(\Theta_t), t \right)  \diff t + \diff M_t^{\mathrm{sde}}(z).
\end{equation}
Here, the predictable part $\mathcal F_z$ matches with the one of the discrete dynamics and, in Lemma \ref{lem:M_DPSGD}, we show that the remaining term is small, i.e., $| \int_0^t \diff M_s^{\mathrm{sde}}(z) | = \tilde O(n^{-1/2})$ with overwhelming probability.
Then, relying on the Lipschitz continuity of $\mathcal F_z$ with respect to its first argument and controlling the discretization errors, Gronwall's inequality yields (see the argument in \eqref{eq:differenceintm}-\eqref{eq:q_SGD_DPSGD})
\begin{equation}
    \sup_{t \in [0,1)} \left| q_z( u_{\lfloor tn\rfloor} )-q_z(V_t) \right| = \tilde O \left(n^{-1/2} \right),
\end{equation}
with overwhelming probability uniformly in $z \in \Omega$, due to a net argument on $\Omega$.
Now, intuitively, one can remove the martingale term from \eqref{eq:sdeevolbody} and consider the resulting deterministic evolution. As
\begin{equation}\label{eq:contoursbody}
    q_z(V) = \frac{1}{2} \sum_{i=1}^d \frac{ (V^\top \omega_i)^2 }{\lambda_i - z}, \qquad \norm{V}_2^2 = \frac{i}{\pi} \oint_\Omega q_z(V) \diff z, \qquad R(V + \theta^*) = \frac{i}{2\pi} \oint_\Omega z q_z(V) \diff z,
\end{equation}
where the last two equalitites follow from the Cauchy integral formula, we have that $q_z(V_t)$ is described by an integro-differential equation in $z$ and $t$ (see \eqref{eq:resolvantwithoutmartingale}). Then, due to Lemma 4.1 in \cite{collins2024hitting}, its unique solution is
\begin{equation}
    q_z(V_t) = \frac{1}{d} \sum_{i=1}^d \frac{D_i(t)}{\lambda_i - z},
\end{equation}
where $D_i(t)$ is defined in \eqref{eq:ODEi}. Using the last relation in \eqref{eq:contoursbody} recovers $R(t) = \sum_{i=1}^d \lambda_i D_i(t) / d$, which gives the desired result. The full argument is deferred to Appendix \ref{app:deteq}.


\begin{figure}
    \centering
    \includegraphics[width=\linewidth]{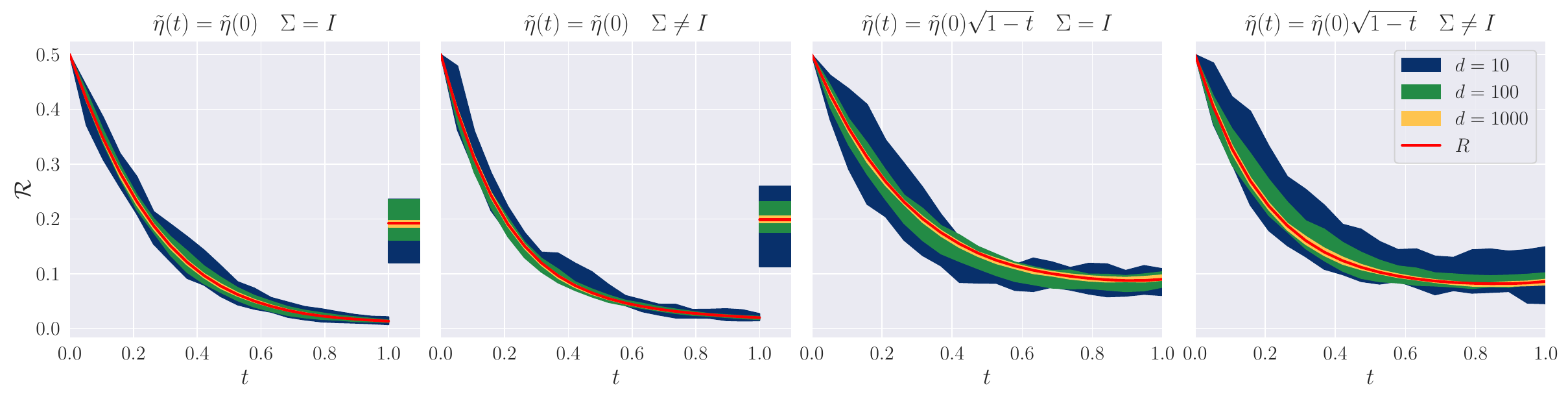}
    \caption{Numerical simulations for Algorithm \ref{alg:dp-sgd} ($d \in \{10, 100, 1000\}$) and the solution $R(t)$ of the system of ODEs in \eqref{eq:ODEi} and \eqref{eq:deteq}. We consider the two schedules $\tilde \eta(t) = \tilde \eta(0)$ (first and second panel) and $\tilde \eta(t) = \tilde \eta(0) \sqrt{1 - t}$ (third and fourth panel).
    We fix $\gamma = 0.1$, $\rho = 1$, $\zeta = 0.3$, $(\theta^*_i)^2 = 1/d$, $\tilde \eta(0) = 3$, $c = 1$, and consider both isotropic data ($\Sigma = I$, first and third panel) and data with diagonal covariance whose  eigenvalues are uniformly distributed in the interval $[0, 2]$ (second and fourth panel).
    For $\alpha = 0$, we also report for $t \geq 1$ the risk $\mathcal R(\theta^p)$, and the value of $R(1) + 2 c^2 \tilde \eta^2(1) \gamma^2 / \rho^2$
    with a red line.
    For each value of $d$, we report bands corresponding to 1 standard deviation around the mean over 10 independent trials of Algorithm \ref{alg:dp-sgd}.}
    \vspace{0.5cm}
    \label{fig:ODElikeAlg}
\end{figure}


\paragraph{Noise in the last iteration.} Theorem \ref{thm:deteq} holds for 
$\{\theta_k\}_{k=1}^{n-1}$, as the supremum in \eqref{eq:deteq1} is taken on the open interval $t \in [0, 1)$, and in general the equivalence \emph{does not hold} for the last iterate $\theta_n$ (corresponding to $t = 1$), which gives the (private) output of the algorithm. 
Intuitively, for $k < n$, \eqref{eq:schedulesdiff} suggests (with abuse of notation)
\begin{equation}\label{eq:heuristic}
    \rho^2 \sigma^2_k = - \frac{\diff}{\diff k} \eta^2_k = - \frac{1}{n} \frac{\diff}{\diff \left(k / n \right)} \frac{\tilde \eta^2 \left(k / n \right)}{n^2} \approx \frac{\rho^2}{n^3} \tilde \sigma^2(k / n),
\end{equation}
where in the second step we use Assumption \ref{ass:learning_rate_schedules} and $\tilde \sigma$ is given by Definition \ref{def:deteq}. In contrast, for the last iterate, again due to \eqref{eq:schedulesdiff}, we have $\rho^2 \sigma^2_n = \eta_n^2 = \tilde \eta^2(1) / n^2$. The additional $n$ factor in the denominator of \eqref{eq:heuristic} implies that, if $\tilde \eta(1) > 0$, the noise added to the last iterate is much larger. Thus, we treat the case separately via the result below (proved in Appendix \ref{app:deteq}), which quantifies the final test loss of $\theta_n=\theta^p$.


\begin{proposition}\label{prop:laststeprisk}
    Let Assumptions \ref{ass:data} and \ref{ass:learning_rate_schedules} hold.  Let $\rho = \Theta(1)$ and $n, d \to \infty$ s.t.\ $\gamma_n = d / n \to \gamma \in (0,\infty)$. Then, with overwhelming probability, we have
    \begin{equation}
        \left| \mathcal R(\theta^p) - R(1) - \frac{2 c^2 \tilde \eta^2(1) \gamma_n^2}{\rho^2} \right| = O \left(  \frac{ \log^2 n }{\sqrt n} \right).
    \end{equation}
    Furthermore, we also have
    \begin{equation}
        \left| \E \left[ \mathcal R(\theta^p) \right] - R(1) - \frac{2 c^2 \tilde \eta^2(1) \gamma_n^2}{\rho^2} \right| = O \left(  \frac{ \log^2 n  }{\sqrt n} \right).
    \end{equation}
\end{proposition}


The convergence predicted by Theorem \ref{thm:deteq} and Proposition \ref{prop:laststeprisk} is already evident at moderate values of $n, d$, as showcased by Figure \ref{fig:ODElikeAlg} for different schedules ($\tilde \eta(t) = \tilde \eta(0)$ and $\tilde \eta(t) = \tilde \eta(0) \sqrt{1 - t}$) and different data covariances. In the panels for $\tilde \eta(t) = \tilde \eta(0)$ of Figure \ref{fig:ODElikeAlg}, we report in a short interval at $t \geq 1$ the risk $\mathcal R(\theta^p)$, after the noise in the last iteration is added, since $\tilde \eta(1) > 0$. 


\section{Optimal rates}\label{sec:optimalrates}

Putting together Theorem \ref{thm:deteq} and Proposition 
\ref{prop:laststeprisk} gives that the final risk $\mathcal R(\theta^p)$ asymptotically converges to a constant independent of $n, d$ and only dependent on $\gamma$. This is in sharp contrast with earlier upper bounds by  \cite{varshney2022nearly, liu2023near, brown2024private} which diverge in the proportional regime ($n$ scaling linearly in $d$) considered here. 
Notably, this is a consequence of taking $\clip=c\sqrt{d}$ (see Assumption \ref{ass:learning_rate_schedules}) with $c$ a constant independent of $n, d$, rather than $\clip = \omega(\sqrt d)$.

We now show that the characterization put forward in Section \ref{sec:DP-GD} is precise enough to capture how 
$\mathcal R(\theta^p)$ depends on $\gamma$ and $\rho$, as well as on the algorithm hyper-parameters $\tilde \eta(t)$ and $c$. While \cite{dwork2025} also derive a result which does not diverge in the proportional regime, their analysis does not draw statistical implications on how the final risk depends on sample size and privacy requirement. One significant challenge towards this goal is that, in the framework of \cite{dwork2025}, the trajectory of DP-GD is expressed via complex DMFT equations, which are then hard to analyze precisely.

Throughout the section, we focus on the non-asymptotic behavior of $\mathcal R(\theta^p)$ as $\gamma \to 0$. 
Importantly, the limit $\gamma\to 0$ is taken \emph{after} the limit $d, n \to \infty$, which informally means that $d$ and $n$ are incomparably larger than $1 / \gamma$. Thus, our bounds on $\mathcal R(\theta^p)$ neglect the smaller terms (vanishing as $d, n \to \infty$) from Theorem \ref{thm:deteq} and Proposition \ref{prop:laststeprisk}, including the difference $| \gamma - \gamma_n | = o(1)$.
We assume that $\lambda_{\min}, \lambda_{\max}, \norm{\theta^*}_2, \zeta = \Theta_\gamma(1)$, i.e., the condition number of the covariance, the test risk of the model at initialization and the label noise variance are strictly positive constants independent of $\gamma$; instead, the privacy parameter $\rho$ is allowed to depend on $\gamma$ (but not on $n, d$). 




\subsection{Output perturbation and constant private noise}

In this section, we consider 
the family of learning rate schedules
\begin{equation}\label{eq:polynomialsched}
    \tilde \eta(t) = \tilde \eta(0) (1 - t)^\alpha,
\end{equation}
for $\alpha = 0$, $\alpha = 1 / 2$ and $\alpha \geq 1$. Taking $\alpha = 0$ corresponds to output perturbation: the learning rate is fixed during the whole training, and the private noise is added only at the end of the algorithm, as \eqref{eq:schedulesdiff} implies $\sigma_k = 0$ for $k \in \{1, \ldots, n-1\}$ and $\sigma_n = \eta_1 / \rho = \tilde \eta(0) / (n \rho)$. 
Taking $\alpha = 1 / 2$ corresponds to a linearly decaying $\tilde \eta^2(t)$, which in turn implies a constant level of noise $\sigma_k = \eta_1 / (\sqrt n \rho) = \tilde \eta(0) / (n^{3/2} \rho)$ in the iterations of DP-GD. These two choices of learning rate schedule are analyzed in Proposition \ref{prop:alpha012}. At the end of the section, we also study the effect of taking $\alpha \geq 1$ in Proposition \ref{prop:alphabigger1body}, so that the private noise decays with time. Note that all these values of $\alpha$ are in agreement with Assumption \ref{ass:learning_rate_schedules}.


\begin{proposition}\label{prop:alpha012}
    Let Assumptions \ref{ass:data} and \ref{ass:learning_rate_schedules} hold. Let $\theta^p_0$ and $\theta^p_{1/2}$ be the solutions obtained with Algorithm \ref{alg:dp-sgd} with $\tilde \eta(t)$ given by \eqref{eq:polynomialsched} for $\alpha = 0$ and $\alpha = 1/2$ respectively, with $\tilde \eta(0) \leq 2 / \gamma$. Consider 
    $\lambda_{\min}, \lambda_{\max}, \norm{\theta^*}_2, \zeta = \Theta_\gamma(1)$, and $\rho = \Omega_{\gamma}\left( \gamma^{1 - h} \right)$ for some $h > 0$.
    Pick 
    \begin{equation}\label{eq:hyperparamsalpha012}
         c = O_\gamma(1), \qquad \tilde \eta(0) c = C \ln(1/\gamma),
    \end{equation}
    for a large enough constant $C$ which does not depend on $\gamma$. 
    Then, we have that, with overwhelming probability,
    \begin{equation}\label{eq:upperboundalg012}
    \begin{split}
        \mathcal R(\theta^p_0) &= O_\gamma \left( \gamma \ln(1 / \gamma) + \frac{\gamma^2 \ln^2 (1 / \gamma)}{\rho^2} \right),\\ \mathcal R(\theta^p_{1/2}) &= O_\gamma \left( \gamma \ln^{2/3}(1 / \gamma) + \frac{\gamma^2 \ln^{4 / 3}(1 / \gamma)}{\rho^2} \right).
        \end{split}
    \end{equation}
    Furthermore, for any choice of the hyper-parameters $c$ and $\tilde \eta(0)$, we have that
    \begin{equation}\label{eq:lowerboundalg012}
    \begin{split}
         \mathcal R(\theta^p_0) &= \Omega_\gamma \left( \gamma \ln(1 / \gamma) + \frac{ \gamma^2 \ln^2(1 / \gamma)}{\rho^2} \right), \\ \mathcal R(\theta^p_{1/2}) &= \Omega_\gamma \left( \gamma \ln^{2 / 3}(1 / \gamma) + \frac{ \gamma^2 \ln^{4/3}(1 / \gamma)}{\rho^2 } \right).
         \end{split}
    \end{equation}
\end{proposition}

Proposition \ref{prop:alpha012} (proved in Appendix \ref{app:ODEs}, also in the setting where $\lambda_{\max}$ and $\lambda_{\min}$ can depend on $\gamma$) gives upper and lower bounds for output perturbation ($\alpha=0$) and DP-GD with constant noise ($\alpha=1/2$).
This result has two remarkable consequences: \emph{(i)} the hyper-parameters in \eqref{eq:hyperparamsalpha012} are optimal in terms of rate, and \emph{(ii)} DP-GD with constant noise outperforms output perturbation when $\gamma$ is sufficiently small. After giving a proof sketch, we elaborate on these two points in the remainder of the section. 

\paragraph{Proof ideas for Proposition \ref{prop:alpha012}.}
For simplicity, we consider here the isotropic case (the case with general covariance is treated in the full argument in Appendix \ref{app:ODEs}). 
The ODE in \eqref{eq:isotropicODE} reads
\begin{equation}
    \diff R(t) = - 2 \tilde{\eta}(0) c \frac{\mu_c(R(t))}{c} R(t) \diff t + (\tilde{\eta}(0) c)^2 \frac{\nu_c(R(t)) (R(t) + \zeta^2/2)}{c^2} \gamma \diff t.
\end{equation}
If $c = O(1)$, the bounds on $\mu_c(R)$ and $\nu_c(R)$ in Lemma \ref{lemma:munubounds} yield that, 
uniformly in $t \in [0, 1]$, 
\begin{equation}
    \frac{\mu_c(R(t))}{c} = \Theta_\gamma(1), \qquad \frac{\nu_c(R(t)) (R(t) + \zeta^2/2)}{c^2} = \Theta_\gamma(1).
\end{equation}
Thus, due to ODE comparison arguments, we can characterize $R(t)$ studying the closed-form solution of the ODE
\begin{equation}\label{eq:closedformODEalpha0}
    \diff \tilde R(t) = - \tilde{\eta}(0) c \tilde R(t) \diff t + (\tilde{\eta}(0) c)^2 \gamma \diff t,
\end{equation}
with $\tilde R(0) = R(0)$, i.e., $\tilde R(t) = \left( R(0) - \gamma \tilde{\eta}(0) c \right) e^{- \tilde{\eta}(0) c t} + \gamma \tilde{\eta}(0) c$. Then, setting $t = 1$, Proposition \ref{prop:laststeprisk} gives that the risk is well approximated by
\begin{equation}
    \left( R(0) - \gamma \tilde{\eta}(0) c \right) e^{- \tilde{\eta}(0) c} + \gamma \tilde{\eta}(0) c + \frac{(\tilde{\eta}(0) c)^2 \gamma^2}{\rho^2}.
\end{equation}
Minimizing over $\tilde{\eta}(0) c$ gives the desired result.

If $c = \Omega(1)$, the bounds in Lemma \ref{lemma:munubounds} yield
\begin{equation}
    \mu_c(R(t)) = \Theta_\gamma(1), \qquad \nu_c(R(t)) (R(t) + \zeta^2/2) = \Omega_\gamma(1).
\end{equation}
Then, $R(t)$ is lower bounded (up to constants independent of $\gamma$) by the solution to the ODE
\begin{equation}\label{eq:bodycomega1}
    \diff \tilde R(t) = - \tilde{\eta}(0) \tilde R(t) \diff t + \tilde{\eta}(0)^2 \gamma \diff t.
\end{equation}
This ODE has the same form as in \eqref{eq:closedformODEalpha0} (after replacing $\tilde \eta(0) c$ with just $\tilde \eta(0)$) and, therefore, it does not give a lower final risk.

The case $\alpha = 1/2$ follows a similar strategy. In particular, if $c = O(1)$, $R(t)$ can be characterized via the ODE
\begin{equation}\label{eq:comparisonbodyalpha12}
    \diff \tilde R(t) =  - \tilde{\eta}(0) c \sqrt{1-t} \tilde R(t) \diff t + (\tilde{\eta}(0) c)^2 \gamma (1-t) \diff t + \frac{(\tilde{\eta}(0) c)^2 \gamma^2}{\rho^2} \diff t,
\end{equation}
with $\tilde R(0) = R(0)$.
The solution to \eqref{eq:comparisonbodyalpha12} can also be written in closed-form (see \eqref{eq:closedformODEalpha12}). 
Thus, setting $\tilde \eta(0) c$ according to \eqref{eq:hyperparamsalpha012} allows to obtain the improved upper bound.



\paragraph{Optimal hyper-parameters.}
We now comment on the optimal hyper-parameter choice in \eqref{eq:hyperparamsalpha012}.
First, notice that if $\tilde \eta(0) > 2 / \gamma$, the gradient update would be proportional to $\bar \eta_k < \eta_k$, while the private noise $\sigma_k$ still depends on $\eta_k$ via \eqref{eq:schedulesdiff}. This suggests the sub-optimality of having $\bar \eta_k < \eta_k$ and of the regime $\tilde \eta(0) > 2 / \gamma$.
Second, the choice $\tilde \eta(0) c = C \ln(1/\gamma)$ provides the optimal trade-off between 
two competing objectives:
on the one hand, $\tilde \eta(t) c$ controls the size of the first term of the ODEs in \eqref{eq:upperlowerODEsbodynew}
(for $c = O_\gamma(1)$ and bounded values of $R$, Lemma \ref{lemma:munubounds} gives that $\mu_c(R) / c$ is lower bounded by a constant), which in turn determines the speed of convergence towards $0$ of the risk;
on the other hand, a large product $\tilde \eta(0) c$ increases the last two terms of the ODEs, which have the opposite effect of increasing the risk.
Third, the choice $c = O_\gamma(1)$ is motivated by the fact that increasing $c$ beyond this point does not further increase $\mu_c(R) < 1$, which drives the risk to $0$. However, larger values of $c$ increase the private noise in DP-GD, and hence the last term in the RHSs of \eqref{eq:upperlowerODEsbodynew}, which increases the risk. 

These conclusions are supported by Figures \ref{fig:heatmaps}, \ref{fig:heatmap_mnist} and \ref{fig:heatmap_housing}: performance deteriorates if either $c$ exceeds $1$ (upper part of the heatmaps) or $\tilde \eta(0)$ exceeds $2 / \gamma$ (right part of the heatmaps); the lowest values of the risk are roughly parallel to the line $c \tilde \eta(0) = \ln(1 / \gamma)$. A more detailed discussion of the numerical experiments is in Section \ref{sec:numerical}.

\paragraph{Benefits of aggressive clipping.} 
Adding the further condition $\rho / c = \Omega(\gamma^{1 - h})$ for some $h > 0$, the lower bounds in Proposition \ref{prop:alpha012}  
can be improved to 
\begin{equation}\label{eq:suboptimalitycbody}
    \mathcal R(\theta^p_{0, 1/2}) = \tilde \Omega_\gamma \left(\gamma + \frac{\max(1, c^2) \gamma^2}{\rho^2}\right),
\end{equation}
where $\tilde \Omega_\gamma$ neglects poly-logarithmic factors. The term $\max(1, c^2)$ in the second quantity at the RHS of \eqref{eq:suboptimalitycbody} quantifies the negative effects on performance due to using a large clipping constant $c = \omega_\gamma(1)$. This is also evident from the numerical results in Figures \ref{fig:heatmaps}, \ref{fig:heatmap_mnist}, and \ref{fig:heatmap_housing}.
We note that the requirement $\rho / c = \Omega(\gamma^{1 - h})$ is analogous to the lower bound for $\rho$ in the statement of Proposition \ref{prop:alpha012}, and it guarantees that the second term in the RHS of  \eqref{eq:suboptimalitycbody} is $o_\gamma(1)$, making it comparable to the first term in the RHS. 

The lower bound in \eqref{eq:suboptimalitycbody} is formally discussed at the end of the proof of Proposition \ref{prop:alpha012} in Appendix \ref{app:proofalpha012}, and it informally follows from the argument that leads to \eqref{eq:bodycomega1}, which shows that the first two terms in the RHS of the ODE in \eqref{eq:isotropicODE} are qualitatively identical in the regimes $c = O_\gamma(1)$ and $c = \Omega_\gamma(1)$. In the first case ($c = O_\gamma(1)$), these two terms are proportional to $\tilde \eta(t) c$ and $(\tilde \eta(t) c)^2$; in the second case ($c = \Omega_\gamma(1)$), they are proportional to $\tilde \eta(t)$ and $\tilde \eta(t)^2$. Thus, with a change of variable, \eqref{eq:isotropicODE} behaves as the following ODE:
\begin{equation}
    \diff \tilde R(t) = - v(t) \tilde R(t) \diff t + v^2(t) \gamma \diff t + \frac{\max(1, c^2)}{\rho^2} \left( - \frac{\diff v^2(t)}{\diff t} \right) \gamma^2 \diff t,
\end{equation}
where $v(t) = \tilde \eta(t) c$ in the case $c = O_\gamma(1)$, and $v(t) = \tilde \eta(t)$ in the case $c = \Omega_\gamma(1)$.




\paragraph{Benefits of decaying the learning rate.} 
Motivated by the improvement of DP-GD with constant noise over output perturbation (Proposition \ref{prop:alpha012}), we consider reducing the noise added at the end of training, i.e., we pick a polynomial schedule as in \eqref{eq:polynomialsched} with $\alpha \ge 1$. This implies that $\tilde \sigma^2(t)$ is proportional to $2 \alpha (1 - t)^{2\alpha - 1}$, which corresponds to decaying the noise at least linearly.

\begin{proposition}\label{prop:alphabigger1body}
    Let Assumptions \ref{ass:data} and \ref{ass:learning_rate_schedules} hold. Let $\theta^p_\alpha$ be the solution obtained with Algorithm \ref{alg:dp-sgd}, with $\tilde \eta(t)$ given by \eqref{eq:polynomialsched} for $\alpha \geq 1$, with $\tilde \eta(0) \leq 2 / \gamma$. Consider $\lambda_{\min}, \lambda_{\max}, \norm{\theta^*}_2, \zeta = \Theta_\gamma(1)$, and 
    \begin{equation}\label{eq:extraassply}
        \ln^2(1 / \gamma) \gamma \left( \alpha + \frac{\gamma}{\rho^2} \right) = o_\gamma(1).
    \end{equation}
    Pick
    \begin{equation}\label{eq:hyperparamsalphalarger1}
        c = O_\gamma(1), \qquad \tilde \eta(0) c = C \alpha \ln(1/\gamma), 
    \end{equation}
    for a large enough constant $C$ which does not depend on $\gamma$. 
    Then, we have that, with overwhelming probability,
    \begin{equation}
        \mathcal R(\theta^p_\alpha) = O_\gamma \left(\alpha \gamma \ln^{\frac{1}{1+\alpha}}(1/\gamma) + \frac{\alpha^{2} \gamma^2 \ln^{\frac{2}{1+\alpha}}(1/\gamma)}{\rho^2} \right). 
    \end{equation}
\end{proposition}


Proposition \ref{prop:alphabigger1body} (proved in Appendix \ref{sec:bigalpha}, also in the setting where $\lambda_{\max}$ and $\lambda_{\min}$ can depend on $\gamma$) provides upper bounds on the final risk as the degree $\alpha$ of the polynomial schedule changes, matching the ones in Proposition \ref{prop:alpha012} for $\alpha = \{ 0, 1/2 \}$. 
We note that the condition in \eqref{eq:extraassply} allows for values of $\alpha$ as large as $1 / \gamma$ (neglecting poly-logarithmic factors), and it imposes a similar lower bound for $\rho$ as the one in Proposition \ref{prop:alpha012}.

An immediate consequence of Proposition \ref{prop:alphabigger1body} is that, 
by taking $\alpha= \ln \ln (1/\gamma)$, 
one has
\begin{equation}\label{eq:boundopts2}
    \mathcal R(\theta^p) =  O_\gamma \left(\gamma \ln\ln(1/\gamma) + \frac{\gamma^2}{\rho^2}(\ln\ln(1/\gamma))^{2} \right).
\end{equation}





Thus, for sufficiently small $\gamma$, it is convenient to increase $\alpha$ and decay the noise faster during training, up to a level $\alpha = \Theta_\gamma(\ln \ln(1 / \gamma))$. Values of $\alpha$ larger than that may then deteriorate performance.
This is illustrated in Figure \ref{fig:schedules} (discussed in detail in Section \ref{sec:numerical}), which compares different schedules after the hyper-parameters $\tilde \eta(0)$ and $c$ have been optimized numerically.


\subsection{Harmonic schedule and minimax rates}


Next, we consider a harmonically decreasing schedule of the form $\tilde \eta(t) = \beta / (t + \tau)$.

\begin{theorem}\label{thm:harmonicbody}
    Let Assumptions \ref{ass:data} and \ref{ass:learning_rate_schedules} hold. Consider 
    $\lambda_{\min}, \lambda_{\max}, \norm{\theta^*}_2, \zeta = \Theta_\gamma(1)$, and $\gamma^2 / \rho^2 = o_\gamma(1)$. Let $\theta^p_h$ be the solution of Algorithm \ref{alg:dp-sgd} with learning rate schedule
    \begin{equation}\label{eq:harmonicschedule}
        \tilde \eta(t) = \frac{\beta}{t + \tau}.
    \end{equation}
    Then, there exist values of $c$, $\tau$ and $\beta$ such that, with overwhelming probability,
    \begin{equation}\label{eq:optimalupperbound}
        \mathcal R(\theta^p_h) = O_\gamma \left( \gamma + \frac{\gamma^2}{\rho^2} \right),
    \end{equation}
    with the same upper-bound also holding for $\E \left[ \mathcal R(\theta^p) \right]$.
\end{theorem}

This result shows that the harmonically decreasing schedule $\tilde \eta(t) = \beta / (t + \tau)$, for an appropriate choice of $\beta$ and $\tau$, allows to get rid of the $\ln \ln (1 / \gamma)$ factors in \eqref{eq:boundopts2}, guaranteeing the upper bound in \eqref{eq:optimalupperbound}.
More precisely, $\beta$ and $c$ are chosen as constants independent of $\gamma$, but dependent on $\lambda_{\min}, \lambda_{\max}, \norm{\theta^*}_2, \zeta$, while $\tau$ scales with the maximum between $\gamma$ and $\gamma^2 / \rho^2$. The exact values are provided in the proof of Theorem \ref{thm:harmonicbody} in Appendix \ref{app:harmonic}, which also covers the setting where $\lambda_{\max}$ and $\lambda_{\min}$ can depend on $\gamma$. 

\paragraph{Proof ideas for Theorem \ref{thm:harmonicbody}.} 
To give an intuition of the benefits of a harmonically decreasing schedule, let us consider the simplified ODE
\begin{equation}\label{eq:closedformODEforharmonic}
    \diff \tilde R(t) = - \tilde{\eta}(t) c \tilde R(t) \diff t + (\tilde{\eta}(t) c)^2 \gamma \diff t.
\end{equation}
We note the resemblance between \eqref{eq:closedformODEforharmonic} and \eqref{eq:closedformODEalpha0}, which we used to study the constant learning rate schedule. Heuristically, to minimize $\tilde R(1)$, we wish to minimize the RHS point-wise in $t$. Setting to $0$ the derivative of the RHS with respect to $\tilde \eta(t) c$, we get
\begin{equation}\label{eq:Rtoscheduleharm}
    \tilde R(t) = 2 \tilde \eta(t) c \gamma.
\end{equation}
Plugging this into \eqref{eq:closedformODEforharmonic}, we obtain
\begin{equation}
    \diff \tilde R(t) = - \frac{\tilde R^2(t)}{4 \gamma} \diff t \quad \Rightarrow \quad R(t) = \frac{4 \gamma}{t + 4 \gamma / \tilde R(0)},
\end{equation}
which together with \eqref{eq:Rtoscheduleharm} gives the harmonic schedule $\tilde \eta(t) c = \frac{2}{t + 4 \gamma / \tilde R(0)}$.

Consider now the original ODE in \eqref{eq:isotropicODE}. Via the same comparison arguments used to obtain \eqref{eq:comparisonbodyalpha12}, when $c = O_\gamma(1)$ this ODE can be reduced to
\begin{equation}
\begin{split}    
    \diff \hat R(t) &= - \tilde \eta(t) c \hat R(t) \diff t + (\tilde \eta(t) c)^2 \gamma \diff t - \frac{\diff (\tilde \eta(t) c)^2}{\diff t} \frac{\gamma^2}{\rho^2} \diff t,
\end{split}
\end{equation}
with initial condition $\hat R(0) = R(0)$. Then, we have that (see Lemma \ref{lemma:ODEclosed})
\begin{equation}
    \hat R (t) = e^{- F(t)}\left( R(0) + \int_0^t e^{F(s)} \left( \gamma (\tilde \eta(s) c)^2 - 2 \frac{\gamma^2}{\rho^2} \tilde\eta(s) \tilde \eta'(s) c^2 \right)\diff s \right),
\end{equation}
where $\tilde \eta'(t)$ denotes the derivative of $\tilde \eta(t)$ and $F(t) = \int_0^t \tilde \eta(s) c \, \diff s$. Relying on the previous harmonic heuristic and carrying out explicit computations after plugging the schedule in \eqref{eq:harmonicschedule} yields the desired result.




\paragraph{Minimax rates.} Given the result of Theorem \ref{thm:harmonicbody}, it is natural to question if a different choice of schedule, or even a completely different choice of algorithm could improve the rates in \eqref{eq:optimalupperbound}. In the following, we prove that this is not the case.

\begin{theorem}\label{thm:lowerbound}
    Let Assumption \ref{ass:data} hold. Consider 
    $\lambda_{\min}, \lambda_{\max}, \zeta^2 = \Theta_\gamma(1)$ and $\gamma^2 / \rho^2 = O_\gamma(1)$.
    Let $\theta^p$ be the solution found by a generic algorithm $\mathcal M$ belonging to the set of all $\rho^2 / 2$-zCDP algorithms $\mathbb M$.
    Then, we have
    \begin{equation}\label{eq:minimaxratebody}
        \inf_{\mathcal M \in \mathbb M} \sup_{\norm{\theta^*}_2 < 1} \E \left[ \mathcal R(\theta^p) \right] = \Omega_\gamma \left( \gamma + \frac{\gamma^2}{\rho^2} \right).
    \end{equation}
\end{theorem}

The proof of Theorem \ref{thm:lowerbound} follows the same tracing attack technique considered in \cite{cai21}. 
We note that Theorem 4.1 in \cite{cai21} is stated for $(\varepsilon, \delta)$ approximate DP and, if applied directly to our setting, would yield an extra $\log n$ at the denominator of the privacy term. To avoid this, our minimax rates concern the set of $\rho^2/2$-zCDP algorithms, as per Definition \ref{def:zcdp}. This gives a minimax lower bound matching the upper bound of Theorem \ref{thm:harmonicbody}.

\paragraph{Proof ideas for Theorem \ref{thm:lowerbound}.} We define the random variable $L_i = p_{\mathcal{M}(D)}({\theta}) / p_{\mathcal{M}(D'_i)} ({\theta})$, 
where $p_{\mathcal{M}(D)}$ and $p_{\mathcal{M}(D'_i)}$ are the probability density functions of the parameters $\mathcal{M}(D)$ and $\mathcal{M}(D'_i)$, with $D$ and $D_i$ being neighboring datasets differing in their $i$-th entry.
Due to the definition of zCDP (see Definition \ref{def:zcdp}), we have that, for all 
$\alpha > 1$,
\begin{align}\label{eq:Lbody}
    \E\left[ L^\alpha_i \right] \leq \exp \left( \frac{\alpha (\alpha - 1) \rho^2 }{2}\right).
\end{align}
Consider now the tracing attack
\begin{align}
    A_i(\theta) = z_i x_i^\top \left(\theta - \theta^* \right).
\end{align}
Then, in Lemma \ref{lemma:omegad} (similar to Lemma 4.1 in \cite{cai21}) we show that 
there exists a prior $\pi$ on $\theta^*$ with support in $\norm{\theta^*}_2 < 1$ such that, for any $\rho^2 / 2$-zCDP algorithm $\mathcal M$,
\begin{equation}
    \E_\pi \left[ \sum_{i \in [n]} \E \left[ A_i (\mathcal M(D)) \right] \right] = \Omega(d).
\end{equation}
In Lemma \ref{lemma:privateminimax} we show that this bound, together with \eqref{eq:Lbody} for $\alpha = 2$, gives
\begin{equation}\label{eq:infbody}
    \inf_{\mathcal M \in \mathbb M} \E_{\pi, \mathcal M, X, Y} \left[ \mathcal R(\mathcal M(X, Y)) \right] = \Omega \left( \frac{d^2}{n^2 \left( e^{\rho^2} - 1 \right)} \right).
\end{equation}
Merging \eqref{eq:infbody} with the Bayes risk obtained in Lemma \ref{lemma:nonprivateminimax} gives the desired result, as the privacy cost dominates only for $\rho = O_\gamma(1)$ such that $\rho^2 = \Theta_\gamma(e^{\rho^2} - 1)$. The full argument is deferred to Appendix \ref{app:lowerbound}.

\section{Scaling laws}\label{sec:scalinglaws}

Let us now focus on the ill-conditioned setting, i.e., when the condition number of the data covariance $\Sigma$ depends on the number of dimensions $d$. In this setting, the bounds in \eqref{eq:upperlowerODEsbodynew} are not tight enough to describe the behavior of Algorithm \ref{alg:dp-sgd}, as $\lambda_{\max}/\lambda_{\min}$ diverges with $d$. Therefore, we analyze
the full system of ODEs in \eqref{eq:ODEi} and \eqref{eq:deteq}.
In particular, $D_i(t)$ has the following implicit solution
\begin{equation}\label{eq:ODEi_sol}
     D_i (t) = D_i(0)e^{- 2 \lambda_i\Gamma(t)} + \int_0^t \left(  \lambda_i  \bar \eta^2(s) \nu_c(R(s)) (R(s) + \zeta^2/2) \gamma + 2 c^2 \tilde \sigma^2(s) \gamma^2 \right) e^{- 2 \lambda_i(\Gamma(t)-\Gamma (s)) }\diff s,
\end{equation}
where we introduced the shorthand $\Gamma(t) = \int_0^t \bar \eta(s) \mu_c(R(s)) \diff s$. 
Multiplying by $\lambda_i$ both sides of \eqref{eq:ODEi_sol}, averaging over $i \in [d]$, and using \eqref{eq:deteq}, we obtain the following implicit equation for the risk $R(t)$: 
\begin{equation}\label{eq:implicitRt}
     R (t) = \cF(\Gamma(t)) + \int_0^t \left( \bar \eta^2(s) \nu_c(R(s)) (R(s) + \zeta^2/2) \gamma \mathcal{K}(\Gamma(t)-\Gamma (s))  + 2 c^2 \tilde \sigma^2(s) \gamma^2 \mathcal J (\Gamma(t)-\Gamma (s)) \right) \diff s,
\end{equation}
where we introduced the shorthands
\begin{equation}\label{eq:cFcK}
\cF(x) := \frac{1}{d} \sum_{i=1}^d D_i(0)\lambda_ie^{- 2 \lambda_i x}, \qquad \mathcal{K}(x) := \frac{1}{d} \sum_{i=1}^d \lambda_i^2 e^{- 2 \lambda_i x}, \qquad \mathcal{J}(x) := \frac{1}{d} \sum_{i=1}^d \lambda_i e^{- 2 \lambda_i x}.
\end{equation}

We assume a power-law distribution for the covariance spectrum. 
\begin{assumption}\label{ass:power_law_1}
   $\Sigma$ has a spectrum that converges weakly as $d \to \infty$ to the power law measure $p(\lambda) = (1 - \phi) C_\phi^{\phi - 1} \lambda^{-\phi}\mathbf{1}_{(0, C_\phi)}$, with $C_\phi = \frac{2 - \phi}{1 - \phi}$, for some $\phi < 1$. Furthermore, for some $\psi < 1 - \phi$, we have that $D_i(0) = d \left(\omega_i^\top \theta^* \right)^2 / 2 = \Theta_\gamma(\lambda_i^{-\psi}) $ uniformly for $i \in [d]$, with $\limsup_{d \to \infty} \sum \lambda_i^{-\psi} / d < + \infty$.
\end{assumption}

In words, $\phi$ regulates the distribution of the spectrum. For $\phi = 0$, the spectrum is uniformly distributed in the interval $[0, 2]$; as $\phi \to 1$ more eigenvalues are concentrated towards 0; and as $\phi\to-\infty$, the covariance approaches the identity. The multiplicative factor $C_\phi$ is to ensure that the density $p(\lambda)$ is normalized and that $\tr(\Sigma) / d \to 1$, in agreement with Assumption \ref{ass:data}.
The coefficient $\psi$ describes how the projections of $\theta^*$ on the eigenvectors $\omega_i$ of the covariance depend on the corresponding eigenvalues. For $\psi = 0$, their squares are distributed uniformly; as $\psi$ increases, more weight is on the directions where the eigenvalues are small; and as $\psi$ decreases, the opposite happens. The last condition is to guarantee $\norm{\theta^*}_2^2 < \infty$, and it directly implies the bound $\psi < 1 - \phi$.


We note that Assumption \ref{ass:power_law_1} differs from those considered in the kernel ridge regression literature, see \cite{caponnetto2007optimal, rudi2017generalization, defilippis2024dimension, paquette20244+}. 
Our assumption is instead inspired by \cite{collins2024high}, and it is better suited to the random matrix theory regime underlying Theorem \ref{thm:deteq}, where a deterministic equivalent is derived. More precisely, existing work on kernel ridge regression sets the eigenvalues of the covariance to $\lambda_i = i^{- \alpha}$, for $\alpha > 1$,  
which makes $\tr(\Sigma)$ converge to a quantity independent of $d$. This is in contrast with the requirement $\tr(\Sigma)=d$ in Assumption \ref{ass:data}, which 
in turn is used to prove the concentration argument in Theorem \ref{thm:deteq}. 

As in the previous section, we study the non-asymptotic behavior of $\mathcal R(\theta^p_\alpha)$ as $\gamma \to 0$ for the class of polynomial schedules $\tilde \eta(t) = \tilde \eta(0) (1 - t)^\alpha$ with $\tilde \eta(0) \leq 2 / \gamma$.
The following result (proved in Appendix \ref{app:scalinglaws}) considers a privacy parameter that scales  polynomially with $\gamma$, i.e., $\rho = \Theta_\gamma(\gamma^b)$, and it establishes the optimal risk obtained by optimizing the hyper-parameters $(c, \tilde\eta(0))$ for fixed $(\phi, \psi, \alpha, b)$. 

\begin{theorem}\label{thm:scalinglawsbody}
    Let Assumptions \ref{ass:data}, \ref{ass:learning_rate_schedules} and \ref{ass:power_law_1} hold. 
    Let $\theta^p_\alpha$ be the solution of Algorithm \ref{alg:dp-sgd}, with $\tilde \eta(t) = \tilde \eta(0) (1 - t)^\alpha$, where $\alpha = \{0 , 1/2\}$ and $\alpha \geq 1$ such that $\tilde \eta(0) \leq 2 / \gamma$. Let 
    $\rho = \Theta_\gamma(\gamma^b)$ with $b < 1$. Denote $K = (2 - \phi - \psi) (\alpha + 1)$, and set $c = O_\gamma(1)$ and $c \tilde \eta(0) = \gamma^a$. Then,
    \begin{itemize}
    \item if $\phi (\alpha + 1) < 2$ and $b \leq \frac{K}{2 (K+1)}$, set
    \begin{equation}\label{eq:scalinglawcase1}
        a = - \frac{\alpha+1}{K+1}, \qquad \text{and define} \qquad h = \frac{K}{K+1};
    \end{equation}
    
    \item If $\phi (\alpha + 1) < 2$ and $b > \frac{K}{2 (K+1)}$, set
    \begin{equation}\label{eq:scalinglawcase2}
        a = - \frac{2(1 - b)(\alpha + 1)}{K + 2}, \qquad \text{and define} \qquad h = \frac{2 K (1 - b)}{K + 2},
    \end{equation}

    \item If $\phi (\alpha + 1) \geq 2$ and $b \leq 1 - \frac{(2 - \psi)(\alpha + 1)}{2 (K + 1)}$, set $a$ and define $h$ as in \eqref{eq:scalinglawcase1};

    \item If $\phi (\alpha + 1) \geq 2$ and $b > 1 - \frac{(2 - \psi)(\alpha + 1)}{2 (K + 1)}$, set
    \begin{equation}\label{eq:scalinglawcase3}
        a = - \frac{2 (1 - b)}{2 - \psi}, \qquad \text{and define} \qquad  h = \frac{2 (2 - \phi - \psi) (1 - b)}{2 - \psi} + \frac{\ln(a \ln \gamma)}{\ln \gamma} \, \mathbf{1} \left\{ \phi (\alpha + 1) = 2\right\}.
    \end{equation}
    \end{itemize}
    For all of these cases, with overwhelming probability, we have $\mathcal R(\theta^p_\alpha) = \Theta_\gamma(\gamma^h)$.
    Furthermore, for all of these cases, with overwhelming probability, we have $\mathcal R(\theta^p_\alpha) = \Omega_\gamma(\gamma^h)$, for all choices of $c$ and all choices of $a$ independent of $\gamma$.
\end{theorem}

        



The scaling laws of Theorem \ref{thm:scalinglawsbody} highlight two regimes. On the one hand, if $b$ is sufficiently small (corresponding to a sufficiently large $\rho$), privacy is achieved without incurring a penalty in performance; in fact, in \eqref{eq:scalinglawcase1}, $h$ does not depend on $b$, which implies that the final risk with optimal hyper-parameter choice is independent of the privacy requirement. On the other hand, if $b$ exceeds a critical value (corresponding to a more stringent requirement in terms of $\rho$), then the cost of privacy dominates and the final risk depends on $b$. This is the case for both $\phi (\alpha + 1) < 2$ and $\phi (\alpha + 1) \geq 2$ (see respectively \eqref{eq:scalinglawcase2} and \eqref{eq:scalinglawcase3}), albeit the critical value of $b$ and the final risk change in their dependence on $(\phi, \psi, \alpha)$. We note that the critical value of $b$ ranges within the interval $b \in (1/4, 1/2)$ for $\phi (\alpha + 1) < 2$ and within the interval $b \in (0, 1/2)$ for $\phi (\alpha + 1) \geq 2$. Similarly to Proposition \ref{prop:alpha012}, Theorem \ref{thm:scalinglawsbody} shows that the optimal value of $c \tilde \eta(0)$ is an increasing function of $1 / \gamma$, since $a$ is always negative. Furthermore, given the admitted ranges of $\phi$ and $\psi$, we have that $a > -1$ in all cases, which guarantees the existence of $c = O_\gamma(1)$ such that $\tilde \eta(0) \leq 2 / \gamma$.

\begin{figure}
    \centering
    \includegraphics[width=\linewidth]{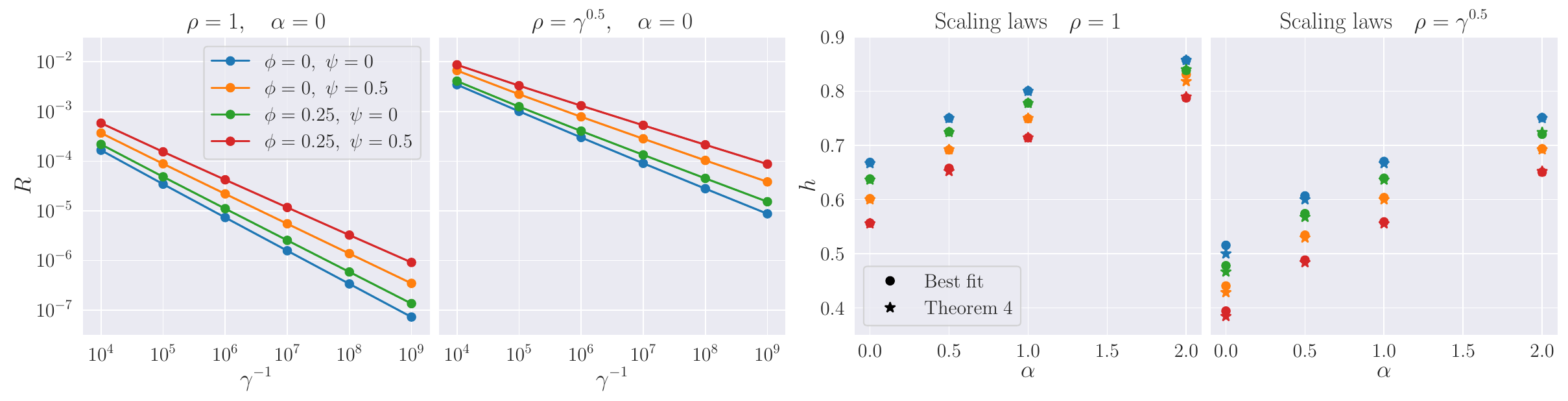}
    \caption{Numerical simulations for the system of ODEs in \eqref{eq:ODEi} and \eqref{eq:deteq}, for $d = 100000$, $\zeta = 0.3$ and $c = 0.1$. In the first and second panel, we report $R(1) + 2c^2 \tilde \eta^2(1) \gamma^2 / \rho^2$ for $\alpha = 0$, $\phi \in \{ 0, 0.25 \}$, $\psi \in \{ 0, 0.5 \}$, and $\rho \in \{1, \gamma^{0.5} \}$, after optimizing w.r.t.\ $\tilde \eta(0)$. In the third and fourth panel, we report the slope of the linear best fit in log-log scale for $\alpha \in \{0, 0.5, 1, 2 \}$ (circle marker), together with the values of $h$ given by Theorem \ref{thm:scalinglawsbody} (star marker).}
    \label{fig:scaling_laws}
\end{figure}


\begin{figure}
    \centering
    \includegraphics[width=\linewidth]{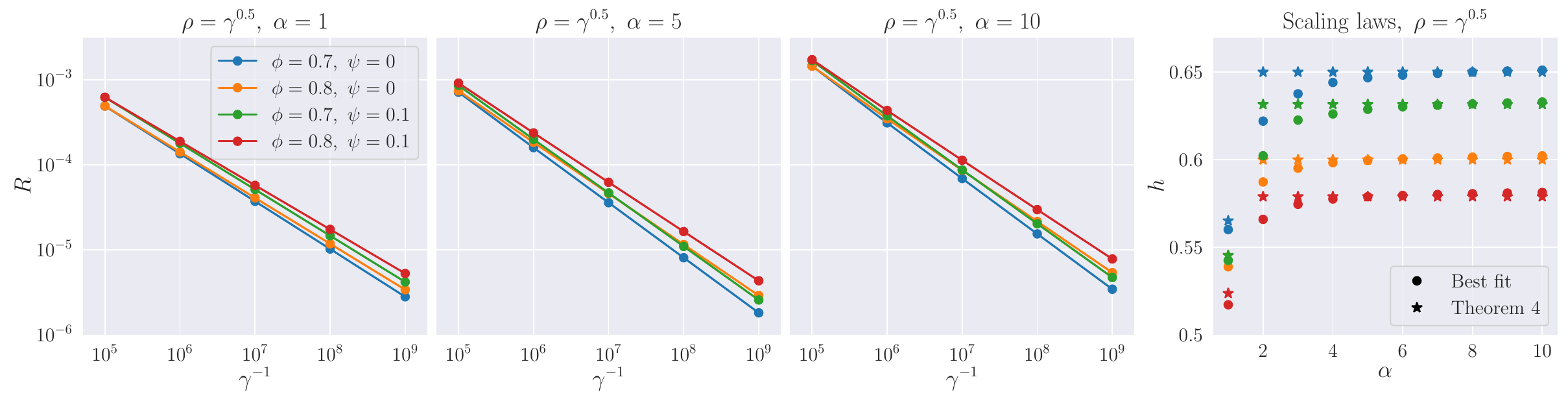}
    \caption{Numerical simulations for the system of ODEs in \eqref{eq:ODEi} and \eqref{eq:deteq}, for $d = 10000$, $\zeta = 0.3$ and $c = 0.1$. In the first three panels, we report $R(1) + 2c^2 \tilde \eta^2(1) \gamma^2 / \rho^2$ for $\alpha = 0$, $\phi \in \{ 0.7, 0.8 \}$, $\psi \in \{ 0, 0.1 \}$,  $\rho = \gamma^{0.5}$, after optimizing w.r.t.\ $\tilde \eta(0)$, for different schedules with $\alpha = \{ 1, 5, 10 \}$. In the fourth panel, we report the slope of the linear best fit in log-log scale for $\alpha \in \{1, \ldots, 10\}$ (circle marker), together with the values of $h$ given by Theorem \ref{thm:scalinglawsbody} (star marker).}
    \label{fig:scaling_laws_new}
\end{figure}

In Figures \ref{fig:scaling_laws} and \ref{fig:scaling_laws_new}, we numerically simulate the system of ODEs in \ref{eq:ODEi} and \eqref{eq:deteq} for different values of $(\phi, \psi,\alpha)$, as $\gamma$ decreases. We fix $c = 0.1$, optimize over $\tilde \eta(0)$, and report the minimum value of $R(1) + 2c^2 \tilde \eta^2(1) \gamma^2 / \rho^2$ for each $\gamma$. 
Then, we perform a linear best fit in log-log scale and report the slope as a function of $\alpha$. 
More precisely, 
in Figure \ref{fig:scaling_laws}, we consider small values of $(\phi, \psi, \alpha)$ and the privacy levels $\rho = 1$ and $\rho = \gamma^{0.5}$, which lead to the scaling laws in \eqref{eq:scalinglawcase1}-\eqref{eq:scalinglawcase2}. The simulation results (circles) agree well with the values of $h$ (stars) from Theorem \ref{thm:scalinglawsbody}, showing that our theory is predictive of the final risk with the optimal hyper-parameter choice. The plots also illustrate that $h$ increases with $K$ (as predicted by \eqref{eq:scalinglawcase1}-\eqref{eq:scalinglawcase2}), 
and therefore, it increases with both $\alpha$ and $2 - \phi - \psi$ ($h$ increases as the color changes from red to orange, green and blue). In Figure \ref{fig:scaling_laws_new}, we consider larger values of $\phi, \alpha$ and the privacy level $\rho = \gamma^{0.5}$, which lead to the scaling laws in \eqref{eq:scalinglawcase3}. For values of $\alpha$ large enough, the simulation results (circles) agree well with the values of $h$ (stars) from Theorem \ref{thm:scalinglawsbody}. We expect a better fit for small values of $\alpha$ if simulations were run for larger values of $\gamma^{-1}$, which is computationally more expensive. The plots illustrate how the value of $h$ saturates after a certain value of $\alpha$, in agreement with \eqref{eq:scalinglawcase3}.


We note that, for fixed $(\phi, \psi)$, $h$ is non decreasing in $\alpha$, suggesting that larger values of $\alpha$ improve performance when $\gamma$ is sufficiently small (see also the discussion after Proposition  \ref{prop:alphabigger1body}).  
This is formalized in the result below (also proved in Appendix \ref{app:scalinglaws}), which readily follows from Theorem \ref{thm:scalinglawsbody} after taking large enough  $\alpha$.


\begin{corollary}\label{cor:scalinglaws}
    Consider the setting of Theorem \ref{thm:scalinglawsbody}, and suppose $\phi \leq 0$. Then, for every $\epsilon > 0$, there exist $\alpha$, $c$ and $\tilde \eta(0)$ such that, with overwhelming probability,
    \begin{equation}
        \mathcal R(\theta^p_\alpha) = O_\gamma \left( \gamma^{1-\epsilon} + \left( \frac{\gamma^2}{\rho^2} \right)^{1 - \epsilon} \right).
    \end{equation}

    Suppose instead $\phi > 0$. Then, for every $\epsilon > 0$, there exist $\alpha$, $c$ and $\tilde \eta(0)$ such that, with overwhelming probability,
    \begin{equation}\label{eq:scalingbestrates2}
        \mathcal R(\theta^p_\alpha) = O_\gamma \left( \gamma^{1-\epsilon} + \left( \frac{\gamma^2}{\rho^2} \right)^{\frac{2 - \psi - \phi}{2 - \psi}} \right).
    \end{equation}
\end{corollary}


We recall that the setting $\phi < 0$ is closer to the well-conditioned case discussed in Section \ref{sec:optimalrates}, since the covariance approaches the identity as $\phi\to-\infty$. 
In fact, in this setting, the risk approaches the same rate $\Theta_\gamma(\gamma + \gamma^2 / \rho^2)$ given by Theorems \ref{thm:harmonicbody}-\ref{thm:lowerbound}. This implies that the privacy requirement has a performance cost when $b \geq 1/2$, corresponding to $\rho =O( \sqrt{\gamma})$.

In contrast, $\phi > 0$ is closer to an ill-conditioned setting, since more eigenvalues concentrate towards 0 as $\phi \to 1$. In this setting, the minimax optimal rate for well-conditioned covariance is not 
approached. In fact, if
\begin{equation}\label{eq:criticalbillconditioned}
    b > \frac{1}{2} - \frac{\phi}{2 (2 - \phi - \psi)},
\end{equation}
the second term in the RHS of \eqref{eq:scalingbestrates2} dominates the first one (when $\epsilon$ is sufficiently small). We note that the RHS of \eqref{eq:criticalbillconditioned} is smaller than $1 / 2$. This implies that, when $\phi > 0$, the privacy requirement has a performance cost for a larger privacy parameter $\rho$ than when $\phi < 0$. In fact, as $\phi$ approaches 1, even a mild privacy requirement (corresponding to $b$ close to $0$) yields a decrease in performance.

We now provide an interpretation of the higher cost of privacy incurred when $\phi>0$. If $\phi > 0$, then $\Sigma$ has many small eigenvalues, which corresponds to having several directions with low signal-to-noise ratio. In this subspace, estimating $\theta^*$ relies on rare samples with unusually large components, making the learning algorithm more dependent on atypical observations. The effect is further enhanced as more weight of $\theta^*$ is on the directions where the eigenvalues are small, which corresponds to having a larger value of $\psi$, as discussed after Assumption \ref{ass:power_law_1}.
This is reminiscent of a setting such that ``learning requires memorization\footnote{Memorization captures the influence an individual sample has on the final model.}'' \citep{feldman2020does}, and in this scenario enforcing differential privacy has been shown to deteriorate performance especially on atypical and under-represented groups \citep{cummings2019compatibility, fioretto2022differential}.

\begin{figure}
    \centering
    \includegraphics[width=\linewidth]{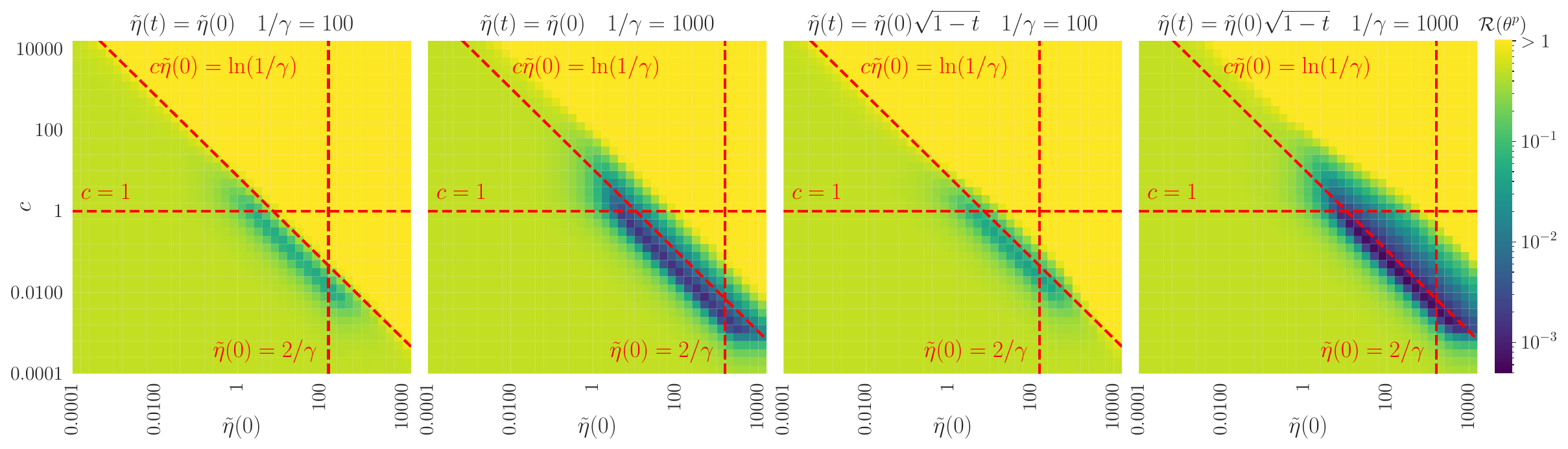}
    \caption{Numerical simulations for $\mathcal R(\theta^p)$ obtained via Algorithm \ref{alg:dp-sgd} for $d = 100$ and $\Sigma = I$, as a function of $c$ and $\tilde \eta(0)$. We consider the schedules corresponding to output perturbation ($\tilde \eta(t) = \tilde \eta(0)$, first and second panel) and constant private noise ($\tilde \eta(t) = \tilde \eta(0) \sqrt{1 - t}$, third and fourth panel),
    with fixed $\rho = 0.1$ and $\zeta = 0.3$. We set $\gamma = 0.01$ in the first and third panel, and $\gamma = 0.001$ in the second and fourth panel. The values of $\mathcal R(\theta^p)$ are capped at $1$, and $\theta^*$ is chosen such that $\mathcal R(\theta_0) = 0.5$. We indicate with red dashed lines the curves $c = 1$, $\tilde \eta(0) = 2 / \gamma$, and $c \tilde \eta(0) = \ln(1 / \gamma)$, and we display the average over 10 independent trials.}
    \label{fig:heatmaps}
\end{figure}

\begin{figure}
    \centering
    \includegraphics[width=\linewidth]{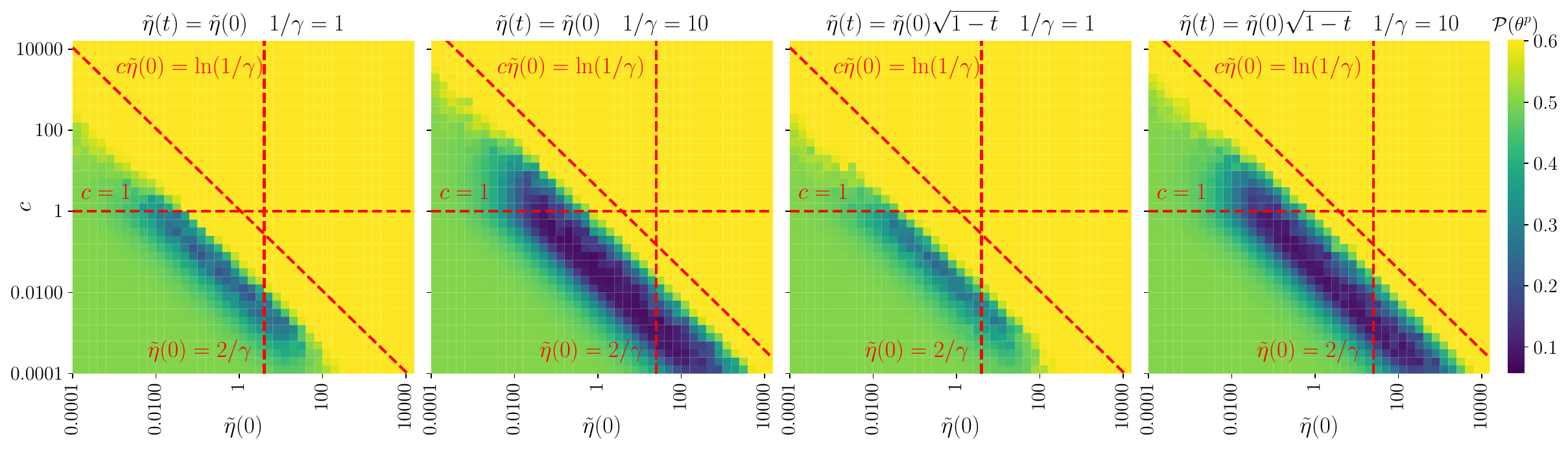}
    \caption{Numerical simulations for $\mathcal P(\theta^p)$ obtained via Algorithm \ref{alg:dp-sgd} on the MNIST dataset, 
    as a function of $c$ and $\tilde \eta(0)$. We consider the schedules corresponding to output perturbation ($\tilde \eta(t) = \tilde \eta(0)$, first and second panel) and constant private noise ($\tilde \eta(t) = \tilde \eta(0) \sqrt{1 - t}$, third and fourth panel), with fixed $\rho = 0.1$. We set $\gamma = 1$ in the first and third panel, and $\gamma = 0.1$ in the second and fourth panel. The values of $\mathcal P(\theta^p)$ are capped at $0.6$. We indicate with red dashed lines the curves $c = 1$, $\tilde \eta(0) = 2 / \gamma$, and $c \tilde \eta(0) = \ln(1 / \gamma)$, and we display the average over 10 independent trials.}
    \label{fig:heatmap_mnist}
\end{figure}

\begin{figure}
    \centering
    \includegraphics[width=\linewidth]{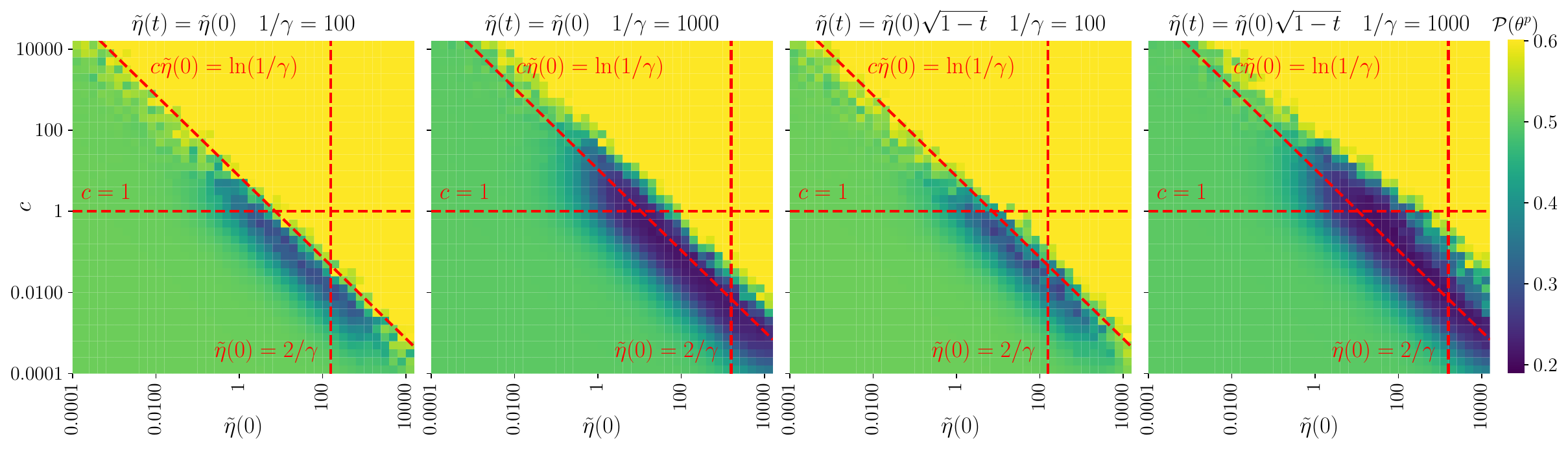}
        \caption{Numerical simulations for $\mathcal P(\theta^p)$ obtained via Algorithm \ref{alg:dp-sgd} on the California housing dataset, 
        as a function of $c$ and $\tilde \eta(0)$. We consider the schedules corresponding to output perturbation ($\tilde \eta(t) = \tilde \eta(0)$, first and second panel) and constant private noise ($\tilde \eta(t) = \tilde \eta(0) \sqrt{1 - t}$, third and fourth panel), with fixed $\rho = 0.1$. We set $\gamma = 0.01$ in the first and third panel, and $\gamma = 0.001$ in the second and fourth panel. The values of $\mathcal P(\theta^p)$ are capped at $0.6$. We indicate with red dashed lines the curves $c = 1$, $\tilde \eta(0) = 2 / \gamma$, and $c \tilde \eta(0) = \ln(1 / \gamma)$, and we display the average over 5 independent trials.}
    \label{fig:heatmap_housing}
\end{figure}

\section{Numerical results}\label{sec:numerical}

\paragraph{Experimental details.} 
We consider synthetic data (Figures \ref{fig:heatmaps} and \ref{fig:schedules}), the MNIST dataset (Figures \ref{fig:heatmap_mnist} and \ref{fig:schedules_real}), and the California housing dataset (Figures \ref{fig:heatmap_housing} and \ref{fig:schedules_real}). 
For experiments on synthetic data, we sample Gaussian data following Assumptions \ref{ass:data} and \ref{ass:power_law_1}. For experiments on MNIST, we cast the problem as a binary task by selecting two classes (digits ``1'' and ``7'') and labeling them with $y = \pm 1$. Each image is flattened into a vector in $\R^d$ with $d = 784$. For experiments on the California housing dataset, we download its scikit-learn version and consider as input covariates its $d = 8$ numerical features. We 
split datasets into a training subset, a disjoint normalization subset, and a held-out validation/test subset. Using the normalization subset, we compute (non-privately) the empirical mean and standard deviation of each input coordinate (and of the labels) and use them to normalize both the training and validation covariates (labels), so that each coordinate has zero mean and unit variance. The plotted quantity $\mathcal P$ in Figures \ref{fig:heatmap_mnist}, \ref{fig:heatmap_housing} and \ref{fig:schedules_real} corresponds to the average value of the final validation square loss. For the housing dataset, we notice that our algorithm becomes numerically unstable for larger values of $n$, with the value of the final loss being disproportionally large for a few values of the random seeds. We manually exclude such seeds from the plots. 

\paragraph{Hyper-parameter space.}

In Figures \ref{fig:heatmaps}, \ref{fig:heatmap_mnist} and \ref{fig:heatmap_housing}, we run Algorithm \ref{alg:dp-sgd} with a fixed privacy budget $\rho = 0.1$ on synthetic, MNIST and California housing data. We focus on the schedules $\tilde \eta(t) = \tilde \eta(0)$ and $\tilde \eta(t) = \tilde \eta(0) \sqrt{1 - t}$, for varying values of the renormalized clipping constant $c$ and learning rate $\tilde \eta(0)$. We report on the heatmaps the resulting values of the final loss as a function of $\tilde \eta(0)$ (on the $x$ axis) and $c$ (on the $y$ axis).
All the results are in agreement with our discussion on the hyper-parameters following Proposition \ref{prop:alpha012}: performance deteriorates if either $c$ exceeds $1$ (upper part of the heatmaps) or $\tilde \eta(0)$ exceeds $2 / \gamma$ (right part of the heatmaps), and the lowest values of the risk are roughly parallel to the line $c \tilde \eta(0) = \ln(1 / \gamma)$.
We remark that these scalings agree with the empirical practice in deep learning of using a sufficiently small clipping constant, with a learning rate renormalized by its value \citep{de2022, mckenna2025}.

\begin{figure}
  \begin{center}
    \includegraphics[width=\textwidth]{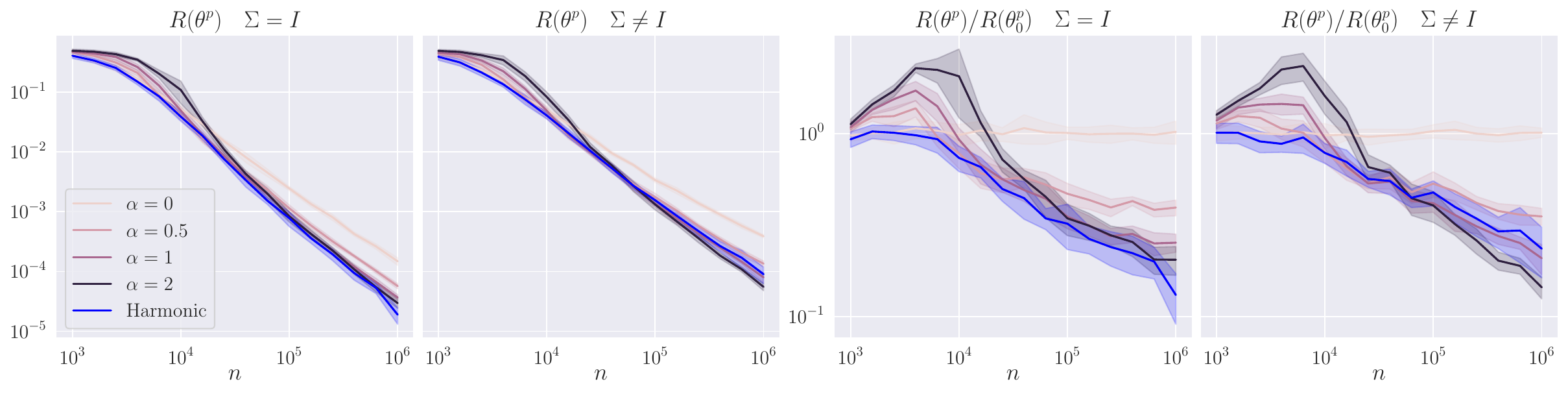}
  \end{center}
  \caption{Numerical simulations for $\mathcal R(\theta^p)$ obtained via Algorithm \ref{alg:dp-sgd} for $d = 100$, $\zeta = 0.3$, $(\theta^*_i)^2 = 1/d$, and $\rho = 0.1$, as a function of $n$. We fix $c = 0.1$ and consider the polynomial schedules in \eqref{eq:polynomialsched} for $\alpha \in \{0, 0.5, 1, 2\}$ (optimizing w.r.t.\ $\tilde \eta(0)$) and the harmonic schedule in \eqref{eq:harmonicschedule} (optimizing w.r.t.\ $\beta$ and $\tau$). We report the average over 10 independent trials, as well as the confidence interval corresponding to 1 standard deviation. We consider both isotropic data ($\Sigma = I$, first and third panel) and data with diagonal covariance whose eigenvalues are uniformly distributed in the interval $[0, 2]$ (second and fourth panel).}
  \label{fig:schedules}
\end{figure}

\paragraph{Effect of the learning rate schedule.}

In Figures \ref{fig:schedules} and \ref{fig:schedules_real}, we compare different schedules for the learning rate $\tilde \eta(t)$, after setting $c = 0.1$ and optimizing $\tilde \eta(0)$ numerically. We consider the polynomial schedules in \eqref{eq:polynomialsched} for $\alpha = \{0, 0.5, 1, 2\}$ and the harmonic schedule in \eqref{eq:harmonicschedule}.
The first and second panel in Figure \ref{fig:schedules} (which display the results for synthetic data) show that all schedules rapidly give better results as $n$ increases, but larger values of $\alpha$ are optimal only for large enough $n$. This effect is clearly shown in the third and fourth panel, which reports the same results normalized by the loss of output perturbation ($\alpha = 0$). \\
\begin{wrapfigure}{l}{0.55\textwidth}
  \centering
  \includegraphics[width=0.53\textwidth]{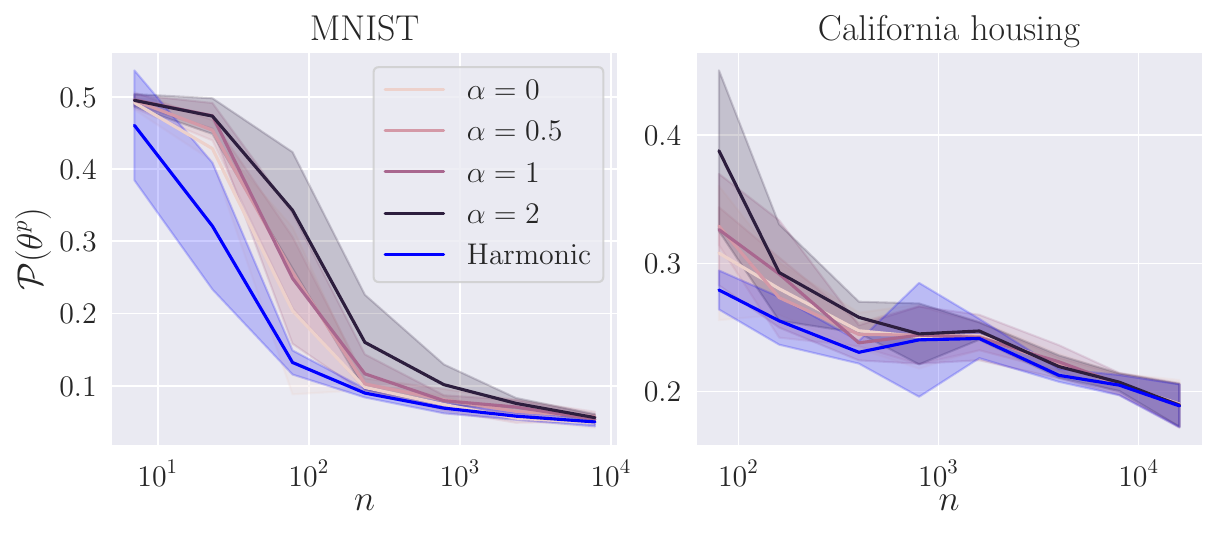}
  \caption{Numerical simulations for $\mathcal P(\theta^p)$ obtained via Algorithm \ref{alg:dp-sgd} with $\rho = 1$ for the MNIST and the California housing dataset as a function of $n$. The method is in common with the one in Figure \ref{fig:schedules}.}
  \label{fig:schedules_real}
\end{wrapfigure}
Hence, if $n$ is sufficiently small compared to $d$, output perturbation can outperform other schedules, in agreement with an observation made in \cite{dwork2025} regarding the comparison between output and objective perturbation. Furthermore, the benefits of using larger values of $\alpha$ for sufficiently large values of $n$ seem more apparent in the case of non-isotropic covariance, displayed in the second and fourth panel.
The results for the harmonic schedule (reported in blue) are obtained after optimizing numerically $\beta$ and $\tau$ as in \eqref{eq:harmonicschedule}. The third panel of Figure \ref{fig:schedules} clearly displays the benefits of this choice for both small and large values of $n$. In the non-isotropic case (fourth panel), this schedule performs optimally for sufficiently small values of $n$, and it is on par with other polynomial schedules when $n = 10^4 d$ (right part of the plot).
In Figure \ref{fig:schedules_real}, we run the same experiment for the MNIST and California housing datasets: we note that the harmonic schedule remains optimal, as suggested by our theory. 



\section{Discussion and future directions}\label{sec:disc}

We show optimal rates for DP linear regression in the proportional regime where the number of samples grows linearly with the input dimension.
To achieve this result, a crucial technical innovation is to establish a deterministic equivalent of the test risk of DP-GD via a family of deterministic ODEs, in the challenging setting where the clipping constant is of the order of the per-sample gradient. Analyzing the family of ODEs then gives bounds that are sharp enough to \emph{(i)} demonstrate the benefit of 
small clipping constants, and to \emph{(ii)} identify the optimal learning rate schedule. 


The optimality of the rate $\gamma + \gamma^2 / \rho^2$ holds for a well-conditioned covariance, as shown in the minimax lower bound of Theorem \ref{thm:lowerbound}. In the ill-conditioned setting, we consider a covariance spectrum distributed according to a power law (Assumption \ref{ass:power_law_1}), and we provide a  tight analysis of the DP-GD risk for polynomially-decaying learning rate schedules. As in the well-conditioned case, our results highlight the benefits of small clipping constants and decaying schedules. However, in contrast to the well-conditioned case, establishing minimax rates remains an interesting direction for future work. We also note that the optimality of our rates holds when the variance of the label noise $\zeta^2$ is a strictly positive constant independent of $\gamma$. We cautiously suspect that, for $\zeta = o_\gamma(1)$, the optimal rate is of order $\zeta^2 (\gamma + \gamma^2 / \rho^2)$ (see e.g. Theorem 4.1 in \cite{cai21}, where this dependence is tracked). In this regime, our bounds on $\mu_c(R)$ and $\nu_c(R)$ in Lemma \ref{lemma:munu} suggest that a sufficiently small value of the clipping constant is $c = O_\gamma(\sqrt{2R + \zeta^2})$ rather than $c = O_\gamma(1)$. If $\zeta = o_\gamma(\gamma)$, setting a fixed $c = O_\gamma(\zeta)$ might not allow to achieve the optimal rates (as in the right side of the panels in Figure \ref{fig:heatmaps}), suggesting the need for a time-dependent scheduling of $c$.
Interesting future directions also involve \emph{(i)} proving that the optimal non-asymptotic rates read $d / n + d^2 / (n \rho)^2$ beyond the proportional regime, and \emph{(ii)} better understanding the precise proportional asymptotics beyond the limit $\gamma \to 0$, for example showing the benefits of output perturbation w.r.t.\ schedules with decaying learning rate when $\gamma$ is sufficiently large, as displayed in Figure \ref{fig:schedules}. 



Finally, we remark that the technical approach of analyzing DP-GD through a deterministic equivalent can be used to study differentially private optimization beyond linear regression. A concrete setting for future work is provided e.g.\ by logistic regression, where the GD dynamics (in the absence of clipping and private noise) has been considered in \cite{collins2025exact}.

\section*{Acknowledgements}

The work of Jialei Luo was done while she was an ISTern at the Institute of Science and Technology Austria.
This research was funded in whole or in part by the Austrian Science Fund (FWF) 10.55776/COE12. For the purpose of open access, the authors have applied a CC BY public copyright license to any Author Accepted Manuscript version arising from this submission. 
Simone Bombari was supported by a Google PhD fellowship. 
Inbar Seroussi was partially supported by the Israel Science Foundation grant no. 777/25, the NSF-BSF grant no. 0603624011, Alon fellowship and the Council for Higher Education in Israel under the Moonshot Project.
The authors would like to thank (in alphabetical order) Edwige Cyffers, Mahdi Haghifam, Elliot Paquette, Thomas Steinke, Kabir Aladin Verchand and Yichen Wang for helpful discussions.

{
\small

\bibliographystyle{plainnat}
\bibliography{bibliography.bib}

}

\newpage

\appendix

\section{Proofs on differential privacy}\label{app:DP}

The notion of zCDP can be converted to $(\varepsilon, \delta)$-DP, which in turn is defined below.
\begin{definition}[$(\varepsilon, \delta)$-DP \citep{dwork2006}]\label{def:dp}
A randomized algorithm $\mathcal A$ 
satisfies $(\varepsilon, \delta)$-differential privacy if for any pair of adjacent datasets $D, D'$, 
and for any subset of the parameters space $S \subseteq \R^d$, we have
\begin{equation}
    \P \left( \mathcal A (D) \in S \right) \leq e^{\varepsilon} \P \left( \mathcal A (D') \in S \right) + \delta.
\end{equation}
\end{definition}
For completeness, we also provide the following definition.
\begin{definition}[Rényi DP \citep{Mironov2017}]\label{def:renyi}
Given $\alpha \in (1, +\infty)$ and $\varepsilon \geq 0$, an algorithm $\mathcal A$ satisfies $(\alpha, \varepsilon)$ Rényi DP if for any pair of adjacent datasets $D, D'$ we have $D_{\alpha} \left( \mathcal{A}(D) \,\|\, \mathcal{A}(D') \right) \leq \varepsilon$, where $D_{\alpha} \left( \mathcal{A}(D) \,\|\, \mathcal{A}(D') \right)$ is the Rényi Divergence \citep{Renyi61} between the probability distributions induced by the randomness of $\mathcal A$, i.e.,
\begin{equation}
    D_{\alpha} \left( \mathcal{A}(D) \,\|\, \mathcal{A}(D') \right) = \frac{1}{\alpha - 1} \ln \int \left( \frac{p_{\mathcal{A}(D)} (\theta)}{p_{\mathcal{A}(D')} (\theta)} \right)^\alpha p_{\mathcal{A}(D')} (\theta)  \diff \theta.
\end{equation}
\end{definition}
Note that Definitions \ref{def:zcdp} and \ref{def:renyi} imply that an algorithm is $\rho$-zCDP if, for any $\alpha > 1$, it is also $(\alpha, \rho \alpha)$ Rényi DP. The following proposition allows to translate Rényi DP and zCDP to $(\varepsilon, \delta)$-DP.
\begin{proposition}[Proposition 1.3 in \cite{bun2016concentrated}]\label{prop:zCDP}
    If $\mathcal{A}$ satisfies $\rho^2 / 2$-zCDP, it also satisfies $\left( \rho^2 / 2 + \rho \sqrt{ 2 \ln(1/\delta) }, \delta \right)$-DP, for any $\delta \in (0 , 1)$.
\end{proposition}

Then, if we consider $\delta$ such that $\rho \leq \sqrt{\ln(1 / \delta)}$, we achieve $(\varepsilon, \delta)$-DP if we have
\begin{equation}
    \rho^2 / 2 + \rho \sqrt{ 2 \ln(1/\delta) } \leq 2 \rho \sqrt{\ln(1 / \delta)} \leq \varepsilon ,
\end{equation}
which means that for algorithms respecting $\rho^2 / 2$-zCDP, we can replace $2 \rho$ by $\varepsilon / \sqrt{\ln(1 / \delta)}$ in the error bounds to evaluate the cost of privacy in terms of $(\varepsilon , \delta)$-DP.

\subsection{Proof of Proposition \ref{prop:DPguarantees}}
    Let us define the family of functions $\ell_{k, \clip}(\cdot): \R \to \R$, for all $k \in [n]$, such that $\ell_{k, \clip}(0) = 0$, and
    \begin{equation}\label{eq:lclip}
        \ell_{k, \clip}'(z) = z \min \left( 1, \frac{\clip}{\left| z \right| {\norm{x_k}_2}} \right).
    \end{equation}
    In words, $\ell_{k, \clip}(z)$ is a quadratic function, but linearized for sufficiently large values of $|z|$, such that it is $\clip / \norm{x_k}_2$-Lipschitz. Then, 
    we have
    \begin{equation}\label{eq:equivclippedloss}
    \begin{aligned}
        \bar g_{k} &= g_{k} \min \left(1, \frac{\clip}{\norm{g_k}_2} \right)\\
        &= x_k \left( x_{k}^\top \theta_{k-1} - y_k \right) \min \left(1, \frac{\clip}{\norm{x_k}_2 \left|  x_{k}^\top \theta_{k-1} - y_k  \right| } \right) \\
        &= x_k \ell'_{k, \clip} \left( x_k^\top \theta_{k-1} - y_k\right) \\
        &= \nabla_{\theta} \ell_{k, \clip} \left( x_k^\top \theta_{k-1} - y_k\right),
    \end{aligned}
    \end{equation}
    where the first step follows from the definition of $\bar g_{k}$ in Algorithm \ref{alg:dp-sgd}. Furthermore, for any $\theta, \theta'\in \R^d$, we have
    \begin{equation}\label{eq:smoothclippedloss}
    \begin{aligned}
        & \norm{\nabla_\theta \ell_{k, \clip}(x_k^\top \theta - y_k) - \nabla_\theta \ell_{k, \clip}(x_k^\top \theta' - y_k)}_2 \\
        & \qquad = \norm{x_k}_2 \left| \ell'_{k, \clip}(x_k^\top \theta - y_k) - \ell'_{k, \clip}(x_k^\top \theta' - y_k) \right| \\
        & \qquad \leq \norm{x_k}_2 \left| x_k^\top \left( \theta - \theta' \right) \right| \\
        & \qquad \leq \norm{x_k}_2^2 \norm{\theta - \theta'}_2,
    \end{aligned}
    \end{equation}
    where the second step follows from the fact that $\ell_{k, \clip}'(z)$ is a 1-Lipschitz function. Let us now define
    \begin{equation}\label{eq:barredloss}
        \bar \ell_{k, \clip}(z) = \min \left( 1 , \frac{2}{\norm{x_k}_2^2 \eta_k} \right) \ell_{k, \clip}(z).
    \end{equation}
    Then, we have that every iteration of Algorithm \ref{alg:dp-sgd} takes the form
    \begin{equation}
        \theta_k = \theta_{k-1} - \eta_k \nabla_\theta \bar \ell_{k, \clip}(x_k^\top \theta_{k -1} - y_k) + 2 \clip \sigma_k b_k,
    \end{equation}
    where $\bar \ell_{k, \clip}(x_k^\top \theta - y_k)$ is $\clip$-Lipschitz 
    with respect to $\theta$, and it is $2 /\eta_k$-smooth, i.e.,
    \begin{equation}
        \norm{\nabla_\theta \bar \ell_{k, \clip}(x_k^\top \theta - y_k) - \nabla_\theta \bar \ell_{k, \clip}(x_k^\top \theta' - y_k)}_2 \leq \frac{2}{\eta_k} \norm{\theta - \theta'}_2,
    \end{equation}
    due to \eqref{eq:smoothclippedloss} and \eqref{eq:barredloss}.
    Thus, the desired result follows from Theorem 3.1 in \cite{feldman2020linear}, after setting their batch sizes $\{ B_k \}$ identically equal to 1, and their projection set $\mathcal K$ equal to all $\R^d$. \qed

\section[Auxiliary lemmas]{Auxiliary lemmas on $\mu_c(\theta)$ and $\nu_c(\theta)$}\label{app:munu}

\begin{lemma}\label{lemma:munu}
Let Assumption \ref{ass:data} hold, and let $\mu_c(\theta)$ and $\nu_c(\theta)$ be defined according to \eqref{eq:munu}. Then, we have
\begin{align}
    \mu_c(\theta) &= \erf \left(\frac{c}{2 \sqrt{\mathcal P(\theta)}} \right),\label{eq:mucdef} \\
    \nu_c(\theta) &= \frac{c^2}{2 \mathcal P(\theta)} \left( 1 - \erf \left(\frac{c}{2 \sqrt{\mathcal P(\theta)}} \right) \right) + F \left( \frac{c}{\sqrt{2 \mathcal P(\theta)}} \right),
\end{align}
where
\begin{equation}
    \erf(z) = \frac{2}{\sqrt{\pi}} \int_{0}^{z} e^{-t^2} \diff t, \qquad  F(z) = \frac{1}{\sqrt{2 \pi}} \int_{-z}^{z} t^2 e^{-t^2 / 2} \diff t = \erf\left( \frac{z}{\sqrt 2} \right) - \sqrt{\frac{2}{\pi}} z e^{-z^2 / 2}.
\end{equation}
In particular, 
$\mu_c(\theta)$ and $\nu_c(\theta)$ depend only on $c$ and the test risk $\mathcal P(\theta)$ via the ratio $c / \sqrt{2 \mathcal P(\theta)}$.
\end{lemma}
\begin{proof}
Recall that 
\begin{equation}
    r(\theta, x, y) = x^\top \theta - y, \qquad r_c(\theta, x, y) = r(\theta, x, y) \min \left(1 , \frac{c}{\left| r(\theta, x, y) \right|} \right),
\end{equation}
\begin{equation}
    \mu_c(\theta) = \frac{\norm{\E_{x,y} \left[ r_c(\theta, x, y) \, x \right]}_2}{\norm{\E_{x,y} \left[ r(\theta, x, y) \, x \right]}_2}, \qquad \nu_c(\theta) = \frac{\E_{x,y} \left[ r_c(\theta, x, y)^2 \right]}{\E_{x,y} \left[r(\theta, x, y)^2 \right]}.
\end{equation}
Until the end of the proof, we will use the notation $\clipc(\cdot) : \R \to \R$ to denote the function such that
\begin{equation}
    \clipc(a) = a \min \left(1, \frac{c}{\left| a \right|} \right).
\end{equation}
In particular, 
$r_c(\theta, x, y) = \clipc \left( r(\theta, x, y) \right)$.

Let us look at the first entry of the vector $\E_{x,y} \left[ r_c(\theta, x, y) \, x \right]$,
\begin{equation}\label{eq:firstcomponent}
    \E_{x,y} \left[ r_c(\theta, x, y) x^\top e_1 \right] = \E_{\rho_1, \rho_2} \left[ \clipc(\rho_1) \rho_2 \right],
\end{equation}
where the second step introduced $\rho_1$ and $\rho_2$, defined as two mean-0 Gaussian random variables, such that
\begin{equation}
    \Var(\rho_1) = \norm{\Sigma^{1/2}(\theta - \theta^*)}_2^2 + \zeta^2, \qquad \Var(\rho_2) = \Sigma_{11}, \qquad \Cov(\rho_1, \rho_2) = e_1^\top \Sigma (\theta - \theta^*).
\end{equation}

Then, we have
\begin{equation}
\begin{aligned}
    \E_{\rho_1, \rho_2} \left[ \clipc(\rho_1) \rho_2 \right] &= \frac{\Cov(\rho_1, \rho_2)}{\Var(\rho_1)} \E_{\rho_1} \left[ \clipc(\rho_1) \rho_1 \right] \\
    &= \frac{\Cov(\rho_1, \rho_2)}{\sqrt{\Var(\rho_1)}} \E_{\hat \rho} \left[ \clipc( \sqrt{\Var(\rho_1)} \hat \rho) \hat \rho \right] \\
    &= \frac{e_1^\top \Sigma (\theta - \theta^*)}{\sqrt{2 \mathcal P(\theta)}}  \E_{\hat \rho} \left[ \clipc( \sqrt{2 \mathcal P(\theta)} \hat \rho) \hat \rho \right],
\end{aligned}
\end{equation}
where we used $\mathcal P(\theta) = \mathcal R(\theta) + \zeta^2 / 2 = \Var(\rho_1) / 2$ and we introduced the standard Gaussian random variable $\hat \rho$. As this argument holds for any component of the vector $\E_{x,y} \left[ r_c(\theta, x, y) \, x \right]$, plugging the equation above in \eqref{eq:firstcomponent} gives
\begin{equation}
\begin{aligned}
    \norm{\E_{x,y} \left[ r_c(\theta, x, y) \, x \right]}_2 &= \frac{\norm{\Sigma (\theta - \theta^*)}_2}{\sqrt{2 \mathcal P(\theta)}}  \E_{\hat \rho} \left[ \clipc( \sqrt{2 \mathcal P(\theta)} \hat \rho) \hat \rho \right] \\
    &= \frac{ \norm{\E_{x,y} \left[ r(\theta, x, y) \, x \right]}_2}{\sqrt{2 \mathcal P(\theta)}} \E_{\hat \rho} \left[ \clipc( \sqrt{2 \mathcal P(\theta)} \hat \rho) \hat \rho \right],
\end{aligned}
\end{equation}
where in the second step we used that $\E_{x,y} \left[ r(\theta, x, y) \, x \right] = \Sigma (\theta - \theta^*)$. Then, we also have
\begin{equation}\label{eq:muP}
    \mu_c(\theta) = \frac{\E_{\hat \rho} \left[ \clipc( \sqrt{2 \mathcal P(\theta)} \hat \rho) \hat \rho \right]}{\sqrt{2 \mathcal P(\theta)}}.
\end{equation}

Defining the shorthand $c'(\theta) = c / \sqrt{2 \mathcal P(\theta)}$, the numerator of the expression above yields
\begin{equation}
\begin{aligned}
\E_{\hat \rho} \left[ \clipc( \sqrt{2 \mathcal P(\theta)} \hat \rho) \hat \rho \right] &= \frac{\sqrt{2 \mathcal P(\theta)}}{\sqrt{2 \pi}} \int_{-c'(\theta)}^{c'(\theta)} {\hat \rho}^2 e^{-{\hat \rho}^2 / 2} \diff \hat \rho + \frac{2 c}{\sqrt{2 \pi}} \int_{c'(\theta)}^{+\infty} \hat \rho e^{-{\hat \rho}^2 / 2} \diff \hat \rho \\
& = \frac{\sqrt{2 \mathcal P(\theta)}}{\sqrt{2 \pi}} \left( \left. - \hat \rho e^{-{\hat \rho}^2 / 2} \right|_{- c'(\theta)}^{c'(\theta)} + \int_{-c'(\theta)}^{c'(\theta)} e^{-{\hat \rho}^2 / 2} \diff \hat \rho \right)
- \frac{2 c}{\sqrt{2 \pi}} \left. e^{-{\hat \rho}^2 / 2} \right|_{ c'(\theta)}^{+\infty} \\
& = \frac{\sqrt{2 \mathcal P(\theta)}}{\sqrt{2 \pi}} \left(- 2 c'(\theta) e^{-{c'(\theta)}^2 / 2} + \int_{-c'(\theta)}^{c'(\theta)} e^{-{\hat \rho}^2 / 2} \diff \hat \rho  \right) + \frac{2 c}{\sqrt{2 \pi}}  e^{-{c'(\theta)}^2 / 2} \\
& = \frac{\sqrt{\mathcal P(\theta)}}{\sqrt{\pi}} \int_{-c'(\theta)}^{c'(\theta)} e^{-{\hat \rho}^2 / 2} \diff \hat \rho \\
& = \frac{2 \sqrt 2 \sqrt{\mathcal P(\theta)}}{\sqrt{\pi}} \int_{0}^{c'(\theta) / \sqrt 2} e^{-{\hat \rho}^2} \diff \hat \rho \\
&= \sqrt{2 \mathcal P(\theta)} \, \erf \left( \frac{c}{\sqrt{4 \mathcal P(\theta)}} \right),
\end{aligned}
\end{equation}
which, plugged in \eqref{eq:muP}, gives the first part of the thesis.

For the second part of the thesis, following the same argument we used to write \eqref{eq:muP}, we have
\begin{equation}\label{eq:nuP}
    \nu_c(\theta) = \frac{\E_{\hat \rho} \left[ \clipc(\sqrt{2 \mathcal P(\theta)} \hat \rho)^2 \right]}{2 \mathcal P(\theta)},
\end{equation}
where, as before, $\hat \rho$ denotes a standard Gaussian random variable. Then, we have
\begin{equation}
\begin{aligned}
    \nu_c(\theta) & = \frac{1}{\sqrt{2 \pi}} \int_{-c'(\theta)}^{c'(\theta)} {\hat \rho}^2 e^{-{\hat \rho}^2 / 2} \diff \hat \rho + \frac{2 c'(\theta)^2}{\sqrt{2 \pi}} \int_{c'(\theta)}^{+\infty} e^{-{\hat \rho}^2 / 2} \diff \hat \rho \\
    & = \frac{1}{\sqrt{2 \pi}} \int_{-c'(\theta)}^{c'(\theta)} {\hat \rho}^2 e^{-{\hat \rho}^2 / 2} \diff \hat \rho + c'(\theta)^2 \left( 1 - \frac{2}{\sqrt{2 \pi}} \int_{0}^{c'(\theta)} e^{-{\hat \rho}^2 / 2} \diff \hat \rho \right) \\
    & = \frac{c^2}{2 \mathcal P(\theta)} \left( 1 - \erf \left( \frac{c}{2 \sqrt{\mathcal P(\theta)}} \right) \right) + F(c'(\theta)),
\end{aligned}
\end{equation}
which concludes the proof. 
\end{proof}

\begin{lemma}\label{lemma:munubounds}
Let Assumption \ref{ass:data} hold, and let $\mu_c(R)$ and $\nu_c(R)$ be defined according to \eqref{eq:munu}, where $R$ denotes a generic value of the test risk, since $\mu_c(\theta)$ and $\nu_c(\theta)$ depend only on $c$ and the test risk $\mathcal P(\theta)$ due to Lemma \ref{lemma:munu}. Then, for any $c > 0$, we have that
\begin{equation}\label{eq:bdcnu}
    \underline c_\mu (c, \zeta) < \frac{\mu_c(R) \sqrt{2R + \zeta^2}}{c} < \sqrt{\frac{2}{\pi}}, \qquad  \underline c_\nu (c, \zeta) < \frac{\nu_c(R) (2R + \zeta^2)}{c^2} < 1,
\end{equation}
where $\underline c_\mu (c, \zeta)$ and $\underline c_\nu (c, \zeta)$ denote two positive constants which depend on the values of $c$ and $\zeta$ and are monotonously decreasing in $c$.

We also have that, as $c / \zeta \to 0$,
\begin{equation}
    \left| \frac{\mu_c(R) \sqrt{2R + \zeta^2}}{c} - \sqrt{\frac{2}{\pi}} \right| = o_{c/\zeta} \left( 1 \right), \qquad \left| \frac{\nu_c(R) \left( 2R + \zeta^2 \right)}{c^2} - 1 \right| = o_{c/\zeta} \left(1\right).
\end{equation}
Furthermore, we have 
\begin{equation}\label{eq:bdsnu}
\begin{aligned}
    &\frac{\nu_c(R) (2R + \zeta^2)}{c^2} > \frac{1}{2}, \qquad \text{if } \frac{c}{\sqrt{2R + \zeta^2}} \leq 1, \\
    &\nu_c(R) > \frac{1}{2}, \qquad\qquad\qquad\,\,\,\,\, \text{if } \frac{c}{\sqrt{2R + \zeta^2}} > 1.
\end{aligned}    
\end{equation}
\end{lemma}

\begin{figure}
    \centering
    \includegraphics[width=\linewidth]{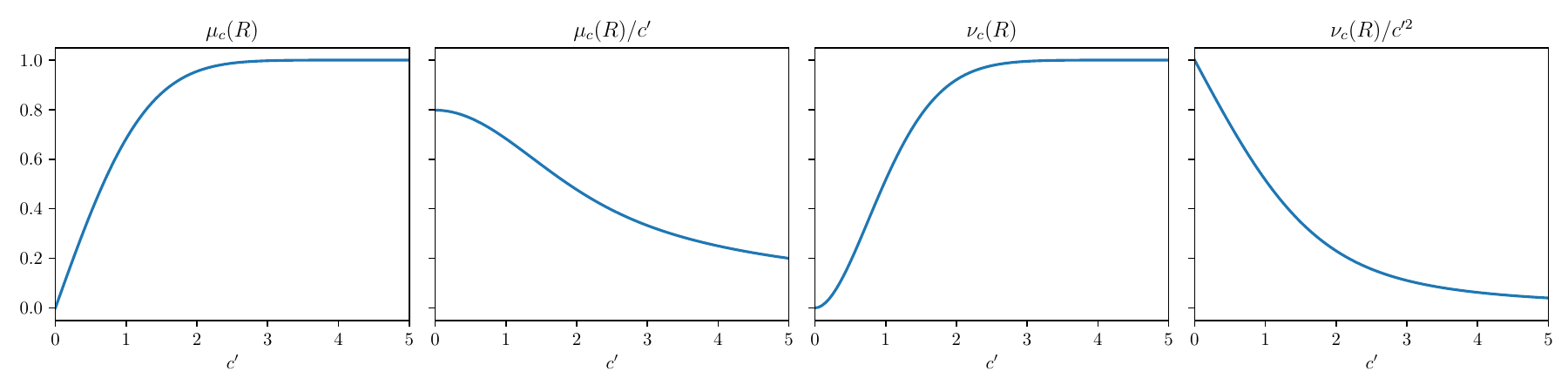} 
    \caption{The functions $\mu_c(R)$, $\mu_c(R) / c'$, $\nu_c(R)$, $\nu_c(R) / c'$, plotted as a function of $c' = c / \sqrt{2R + \zeta^2}$.}
    \label{fig:munu}
\end{figure}

\begin{proof}
    Note that, introducing the notation
    \begin{equation}\label{eq:c'}
        c' = \frac{c}{\sqrt{2R + \zeta^2}} \leq \frac{c}{\zeta},
    \end{equation}
    we have that
    \begin{equation}
        \mu_c(R) = \erf \left( c' / \sqrt 2 \right) < 1,
    \end{equation}
    and
    \begin{equation}
        \nu_c(R) = (c')^2 \left( 1 - \erf \left( c' / \sqrt 2 \right) \right) + F(c') < 1,
    \end{equation}
    where the last inequalities can be verified directly via the definitions in \eqref{eq:munu}.
    Furthermore, we have that both $\mu_c(R)$ and $\nu_c(R)$ are increasing functions of $c'$, equal to $0$ when $c' = 0$. This follows from the definition for $\mu_c(R)$, and can be promptly verified for $\nu_c(R)$ via derivation.

    We further have that, for $c' > 0$,
    \begin{equation}
        \frac{\mu_c(R) \sqrt{2R + \zeta^2}}{c} = \frac{1}{c'} \erf \left( c' / \sqrt 2 \right) < \sqrt{\frac{2}{\pi}},
    \end{equation}
    and
    \begin{equation}
        \frac{\nu_c(R) \left(2R + \zeta^2\right)}{c^2} = \left( 1 - \erf \left( c' / \sqrt 2 \right) \right) + \frac{F(c')}{(c')^2} < 1,
    \end{equation}
    where both the LHSs are decreasing functions of $c'$, going to $0$ for $c' \to + \infty$ (this can be seen via explicit derivation with respect to $c'$, and via the identity $0 \geq \int_0^z - 2t^2 e^{-t^2} \diff t = z e^{-z^2} - \int_0^z e^{-t^2} \diff t$), and where the last inequalities can be verified computing the limit for $c' \to 0^+$ via l'H\^{o}pital rule, which gives
    \begin{equation}
        \lim_{c' \to 0^+} \frac{\mu_c(R) \sqrt{2R + \zeta^2}}{c} = \sqrt{\frac{2}{\pi}}, \qquad \lim_{c' \to 0^+} \frac{\nu_c(R) \left(2R + \zeta^2\right)}{c^2} = 1.
    \end{equation}
    Note that, as $R \geq 0$, the above limit is achieved when $c / \zeta \to 0$. Then, due to the inequality in \eqref{eq:c'}, the first and second part of the thesis follow.

    Note that, for $c' = 1$, we have
    \begin{equation}
        \nu_c(R) = 1 - \sqrt{\frac{2}{\pi e}} \approx 0.516.
    \end{equation}
    Thus, the third and fourth part of the thesis follow from the monotonicity of $\nu_c(R) / c'$ and $\nu_c(R)$.
\end{proof}

\section{Proofs for Section \ref{sec:DP-GD}}\label{app:deteq}

We will use the notation $\|Z\|_{\psi_p} = \inf\{t>0: \E \exp(|Z|^p/t^p)\le 2\}$ to denote the Orlicz norm of a random variable $Z$ for any $p \ge 1$. We denote the inner product between two vectors $a$ and $b$ as $\ip{a,b} = a^\top b$ (with notation applicable to matrices $\ip{aa^\top, A} = a^\top A a$). Given two real-valued quantities $a, b$, we denote by $a \wedge b = \min(a, b)$. Given a quadratic functions $q : \mathbb{R}^d \to \mathbb{C}$ we define its $\|\cdot\|_{C^2}$ norm such that $\|q\|_{C^2} = \|\nabla^2 q\|_{\text{op}} + \|\nabla q(0)\|_2 + |q(0)|$. Furthermore, we will denote with $\gamma_n$ the ratio $d / n$ for a fixed value of $n$.

We start by proving a general result about homogenized DP-GD, informally defined in Remark \ref{remark:H} and formally defined below. 

\begin{definition}[Homogenized DP-GD]\label{def:hdpsgd}
For any $t \in [0, 1)$, we define \emph{homogenized DP-GD} (H-DP-GD) as the solution of the SDE
\begin{equation} 
\begin{aligned}
    \diff \Theta_t = - 2 \bar \eta(t) {\mu_c(\Theta_t)}\nabla \mathcal{P}(\Theta_t) \diff t 
    + \bar \eta(t) \sqrt{\frac{2 \nu_c (\Theta_t) \mathcal{P}(\Theta_t) \Sigma}{n }} \diff B^s_t + 2 \frac{\sqrt d}{n} c \tilde \sigma(t) \diff B^p_t,
\end{aligned}
\end{equation}
where $\Theta_0 = \theta_0 = 0$, $B^s_t$ and $B^p_t$ are two independent standard Brownian motions in $\R^d$, $\bar \eta(t) = \min(\tilde \eta(t), 2n / d)$, 
and $\tilde \sigma(t)$ is such that $\rho^2 \tilde \sigma ^2(t) = - \diff \tilde \eta^2(t) / \diff t$.
\end{definition}

Theorem \ref{thm:H-DP-SGD_q} below shows that DP-GD is close to H-DP-GD over the class of quadratic functions $Q$, given by 
\begin{align}\label{eq:Q_class}
    Q = \{q : \mathbb{R}^d \to \mathbb{C} \; \textup{ s.t .}  \;  q(v) = v^\top R(z; \Sigma) v / 2, \; \textup{ with } \; z \in \Omega \},
\end{align}
where $R(z; \Sigma) = (\Sigma - z I)^{-1}$ is the resolvent matrix of $\Sigma$ and  $\Omega := \left\{w\in\mathbb{C}:|w| = 2 \|\Sigma\|_{\text{op}} \right\}$. 

\begin{theorem}\label{thm:H-DP-SGD_q}
Let Assumptions \ref{ass:data} and \ref{ass:learning_rate_schedules} hold. Let $\rho = \Theta(1)$, $n, d \to \infty$ s.t.\ $d / n \to \gamma \in (0,\infty)$.
Denote by $\Theta_t$ and $\theta_k$ independent realizations of H-DP-GD (as per Definition \ref{def:hdpsgd}) and Algorithm \ref{alg:dp-sgd}.
Then, with overwhelming probability, we have
\begin{align}
    \sup_{t\in[0,1)} \left |q(\Theta_t - \theta^*)- q(\theta_{\lfloor tn \rfloor} - \theta^*)\right| =O \left( \frac{\log^2 n}{\sqrt n} \right),
\end{align}
uniformly for all functions $q \in Q$.
\end{theorem}

\begin{proof}
Consider the notation $u_k = \theta_k - \theta^*$. We have the following update rule for $k \in [n]$:
\begin{align}
u_{k} =  u_{k - 1} - \bar \eta_k \bar g_{k} + 2 c \sqrt{d} \sigma_k b_{k},
\end{align}
where we recall that 
\begin{equation}
\bar \eta_k = \min \left( \eta_k, \frac{2}{\norm{x_k}_2} \right), \qquad \bar g_{k} = g_{k} \min \left(1, \frac{c\sqrt{d}}{\|g_k\|_2} \right),
\qquad g_{k} = \left( \ip{x_{k}, u_k} - z_k \right) x_{k}.
\end{equation}
Let $q: \R^d \to \R$ be a generic quadratic function. Then, using a Taylor expansion, the update rule for $q(u_k)$ reads
\begin{align}\label{eq:updateq}
q(u_{k+1}) &=  q(u_{k}) - {\bar \eta_k} \bar g_{k}^\top\nabla q(u_{k}) + 2 c \sqrt{d} \sigma_k b_{k}^\top\nabla q(u_{k}) 
\\\nonumber &\qquad + \frac{1}{2}\tr\left(({2 c\sqrt{d} \sigma_k b_{k}}- \bar \eta_k \bar g_{k})^{\otimes 2}  \nabla^2 q(u_{k})\right).
\end{align}

Defining the $\sigma$-algebra $\cF_k := \sigma(\{u_i\}_{i=0}^k)$ generated by the iterates of DP-GD in \eqref{eq:updateq}, we write the Doob's decomposition of the above process: 
\begin{align}\label{eq:SGD_q_Doob}
q(u_{k+1}) -  q(u_{k}) &= - \frac{{\bar \eta(k/n)}}{n} \mu_c(u_k)\ip{\Sigma u_k,\nabla q(u_{k})} \\
\nonumber & \qquad +  \frac{{\bar \eta^2(k/n)}}{{2} n}\nu_c(u_k)\mathcal{P}(u_k)\frac{1}{n}\tr( \Sigma\nabla^2 q(u_{k}))
+\frac{1}{ n}  \frac{2d}{n^2}c^2\tilde{\sigma}(k/n)^{2}\tr(\nabla^2 q(u_{k}))
\\\nonumber & \qquad + \Delta \mathcal{M}_k^{\text{Grad}}(q)+
\Delta \mathcal{M}_k^{\text{Hess}}(q)+
\E[\Delta \mathcal{E}_k^{\text{Hess}}(q)\mid \cF_k] 
\\\nonumber & \qquad + \mathcal{M}_k^{\text{Noise}}(q) +\E [\Delta \mathcal{E}_k^{\text{Noise}} (q) \mid \cF_k] + \E [\Delta \mathcal{E}_k^{\text{Step}} (q) \mid \cF_k],
\end{align}
{where we recall the notation $\bar \eta(k/n) = \min \left( \tdeta(k / n), 2/ \gamma_n \right)$, $\eta_k = \tilde \eta(k / n) / n$,} 
and where we introduced the shorthands
\begin{align}\label{eq:martingales}
\Delta \mathcal{M}_k^{\text{Grad}}(q) &:= - \bar \eta_k \ip{\bar g_{k}, \nabla q(u_{k})} + \bar \eta_k \mu_c(u_k)\ip{\Sigma u_k,\nabla q(u_{k})}
\\\nonumber  
\Delta \mathcal{M}_k^{\text{Hess}}(q)&:=\frac{1}{2}\ip{ (\bar \eta_k \bar g_{k})^{\otimes 2} , \nabla^2 q(u_{k})} - \frac{1}{2}\ip{ \E[(\bar \eta_k \bar g_{k})^{\otimes 2} \mid \mathcal{F}_k], \nabla^2 q(u_{k})}
\\\nonumber 
\E[\Delta \mathcal{E}_k^{\text{Hess}}(q)\mid \cF_k] &:= \frac{1}{2}\ip{ \E[(\bar \eta_k \bar g_{k})^{\otimes 2} \mid \mathcal{F}_k], \nabla^2 q(u_{k})}
\\\nonumber
& \qquad - \frac{n \bar \eta_k^2}{2}\nu_c(u_k)\mathcal{P}(u_k)\frac{1}{n}\tr( \Sigma\nabla^2 q(u_{k}))
\\\nonumber 
\Delta \mathcal{M}_k^{\text{Noise}}(q)&:= 2 c\sqrt{{d}}{\sigma}_k\ip{b_{k},\nabla q(u_{k})} + \frac{1}{2 \sqrt{n}}\ip{ 2 c \sqrt{\frac{d}{n}}{\sigma}_k b_{k} {n \bar \eta_k} \bar g_{k}^\top , \nabla^2 q(u_{k})} \\\nonumber& \qquad + \frac{d}{2 }\ip{ (2 c{\sigma}_k b_{k})^{\otimes 2} , \nabla^2 q(u_{k})} - 2dc^2\ip{  {\sigma}_k^{2}I_d, \nabla^2 q(u_{k})}
\\\nonumber 
\E [\Delta \mathcal{E}_k^{\text{Noise}} (q) \mid \cF_k]&: =  {2d}c^2\ip{ \left(\sigma_k^2 - \frac{1}{ n^{3}}\tilde{\sigma}(k/n)^{2}\right)I_d, \nabla^2 q(u_{k})},
\\\nonumber 
{\E [\Delta \mathcal{E}_k^{\text{Step}} (q) \mid \cF_k]} &:= \left(\frac{\bar \eta(k/d)}{n} - \bar \eta_k \right) \mu_c(u_k)\ip{\Sigma u_k,\nabla q(u_{k})}
\\\nonumber
& \qquad + \left(\frac{\bar \eta_k^2}{2} - \frac{\bar \eta(k/n)^2 }{2 n^2}\right) \nu_c(u_k)\mathcal{P}(u_k)\tr( \Sigma\nabla^2 q(u_{k})).
\end{align}
We note that the first three terms are in common with the analysis in \cite{marshall2025clip}, while the last three are the result of the private noise and the adaptive learning rate step in Algorithm \ref{alg:dp-sgd}.

In a similar way, we introduce the shorthand $V_t = \Theta_t - \theta^*$ such that $V_0 = u_0$. Using It\^{o}'s formula on \eqref{eq:HDPSGD}, for any quadratic function $q$, we have that
\begin{equation}\label{eq:DPHSGD_q}
\begin{aligned}
    \diff q(V_t) = & - \bar \eta(t) \mu_c(V_t) \ip{\Sigma V_t, \nabla q(V_t)} \diff t 
    + \bar \eta^2(t) \nu_c (V_t) \mathcal{P}(V_t) \frac{1}{n} \tr \left( \Sigma \nabla^2 q(V_t) \right) \diff t \\
    & \qquad + 2 \frac{d}{n^2} c^2 \tilde \sigma^2(t)  \tr \left(\nabla^2 q (V_t) \right) \diff t + \diff \mathcal M_t^{\text{H-DP-GD}},
\end{aligned}
\end{equation}
where we introduced the shorthand
\begin{equation}
    \diff \mathcal M_t^{\text{H-DP-GD}}: = \ip{\nabla q(V_t), \sqrt{\frac{2 \bar \eta(t) ^2 \nu_c (\Theta_t) \mathcal{P}(\Theta_t) \Sigma}{n } + 4 \frac{ d}{n^2} c^2 \tilde \sigma(t)^2 I_d} \, \diff B_t},
\end{equation}
with $B_t$ being a $d$-dimensional standard Brownian motion.

Let $M$ be a positive constant that will be fixed later. The dynamic is first controlled up to the stopping time 
\begin{equation}\label{eq:taustopping}
    \tau := \inf\{k: \|u_k\|_2 \geq M \cup \lfloor tn \rfloor : \|V_t\|_2 \geq M\}.
\end{equation}
Then, we will denote the stopped processes $u_k^\tau = u_{k\wedge \tau}$ and $V_t^\tau = V_{t\wedge (\tau/n)}$, which will be the objects we will compare in the following arguments. This stopping time is introduced for technical reasons, 
and we will later show that $\tau \geq n$.

Taking the difference between \eqref{eq:SGD_q_Doob} and \eqref{eq:DPHSGD_q}, 
and following the same argument in Lemma 1 in \cite{marshall2025clip}, 
we get that 
there exists an absolute constants $C_1 = C_1(\|\Sigma\|_{\op}, c, \bar{\eta})$, 
such that 
\begin{multline}\label{eq:differenceintm}
\sup_{0\le t< 1} \left|q(u_{\lfloor t n\rfloor}^\tau) - q(V_t^\tau)\right|
\le C_1 \int_0^{1} \sup_{q\in Q}\left|q(u_{\lfloor s n\rfloor}^\tau) - q(V_s^\tau)\right|\diff s
\\+ \sup_{0\le t< (1\wedge (\tau/n))} \left( |\sum_{k=1}^{\lfloor tn \rfloor} \E[\Delta \mathcal{E}_k^{\textup{Hess}}(q)\mid \cF_k ]|+|\sum_{k=1}^{\lfloor tn \rfloor} \E[\Delta \mathcal{E}_k^{\textup{Noise}}(q)\mid \cF_k ]| + |\sum_{k=1}^{\lfloor tn \rfloor} \E [\Delta \mathcal{E}_k^{\text{Step}} (q) \mid \cF_k]| \right)
\\ +\sup_{0\le t< (1\wedge (\tau/n))}\left(|\mathcal{M}_{\lfloor t n\rfloor}^{\textup{Grad}}(q)|+|\mathcal{M}_{\lfloor t n\rfloor}^{\textup{Hess}}(q)|+|\mathcal{M}_{\lfloor t n\rfloor}^{\textup{Noise}}(q)|  +|\mathcal{M}_{t}^{\textup{H-DP-GD}}(q)| \right) + O(d^{-1}),
\end{multline}
where we introduced the notation $\mathcal M_k(q) = \sum_{j = 1}^k \Delta \mathcal M_j(q)$ and $\mathcal{M}_{t}^{\textup{H-DP-GD}}(q) = \int_0^t \diff \mathcal{M}_{s}^{\textup{H-DP-GD}}(q)$. The last term follows from transitioning \eqref{eq:SGD_q_Doob} to continuous times, which involves an additional discretization error of $O(d^{-1})$ (see also Section A.3 in \cite{collins2024hitting} for more details). 
Importantly, differently from Lemma 1 in \cite{marshall2025clip}, we miss a term proportional to $\left(\mathcal{R}(\Theta_s)+\mathcal{R}(\theta_{\lfloor sn\rfloor})\right)^{-1/2}$ in \eqref{eq:differenceintm}. This is possible since, as $\zeta > 0$, both $\mu_c(\theta)$ and $\nu_c(\theta) \mathcal P(\theta)$ are Lipschitz with respect to $\mathcal R(\theta)$ with constant-order Lipschitz constant due to Lemma \ref{lemma:munubounds}, strengthening the condition in their Assumption 5.

Denoting with $\mathcal M$ the sum of the last two lines in \eqref{eq:differenceintm}, by Lemma 2 in \cite{marshall2025clip} (for $|\mathcal{M}_{\lfloor t n\rfloor}^{\textup{Grad}}(q)|$, $|\mathcal{M}_{\lfloor t n\rfloor}^{\textup{Hess}}(q)|$ and $|\sum_{k=1}^{\lfloor tn \rfloor} \E[\Delta \mathcal{E}_k^{\textup{Hess}}(q)\mid \cF_k ]|$), Lemma \ref{lem:mar_noise} (for $|\mathcal{M}_{\lfloor t n\rfloor}^{\textup{Noise}}(q)|$ and $|\sum_{k=1}^{\lfloor tn \rfloor} \E[\Delta \mathcal{E}_k^{\textup{Noise}}(q)\mid \cF_k ]|$), Lemma \ref{lem:M_DPSGD} (for $|\mathcal{M}_{t}^{\textup{H-DP-GD}}(q)|$), and Lemma \ref{lemma:step} (for $|\mathcal M_{\lfloor t n\rfloor}^{\text{Step}} (q)|$), there are two constants $C_2(\|\Sigma\|_{\op}, \bar{\eta}, M, c, \gamma)$ and $C_3(\|\Sigma\|_{\op}, \bar{\eta}, M, c)>0$ such that 
\begin{equation}\label{eq:Msmall}
    \mathcal M \leq C_2 n^{-1/2} \left( \log^2 n + C_3 \right),
\end{equation}
with probability at least $1- e^{-\log^2 n}$. 

Then, following Lemma 3 in \cite{marshall2025clip}, we can define a set $\bar Q \subseteq Q$ with $|\bar Q| \le C_4(\|\Sigma\|_{\op}) d^{4}$, such that, for all $q \in Q$, there exists a $\bar q \in \bar Q$ that satisfies $\|q - \bar q \|_{C^2} \leq d^{-2}$. Then, taking the union bound over this set yields 
\begin{equation}\label{eq:beforeLemma3}
\sup_{q\in Q} \sup_{0\le t< 1} \left|q(u_{\lfloor t n\rfloor}^\tau) - q(V_t^\tau)\right|
\le C_1 \int_0^{1} \sup_{q\in Q}\left|q(u_{\lfloor s n\rfloor}^\tau) - q(V_s^\tau)\right|\diff s
 + \mathcal M \, + \, C_5 d^{-2},
\end{equation}    
with probability at least $1- C_4  d^4 e^{- \log^2 n}$. 
Thus, with this same probability, the application of Gronwall's inequality gives
\begin{equation}\label{eq:q_SGD_DPSGD}
     \sup_{q\in Q} \sup_{0\le t< 1}  \left|q(u_{\lfloor t n\rfloor}^\tau) - q(V_t^\tau)\right| \leq \mathcal (\mathcal M + C_5 d^{-2}) \exp \left( C_1 \right) = O \left( \frac{\log^2 n}{\sqrt n}\right),
\end{equation}  
where the last step hides a dependence on $M$ (see \eqref{eq:taustopping}). Then, we are left to show that $\tau \geq n$ for $M$ chosen to be a constant independent of $d$ and $n$. This follows the same approach as in Lemma 4 in \cite{marshall2025clip}, which is here formalized in Lemma \ref{lem:tau_removal}. In particular, in Lemma \ref{lem:tau_removal}, we show that there exists a constant $C_6(\|\Sigma\|_{\text{op}}, c, \bar \eta) > 0$, such that for any $r \ge 0$, with probability at least $1 - 2 e^{-r^2/2}$, it holds that $\sup_{t \in[0,1]} \|u_{\lfloor tn \rfloor}\|^2 \le \|u_0\|^2 e^{C_6 (1 + d^{-1/2}r)}$ and $\sup_{t\in[0,1]} \|V_t\|^2 \le \|V_0\|^2 e^{C_6 (1 + d^{-1/2}r)}$. Thus, $M$ can be chosen as a constant independent from $d$ and $n$, such that $\tau > n$ with overwhelming probability, and the thesis follows.
\end{proof}

\paragraph{Proof of Remark \ref{remark:H}.}
The result follows the same proof of Theorem \ref{thm:H-DP-SGD_q}, but considering $q( \theta - \theta^*) = \norm{\Sigma^{1 / 2} (\theta - \theta^*)}_2^2 / 2$. While this function is not included in $Q$, it is still such that $\norm{q}_{C^2} = O(1)$, allowing the same argument to go through. Alternatively, it can be recovered by the Cauchy integral formula taking $ \frac{i}{4 \pi} \oint_{\Omega} z q(\cdot) \diff z$ (see, e.g., \eqref{eq:firstcontours}). \qed



\subsection{Technical lemmas for Theorem \ref{thm:H-DP-SGD_q}}\label{app:lemmasdeter}

In this section, we present the statements and proofs of the technical lemmas used in the proof of Theorem \ref{thm:H-DP-SGD_q}. All notation is defined as needed throughout, and Assumptions \ref{ass:data} and \ref{ass:learning_rate_schedules} are assumed to hold unless otherwise specified. For notational simplicity, we use the shorthand $\|\cdot\|$ to denote the Frobenius norm for matrices. 

\begin{lemma}\label{lem:mar_noise}
For any quadratic $q \in Q$ such that $\|q\|_{C^2} \leq 1$, it holds that
\begin{align}\label{eq:firsteqfirstmart}
   \left | \sum_{k=1}^{\lfloor tn \rfloor} \E[\Delta \mathcal{E}_k^{\textup{Noise}}(q)\mid \cF_k ] \right|\le \frac{2 \gamma_n }{\rho^2 n} C_{\eta,2}, \quad \text{a.s.}
\end{align}
where $C_{\eta,2}$ denotes the upper bound on the absolute value of the second derivative of $\tilde \eta^2(t)$. In addition, there is a constant $C = C(c, \gamma, \rho, M)>0$, such that, for any $y\in[1, n]$,
\begin{align}
 \sup_{1\le k\le {(n\wedge \tau)}} \, |\mathcal{M}_k^{\textup{Noise}}(q) |  \le C n^{-\frac{1}{2}} y,
\end{align}
with probability at least $1-e^{-y}.$
\end{lemma}
\begin{proof}
Recall the definition in \eqref{eq:martingales}
\begin{equation}
\begin{aligned}
    \Delta \mathcal{M}_k^{\text{Noise}}(q) & = 2 c\sqrt{{d}}{\sigma}_k\ip{b_{k},\nabla q(u_{k})} + \frac{1}{2 \sqrt{n}} 2 c \sqrt{\frac{d}{n}}{\sigma}_k  n \bar \eta_k \bar g_{k}^\top \nabla^2 q(u_{k})b_{k} \\\nonumber& \qquad + 2d c^2 \sigma_k^2\tr( (b_{k})^{\otimes 2}  \nabla^2 q(u_{k})) - 2dc^2{\sigma}_k^{2}\tr(\nabla^2 q(u_{k})).
\end{aligned}
\end{equation}
Then, for any $k\le \tau,$ we rewrite the martingale as a combination of the following three terms
\begin{align} 
\Delta \mathcal{M}_k^{\textup{Noise}}(q)   = \frac{1}{n^{3/2}}\ip{b_{k}, A_{k,1} + A_{k,2}} + \frac{1}{ n^2}\tr\left(b_{k}^{\otimes 2 } - I_d\right),
\end{align}
where we introduced the shorthands
\begin{equation}
    A_{k,1} := 2 c n^{3/2} \sqrt{{d}}{\sigma}_k \nabla q(u_{k}), \; A_{k,2} := c n \sqrt{\frac{d}{n}}{\sigma}_k n\bar \eta_k \nabla^2 q(u_{k})\bar g_{k} , \; C_k := 2dn^2c^2{\sigma}_k^{2}\nabla^2 q(u_{k}).
\end{equation}

We will separately bound the contribution of each term in terms of its Orlicz norm. Let us start with the second term, and 
consider $q\in Q$. This is a quadratic function of the iterates; its Hessian, therefore, does not depend on the iterates and explicitly on $k.$ In addition, we have that $\sup_{z\in \Omega}\|R(z;\Sigma)\|_{\text{op}} \le 2$. Since $b_{k}\sim\mathcal{N}(0,I_d)$ and independent, using the Hanson-Wright inequality (Theorem 6.2.1 in \cite{vershynin2018high}) 
we have that, for some $c>0$ and $K = \max_{i\in [d]}\|b_{k+1,i}\|_{\psi_2}$,
\begin{align*}
    \mathbb{P}\left(\frac{1}{ n^2}\tr\left((b_{k}^{\otimes 2 } - I_d\right) C_{k})\ge t\right) &\leq 2 \exp\left(-c \min\left(\frac{t^2 n^4}{K^4\|C_k\|^2}, \frac{tn^2}{K^2\|C_k\|_{\text{op}}}\right)\right)
    \\\nonumber &\le 
2 \exp\left(-c'\min\left(\frac{t^2 n^4}{d}, {t n^2}\right)\right),
\end{align*}
for some $c'>0$ independent of $k,d,n.$ To justify the last passage, note that, by the structure of the noise, $$\sigma_k^2 \le \frac{1}{\rho^2n^2}|\tdeta({k}/{n})^2 - \tdeta((k+1)/n)^2|\le \frac{2}{\rho^2n^3} \max_{x \in [\frac{k}{n},\frac{k+1}{n}]} \left|\frac{\diff}{\diff x}\tdeta^2 (x)\right| \le \frac{2}{\rho^2n^3} C_{\eta,1}, $$
as  $\left| \frac{\diff}{\diff x}\tdeta^2 (x) \right|\le C_{\eta,1} $ for some $C_{\eta,1}$ due to Assumption \ref{ass:learning_rate_schedules}. Using the above bound, we obtain that $\|C_k\|_{\op} \le  2dn^2 c^2 \sigma_k^2 \|q\|_{C^2} \le {8} c^2/\rho^2 C_{\eta,1},$ and similarly {$\|C_k\| \le 8\sqrt{d}c^2/\rho^2C_{\eta,1}.$}
We, therefore, have that 
$\|\frac{1}{n^2}\ip{b_{k}^{\otimes 2 } - I_d , C_{k}}\|_{\psi_1} \le C   n^{-1} ,$ for some constant $C(\rho, c, C_{\eta,1})>0.$

For the first term, by Eq. (89) in  \cite{marshall2025clip}, the norm of $q$ is bounded for the stopped process $u_k^\tau$ as follows:
\begin{align}\label{eq:qnorm_bound}
    \|\nabla q(u)\|\le \|\nabla^2 q\|_{\text{op}}\|u\|+ \|\nabla q(0)\|\le \|q\|_{C^2} (1+\|u\|) \le C(1+M).
\end{align} 
Hence, we have that $\|A_{k,1}\|\le 2 n^{3/ 2} c \sqrt{n}\sigma_k C(1+M) \le 4 \sqrt d c \sqrt{ C_{\eta,1} / \gamma_n}C(1+M)/\rho $, which implies that 
\begin{align}
\left\| \frac{1}{{n^{3/2}}}\ip{b_{k}, A_{k,1}}\right\|_{\psi_2} \le \frac{C}{n},    
\end{align}
with $C = C(M,\rho, \gamma_n, C_{\eta,1}, c)>0.$
Then,
we have
\begin{align} 
\left|\frac{1}{{n}^{3/2}}\ip{b_{k}, A_{k,2}} \right|& 
\le \frac{1}{ n^{3/2}}  c n \sqrt{\gamma_n}{\sigma}_k \tdeta(k/ n)  \left|  \ip{b_{k},\nabla^2 q(u_{k})g_{k}} \right| 
\\\nonumber &\le \frac{2 c\sqrt{C_{\eta,1}}}{ n \sqrt{\gamma_n} \rho} \left|\ip{b_{k},\nabla^2 q(u_{k})x_{k+1}}(\ip{x_{k+1},u_k} +z_{k+1}) \right|.
\end{align}
Due to \eqref{eq:qnorm_bound} and Assumptions \ref{ass:data} and \ref{ass:learning_rate_schedules}, we have $ \|\ip{b_{k},\nabla^2 q(u_{k})x_{k+1}} \|_{\psi_1}\le C\|b_{k}\|_{\psi_2}\|x_{k+1}\|_{\psi_2}\le C$ for some $C>0$ that depends on $\|\Sigma\|_{\op}.$ Finally, we note that $\|\ip{x_{k+1},u_k}\|_{\psi_2} \le C(1+M)$, and $\|z_{k}\|_{\psi_2}\le C\zeta$ for any $k\le \tau$.
Combining the above, we obtain that there is a constant $C = C(c, C_{\eta,1}, \rho, \gamma_n, M, \zeta)>0$ such that 
\begin{align}
\phi_{k,1} := \inf \{t>0: \, \E[\exp (|\Delta \mathcal{M}_k^{\text{Noise}}(q)|/t)\mid \mathcal{F}_{k-1} ] \le 2\}\le Cn^{-1} . 
\end{align}
We then apply Lemma 5 in \cite{marshall2025clip} for some absolute constant $C>0$ for all $t>0$: 
\begin{align}
 \P\left(\sup_{1\le k \leq n\wedge \tau } |\mathcal{M}_k^{\text{Noise}}(q) - \E \mathcal{M}_0^{\text{Noise}}(q)|\geq t\right) &\le 2\exp\left( - \min \{\frac{t}{C\max_{k\in [n]} \phi_{k,1}}, \frac{t^2}{C\sum_{i=1}^n \phi_{i,1}^2}\}\right) 
 \\\nonumber &\le  2\exp\left( - Cn\min \{t, t^2\}\right) .
\end{align}
As we assume that $n$ is proportional to $d$ and noting that $\E \mathcal{M}_0^{\text{Noise}}(q) = 0$ 
by definition, we then have that, for any $y\in[1, n]$,
\begin{align}
 \sup_{1\le k\le {(n\wedge \tau})} \, |\mathcal{M}_k^{\textup{Noise}}(q)|
 \le Cn^{-\frac{1}{2}} y,
\end{align}
with probability at least $1-e^{-y}$ for any $y\in [1, n].$

Next, we bound the error due to the discretization: 
\begin{align*}
   \left | \sum_{k=1}^{\lfloor tn \rfloor} \E[\Delta \mathcal{E}_k^{\textup{Noise}}(q)\mid \cF_k ] \right|&\le  {2d}c^2\sum_{k=1}^{\lfloor tn \rfloor} \left|\sigma_k^2 - \frac{1}{ n^{3}}\tilde{\sigma}(k/n)^{2}\right|\cdot \left|\tr(\nabla^2 q(u_{k}))\right| 
   \\\nonumber &\le \frac{2d^2 \| q\|_{C^2}}{n^2\rho^2}c^2\sum_{k=1}^{\lfloor tn \rfloor} \left|\tdeta(k/n)^2 - \tdeta(({k+1})/n)^2 - \frac{1}{ n}\tilde{\sigma}(k/n)^{2}\right|
   \\\nonumber & \le \frac{2 \gamma_n^2 c^2 \| q\|_{C^2}}{\rho^2 n^2}\sum_{k=1}^{\lfloor tn \rfloor} \max_{x \in (\frac{k}{n}, \frac{k+1}{n})}\left|\frac{\diff^2 }{\diff x^2}\tdeta(x)^2\right|,
\end{align*}
where we use the definition of the noise function as the derivative of the learning rate $\rho^2\tilde{\sigma}(x)^{2} = - \frac{\diff }{\diff x}\tdeta^2(x)$ at any point $x\in[0,1).$ Then, as $|\frac{\diff^2}{\diff x^2}\tdeta^2 (x) | \le C_{\eta,2}$ for some constant $C_{\eta,2}$, the desired result follows.
\end{proof}

\begin{lemma}\label{lem:M_DPSGD}
    Denote by $\mathcal{M}_t^{\textup{H-DP-GD},\tau}$ the H-DP-GD martingale in which the stopping time is imposed. There is a constant $C = C(c,C_{\eta,1},\gamma_n,\|\Sigma\|_{\op}, M)>0$, such that for any quadratic $q\in Q$ with $\|q\|_{C^2}\le 1,$ and for any $y\in[1, n]$, we have
\begin{align}
 \sup_{0\le  t< {1}} \, |\mathcal{M}_t^{\textup{H-DP-GD},\tau}(q)|  \le C n^{-\frac{1}{2}} y,
\end{align}
with probability at least $1-e^{-y}.$ 
\end{lemma}
\begin{proof}
To bound the martingale error from H-DP-GD under some general statistic $q\in Q$, differently from the argument to obtain Eq. (72) in Lemma 2 in \cite{marshall2025clip}, we need to control the additional term due to the additive noise in Algorithm \ref{alg:dp-sgd}.

Using It\^{o}'s formula for any quadratic function $q$:
\begin{equation}\label{eq:DPHSGD_q_lemma}
\begin{aligned}
    \diff q(V_t) = & - \bar \eta(t) \mu_c(\Theta_t) \nabla \mathcal{P}(\Theta_t)^\top \nabla q(V_t) \diff t 
    + \bar \eta^2(t) \nu_c (\Theta_t) \mathcal{P}(\Theta_t) \frac{1}{n} \tr \left( \Sigma \nabla^2 q \right) \diff t \\
    & \qquad + 2 \frac{d}{n^2} c^2 \tilde \sigma^2(t)  \tr \left(\nabla^2 q \right) \diff t + \diff \mathcal M_t^{\text{H-DP-GD}},
    \\
    \diff \mathcal M_t^{\text{H-DP-GD}}&: = \ip{\nabla q(V_t), \sqrt{\frac{2\bar \eta(t) ^2 \nu_c (\Theta_t) \mathcal{P}(\Theta_t) \Sigma}{n } + 4 \frac{ d}{n^2} c^2 \tilde \sigma(t)^2 I_d} \,  \diff B_t},
\end{aligned}
\end{equation}
with $B_t$ being a standard $d$ dimensional Brownian motion. 
The quadratic variation of the martingale is then bounded a.s. 
$$\ip{\mathcal{M}(q)}_t\le \frac{C t}{n} + \frac{d}{n^2} (4c^2\bar{\eta}^2\|\Sigma\|_{\op} (1+M)^2),$$
for some constant $C = C(\|\Sigma\|_{\op}, M, C_{\eta,1}, c, \gamma_n)>0,$
where we used Assumption \ref{ass:learning_rate_schedules} which gives $|\rho^2 \tilde \sigma(t) | = |\frac{\diff }{\diff t}\tdeta^2(t)| \le C_{\eta,1}$.
The claim is then proved by an application of Gaussian concentration inequalities as in Section B.6 in \cite{marshall2025clip}.
\end{proof}

\begin{lemma}\label{lem:tau_removal}
    For any $r \geq 0$, with probability at least $1 - 2 e^{-r^2/2}$, we jointly have
    \begin{equation}
        \sup_{t\in[0,1)} \|V_t\|^2 \le \|V_0\|^2 e^{C ( 1 + d^{-1/2}r)}, \qquad \sup_{t\in[0,1)} \|u_{\lfloor nt \rfloor}\|^2 \le \|u_0\|^2 e^{C ( 1 + d^{-1/2}r)}.
    \end{equation}
\end{lemma} 
\begin{proof}
The proof follows a path similar to the one of Lemma 4 in \cite{marshall2025clip}. In particular, consider the function $\varphi(V_t) = \log(1+\|V_t\|^2).$  Then, by application of It\^{o}'s Lemma to \eqref{eq:HDPSGD},   
\begin{equation}
\begin{aligned}
    \diff \varphi(V_t) = & - \bar \eta(t) \frac{\mu_c(\Theta_t)}{1+\|V_t\|^2} \nabla \mathcal{P}(\Theta_t)^\top V_t\diff t \\
    & \qquad + {2}\bar \eta^2(t) \frac{\nu_c (\Theta_t) \mathcal{P}(\Theta_t)}{(1+\|V_t\|^2)^2} \frac{1}{n} \tr \left( \Sigma (I_d (1+\|V_t\|^2) - V_t \otimes V_t )\right) \diff t \\
    & \qquad + 2 \frac{d}{n} c^2 \tilde \sigma^2(t) \frac{1}{(1+\|V_t\|^2)^2} \diff t + \diff \mathcal{M}_t(\varphi) , 
\end{aligned} 
\end{equation}
with 
\begin{equation}
    \diff \mathcal{M}_t(\varphi) =  \frac{1}{(1+\|V_t\|^2)}\ip{V_t, \sqrt{\frac{2\bar \eta(t) ^2 \nu_c (\Theta_t) \mathcal{P}(\Theta_t) \Sigma}{n } + 4 \frac{ d}{n^2} c^2 \tilde \sigma(t)^2 I_d} \, \diff B_t}.
\end{equation}
The drift terms and the quadratic variation terms can be bounded by some $C(c,\|\Sigma\|_{\op }, \zeta, \gamma_n, M)$. The quadratic variation of the martingale term is bounded as 
$\ip{M}_t \le \frac{C t}{n}.$ We then have, by Gaussian concentration, \begin{align}
    \P\left(\sup_{0\le t< 1} \varphi(V_t)\ge C(1 + r/\sqrt{n}) \right) \le e^{-r^2/2},
\end{align} 
which proves the first part of the claim. The bound on $\sup_{t\in[0,1)} \|u_{\lfloor nt \rfloor}\|^2$ follows the same strategy as the argument following Lemma 4 in \cite{marshall2025clip}.
\end{proof}

\begin{lemma}\label{lemma:step}
For any quadratic $q \in Q$ such that $\|q\|_{C^2} \le 1$, 
\begin{align}
 \sup_{1\le k\le {(n\wedge \tau)}} \, |\sum_{k=1}^{n} \E [\Delta \mathcal{E}_k^{\textup{Step}} (q) \mid \cF_k]|  \le C n^{-\frac{1}{2}} \log^2 n,
\end{align}
with probability at least $1-e^{-c \log^2 d}$.
\end{lemma}
\begin{proof}
Recall that
\begin{align}
\E [\Delta \mathcal{E}_k^{\textup{Step}} (q) \mid \cF_k] &:= \left(\frac{\bar \eta(k/d)}{n} - \bar \eta_k \right) \mu_c(u_k)\ip{\Sigma u_k,\nabla q(u_{k})}
\\\nonumber
& \qquad + \left(\frac{\bar \eta_k^2}{2} - \frac{\bar \eta(k/n)^2}{2 n^2}\right) \nu_c(u_k)\mathcal{P}(u_k)\tr( \Sigma\nabla^2 q(u_{k})),
\end{align}
and $\norm{u_k}_2^2 < M$ by the definition of $\tau$. Thus, this implies
\begin{equation}
    \left| \mu_c(u_k)\ip{\Sigma u_k,\nabla q(u_{k})} \right| = O(1), \qquad \left| \nu_c(u_k)\mathcal{P}(u_k)\tr( \Sigma\nabla^2 q(u_{k})) \right| = O(d).
\end{equation}

First, notice that
\begin{equation}
    \frac{\bar \eta(k/d)}{n} - \bar \eta_k = \min \left( \frac{\tilde \eta(k / d)}{n} , \frac{2}{d} \right) - \min \left( \frac{\tilde \eta(k / d)}{n} , \frac{2}{\norm{x_k}_2^2} \right),
\end{equation}
which implies
\begin{equation}
    \left| \frac{\bar \eta(k/d)}{n} - \bar \eta_k \right| \leq 2 \left| \frac{1}{d} - \frac{1}{\norm{x_k}_2^2}\right| =  \frac{2 \left| \norm{x_k}_2^2 - d \right|}{d \norm{x_k}_2^2}.
\end{equation}
Similarly, it can be shown that
\begin{equation}
    \frac{1}{2} \left| \frac{\bar \eta(k/d)^2}{n^2} - \bar \eta_k^2 \right| \leq \frac{2 \left| \norm{x_k}_2^2 - d \right| \left( \norm{x_k}_2^2 + d  \right)}{d^2 \norm{x_k}_2^4}.
\end{equation}


Hanson-Wright inequality gives 
\begin{equation}
    \P \left( \left| \norm{x_k}_2^2 - d \right| > \sqrt d \log d \right) \leq 2 \exp \left( - c_1 \log^2 d \right),
\end{equation}
which is sufficient to show that, with probability $1 - 2 \exp \left( - c_1 \log^2 d \right)$,
\begin{equation}
    \left| \E [\Delta \mathcal{E}_k^{\textup{Step}} (q) \mid \cF_k] \right| = O \left( \frac{\log d}{d^{3/ 2}} \right),
\end{equation}
which yields the thesis after performing a union bound over $k \in [n]$.
\end{proof}

\subsection{Proof of Theorem \ref{thm:deteq}}\label{app:deteqchange}


Consider \eqref{eq:DPHSGD_q_lemma}, which reads
\begin{equation}
\begin{aligned}
    \diff q(V_t) = & - \bar \eta(t) \mu_c(\Theta_t) \nabla \mathcal{P}(\Theta_t)^\top \nabla q(V_t) \diff t 
    + \bar \eta^2(t) \nu_c (\Theta_t) \mathcal{P}(\Theta_t) \frac{1}{n} \tr \left( \Sigma \nabla^2 q \right) \diff t \\
    & \qquad + 2 \frac{d}{n^2} c^2 \tilde \sigma^2(t)  \tr \left(\nabla^2 q \right) \diff t + \diff \mathcal M_t^{\text{H-DP-GD}}(q),
\end{aligned}
\end{equation}
where we recall the notation $V_t = \Theta_t - \theta^*$, and $q(V_t) = V_t^\top (\Sigma-zI)^{-1} V_t / 2$, and where in Lemma \ref{lem:M_DPSGD} we showed that $\sup_{0\le t< {1}} \, | \int_0^t \mathcal{M}_\tau^{\textup{H-DP-GD}}(q) \diff \tau| \le C n^{-\frac{1}{2}} y$ with probability at least $1-e^{-y}$. Then, we have
\begin{equation}
\begin{aligned}
    \diff q(V_t) = & - \bar \eta(t) \mu_c(\Theta_t) V_t^\top \Sigma (\Sigma - zI)^{-1} V_t \diff t 
    + \bar \eta^2(t) \nu_c (\Theta_t) \mathcal{P}(\Theta_t) \frac{1}{n} \tr \left( \Sigma (\Sigma - zI)^{-1} \right) \diff t \\
    & \qquad + 2 \frac{d}{n^2} c^2 \tilde \sigma^2(t)  \tr \left((\Sigma - zI)^{-1}\right) \diff t + \diff \mathcal M_t^{\text{H-DP-GD}}(q),
\end{aligned}
\end{equation}
which, using the identity $\Sigma (\Sigma - zI)^{-1} = I + z(\Sigma - zI)^{-1}$, can be written as
\begin{equation}\label{eq:resolvantwithmartingale}
\begin{aligned}
    \diff q(V_t) = & - \bar \eta(t) \mu_c(\Theta_t) \left(\norm{V_t}_2^2 + z q(V_t) \right) \diff t 
    + \bar \eta^2(t) \nu_c (\Theta_t) \mathcal{P}(\Theta_t) \gamma_n \frac{\tr \left( \Sigma (\Sigma - zI)^{-1} \right)}{d} \diff t \\
    & \qquad + 2  c^2 \tilde \sigma^2(t) \gamma_n^2 \frac{\tr \left((\Sigma - zI)^{-1}\right)}{d} \diff t + \diff \mathcal M_t^{\text{H-DP-GD}}(q).
\end{aligned}
\end{equation}

Notice that
\begin{equation}\label{eq:firstcontours}
    q(V_t) = 
    \frac{1}{2} \sum_{i = 1}^d \frac{\ip{V_t, \omega_i}^2}{\lambda_i - z}, \qquad \norm{V_t}_2^2 = \frac{i}{\pi} \oint_\Omega q(V_t), \qquad \mathcal R(\Theta_t) = \frac{i}{2 \pi} \oint_\Omega z q(V_t),
\end{equation}
where the last two equations follow from Cauchy integral formula. Then, we can define the following (deterministic) integro-differential equation after removing the martingale term from \eqref{eq:DPHSGD_q_lemma}: 
\begin{equation}\label{eq:resolvantwithoutmartingale}
\begin{aligned}
    \diff \hat q(z, t) = & - \bar \eta(t) \mu_c \left( \frac{i}{2 \pi} \oint_\Omega z \hat q(z, t) \right) \left(\frac{i}{\pi} \oint_\Omega \hat q(z, t) + z \hat q(z, t)  \right) \diff t \\
    & \qquad + \bar \eta^2(t) \nu_c \left( \frac{i}{2 \pi} \oint_\Omega z \hat q(z, t) \right) \left(\frac{\zeta^2}{2} + \frac{i}{2 \pi} \oint_\Omega z \hat q(z, t) \right) \gamma_n \frac{\tr \left( \Sigma (\Sigma - zI)^{-1} \right)}{d} \diff t \\
    & \qquad + 2  c^2 \tilde \sigma^2(t) \gamma_n^2 \frac{\tr \left((\Sigma - zI)^{-1}\right)}{d} \diff t,
\end{aligned}
\end{equation}
where all differentials are taken with respect to time. \eqref{eq:resolvantwithoutmartingale} is defined for $z \in \Omega$, $t \in [0, 1]$, and is such that $\hat q(z, 0) = q(V_0) = {\theta^*}^\top (\Sigma - z I)^{-1} \theta^* / 2$.

Using the same argument as in the proof of Theorem \ref{thm:H-DP-SGD_q} based on Gronwall's inequality, relying on $\zeta > 0$ which makes both $\mu_c(\theta)$ and $\nu_c(\theta) \mathcal P(\theta)$ Lipschitz with respect to $\mathcal R(\theta)$ with constant-order Lipschitz constant, and since the length of $\Omega$ is bounded, we have
\begin{equation}
    \sup_{t\in[0,1)} \left| \mathcal{R}(\Theta_t) - R(t) \right| = O \left( \frac{ \log^2 n}{\sqrt n} \right),
\end{equation}
with overwhelming probability, where we have introduce the shorthand $R(t) = \frac{i}{4 \pi} \oint_\Omega z \hat q(z, t)$.

Then, due to Lemma 4.1 in \cite{collins2024hitting}, we have that the unique solution of \eqref{eq:resolvantwithoutmartingale} is
\begin{equation}
    \hat q(z, t) = \frac{1}{d} \sum_{i = 1}^d \frac{D_i(t)}{\lambda_i - z},
\end{equation}
where $D_i(t)$ is defined according to \eqref{eq:ODEi}. Multiplying both sides by $z$, integrating it over $\Omega$ and using Cauchy formula, together with Theorem \ref{thm:H-DP-SGD_q}, gives the first desired result, i.e.,
\begin{equation}\label{eq:thesisprob}
    \sup_{t\in[0,1)} \left| \mathcal{R}(\theta_{\lfloor tn\rfloor}) - R(t) \right| = O \left( \frac{\log^2 n}{\sqrt n}  \right),
\end{equation}
with overwhelming probability.

To prove the same inequality in expectation, define the event $\mathcal E$ such that 
\begin{equation}
    \Delta := \sup_{t\in[0,1)} \left| \mathcal{R}(\theta_{\lfloor tn\rfloor}) - R(t) \right| \geq C_1 \frac{\log^2 n}{\sqrt n}.
\end{equation}
By \eqref{eq:thesisprob}, there exists a constant $C_1$ such that $\mathbb P(\mathcal E) \leq 2 e^{- c_1 \log^2 d}$. Then
\begin{equation}\label{eq:expectationsplit}
    \sup_{t\in[0,1)} \left| \E \left[ \mathcal{R}(\theta_{\lfloor tn\rfloor}) \right] - R(t) \right| \leq \sup_{t\in[0,1)} \E \left[  \left| \mathcal{R}(\theta_{\lfloor tn\rfloor}) - R(t) \right|  \right] \leq  C_1 \frac{\log^2 n}{\sqrt n} + \E \left[ \Delta \mathbf{1}_{\mathcal E} \right].
\end{equation}
Notice that
\begin{equation}
    \Delta \leq \sup_{t\in[0,1)} \left| \mathcal{R}(\theta_{\lfloor tn\rfloor}) \right| + C_2 \leq C_3 \left( 1 + \sup_{t\in[0,1)} \norm{u_{\lfloor nt \rfloor}}_2^2 \right),
\end{equation}
since $R(t)$ exists in $t \in [0, 1]$ and does not depend on $d$ and $n$, and since $\opnorm{\Sigma} = O(1)$. Thus, Cauchy-Schwartz inequality yields
\begin{equation}
    \E \left[ \Delta \mathbf{1}_{\mathcal E} \right]^2 \leq \E \left[ \Delta^2 \right] \P(\mathcal E) \leq 2 C_4 e^{- c_1 \log^2 d} \left( 1 + \E \left[ \sup_{t\in[0,1)} \norm{u_{\lfloor nt \rfloor}}_2^4 \right] \right).
\end{equation}
By Lemma \ref{lem:tau_removal}, we have that for any $r \ge 0$, $Z = \sup_{t\in[0,1]} \norm{u_{\lfloor nt \rfloor}}_2^2 \le \|u_0\|^2 e^{C_4(1 + d^{-1/2}r)}$ with probability at least $1 - 2 e^{-r^2/2}$. This bound implies that $2 \log Z$ has a dimension free sub-Gaussian tail, which in turn implies that 
$\E[e^{2 \log Z}] = \E \left[ \sup_{t\in[0,1)} \norm{u_{\lfloor nt \rfloor}}_2^4 \right] = O(1)$, giving
\begin{equation}
    \E \left[ \Delta \mathbf{1}_{\mathcal E} \right]^2 \leq  C_5 e^{- c_1 \log^2 d} = O\left(\frac{1}{n}\right).
\end{equation}
Merging in \eqref{eq:expectationsplit} gives the desired result. \qed

\subsection{Proof of Proposition \ref{prop:laststeprisk}}

Due to the update rule in Algorithm \ref{alg:dp-sgd}, we have
\begin{equation}\label{eq:removeeta}
\begin{aligned}
    2 \mathcal R(\theta_n) &= \norm{\Sigma^{1/ 2} \left(\theta_n - \theta^* \right)}_2^2 \\
    &= \norm{\Sigma^{1/ 2} \left(\theta_{n-1} - \bar \eta_n \bar g_{n} + 2 \clip \sigma_n b_n - \theta^* \right)}_2^2 \\
    &= \norm{\Sigma^{1/ 2} \left(\theta_{n-1} - \theta^*  + 2 \clip \sigma_n b_n\right)}_2^2
    + \norm{ \Sigma^{1/ 2} \bar \eta_n \bar g_{n}}_2^2 \\
    &\qquad - 2 \bar \eta_n \bar g_{n}^\top \Sigma \left(\theta_{n-1}  - \theta^* \right) - 4 \bar \eta_n \bar g_{n}^\top \Sigma \clip \sigma_n b_n.
\end{aligned}
\end{equation}
By Assumptions \ref{ass:data} and \ref{ass:learning_rate_schedules}, and due to the definition of $\bar g_n$, we have
\begin{equation}
    \norm{ \Sigma^{1/ 2} \bar \eta_n \bar g_{n}}_2 \leq \opnorm{\Sigma}^{1/ 2} \left |\bar \eta_n \right| \norm{\bar g_{n}}_2 = O \left( \frac{\sqrt d}{n} \right) = O \left( \frac{1}{\sqrt d} \right).
\end{equation}
Theorem 3.1.1 in \cite{vershynin2018high} also guarantees that $\norm{b_n}_2 = O(\sqrt d + u)$ with probability at least $1 - 2 \exp \left( - c_1 u^2 \right)$, for some absolute constant $c_1>0$. Thus, with this probability, we have
\begin{equation}
    \left| 4 \bar \eta_n \bar g_{n}^\top \Sigma \clip \sigma_n b_n \right| = \left| 4 \bar \eta_n \bar g_{n}^\top \Sigma \clip \frac{\eta_n}{\rho} b_n \right| = O \left( \frac{1}{n} \, \sqrt d \, \sqrt d \, \frac{1}{n} (\sqrt d + u) \right) = O \left( \frac{1}{\sqrt d} + \frac{u}{d} \right),
\end{equation}
which gives
\begin{equation}
    \left| 2 \mathcal R(\theta_n) -  \norm{\Sigma^{1/ 2} \left(\theta_{n-1}  - \theta^*  + 2 \clip \sigma_n \Sigma^{1 / 2} b_n\right)}_2^2 \right| = O \left( \frac{1 + \sqrt{\mathcal R(\theta_{n- 1})} }{\sqrt d} + \frac{u}{d} \right).
\end{equation}

Similarly, we have 
\begin{equation}
\begin{aligned}
    \norm{\Sigma^{1/ 2} \left(\theta_{n-1} - \theta^*  + 2 \clip \sigma_n b_n\right)}_2^2 &= \norm{\Sigma^{1/ 2} \left(\theta_{n-1} - \theta^* \right)}_2^2 + \norm{2 \clip \sigma_n \Sigma^{1/ 2} b_n}_2^2 \\
    \qquad + 4 \clip \sigma_n b_n^\top \Sigma \left(\theta_{n-1}  - \theta^* \right).
\end{aligned}
\end{equation}
However, since $b_n$ is a standard Gaussian vector, we have that, with probability at least $1 - 2 \exp \left( -c_2 u^2\right)$,
\begin{equation}
\begin{aligned}
    \left| 4 \clip \sigma_n b_n^\top \Sigma \left(\theta_{n-1} - \theta^* \right) \right| &\leq 4 \clip \sigma_n \opnorm{\Sigma^{1/2}} \norm{\Sigma^{1/ 2} \left(\theta_{n-1} - \theta^* \right)}_2 u \\
    &= 4 c \sqrt d \frac{\tilde \eta(1)}{\rho n} \opnorm{\Sigma^{1/2}} \norm{\Sigma^{1/ 2} \left(\theta_{n-1} - \theta^* \right)}_2 u \\
    &= O \left( \frac{\sqrt d u}{n} \right) \norm{\Sigma^{1/ 2} \left(\theta_{n-1} - \theta^* \right)}_2 \\
    &= O \left(  \frac{\sqrt{\mathcal R(\theta_{n- 1})} u  }{\sqrt d} \right).
\end{aligned}
\end{equation}
Then, an application of the Hanson-Wright inequality (see Theorem 6.2.1 in \cite{vershynin2018high}) yields
\begin{equation}
    \left| \norm{2 \clip \sigma_n \Sigma^{1/2} b_n}_2^2 - 4 c^2 d \frac{\tilde \eta^2(1)}{\rho^2 n^2} \tr(\Sigma) \right| \leq 4 c^2 d \frac{\tilde \eta^2(1)}{\rho^2 n^2} \|\Sigma\| 
    u = O \left( \frac{u}{\sqrt d}\right),
\end{equation}
with probability at least $1 - 2 \exp \left( -c_3 \min( \sqrt d u, u^2 ) \right)$. 

Putting everything together gives
\begin{equation}\label{eq:udistrlaststep}
    \left| \mathcal R(\theta_n) - \mathcal R(\theta_{n- 1}) - 2 c^2 \tilde \eta^2(1) \gamma_n^2 / \rho^2 \right| = O \left(\frac{\left( 1 + \sqrt{\mathcal R(\theta_{n - 1})} \right) (u + 1)}{\sqrt d}  \right),
\end{equation}
with probability at least $1 - 2 \exp \left( -c_4 \min( \sqrt d u, u^2 ) \right)$.

Setting $u = \log d$, we have that
\begin{equation}
    \left| \mathcal R(\theta_n) - \mathcal R(\theta_{n- 1}) - 2 c^2 \tilde \eta^2(1) \gamma_n^2 / \rho^2 \right| = O \left(\frac{\log d}{\sqrt d} \right),
\end{equation}
with overwhelming probability, where we used $\tr(\Sigma) = d$ and Theorem \ref{thm:deteq}, which implies $\mathcal R(\theta_{n- 1}) = O(1)$ with overwhelming probability. Then, the first part of Theorem \ref{thm:deteq} and the triangle inequality concludes the first part of the proof.

For the result in expectation, notice that we have
\begin{equation}\label{eq:forlaststepexp}
\begin{aligned}
    \left| \E \left[  \mathcal R(\theta_n) \right] - R(1) - 2 c^2 \tilde \eta^2(1) \gamma_n^2 / \rho^2 \right| \leq \, & \E \left[  \left|  \mathcal R(\theta_n) - \mathcal R(\theta_{n-1}) - 2 c^2 \tilde \eta^2(1) \gamma_n^2 / \rho^2 \right| \right] \\
    & \qquad + \left| \E \left [\mathcal R(\theta_{n - 1}) \right] - R(1) \right|.
\end{aligned} 
\end{equation}
The second term on the RHS is $O(\log^2 n / \sqrt n)$ because of the second part of the statement of Theorem \ref{thm:deteq}, and due to continuity of $R(t)$ in $t = 1$. The argument of the first term of the RHS is the product of two random variables. The first one is such that
\begin{equation}\label{eq:forlaststepexp1}
    \E \left[ \left ( 1 + \sqrt{\mathcal R(\theta_{n-1})} \right)^2 \right] \leq 2 + 2 \E \left[ \mathcal R(\theta_{-1}) \right] = O(1),
\end{equation}
due to the second part of the statement of Theorem \ref{thm:deteq}. The second one, due to \eqref{eq:udistrlaststep}, is a random variable $Z$ such that $Z \leq (1 + u) / \sqrt d$ with probability at least $1 - 2 \exp \left( -c_4 \min( \sqrt d u, u^2 ) \right)$. Thus, $\sqrt d Z$ is sub-exponential with norm smaller than a constant, which implies
\begin{equation}\label{eq:forlaststepexp2}
    \E \left[ Z^2 \right] = O \left( \frac{1}{d} \right).
\end{equation}
Plugging \eqref{eq:forlaststepexp1} and \eqref{eq:forlaststepexp2} in \eqref{eq:forlaststepexp}, and using Cauchy-Schwartz inequality, gives the desired result. \qed

\section{Proofs for Section \ref{sec:optimalrates}}\label{app:ODEs}

In this section, all asymptotic notations are with respect to the limit $\gamma \to 0$.

\begin{proposition}\label{prop:conditionsandwhich}
Define $\overline R(t), \underline R(t): [0, 1] \to \R$ as the unique solutions of the following ODEs
\begin{equation}\label{eq:upperlowerODEsbody}
\begin{aligned}
    \diff \overline {R}(t) &= - 2 \lambda_{\min} \tilde{\eta}(t) \mu_c(\overline R) \overline{R} \diff t +  \lambda_{\max} \tilde{\eta}^2(t) \nu_c(\overline R) (\overline {R} + \zeta^2/2) \gamma \diff t + 2 c^2 \tilde \sigma^2(t) \gamma^2 \diff t, \\
    \diff \underline {R}(t) &= - 2 \lambda_{\max} \tilde{\eta}(t) \mu_c(\underline R) \underline{R} \diff t + \tilde{\eta}^2(t) \nu_c(\underline R) (\underline {R} + \zeta^2/2) \gamma \diff t + 2 c^2 \tilde \sigma^2(t) \gamma^2 \diff t,
\end{aligned}
\end{equation}
where $\overline R(0) = \underline R(0) = R(0) = \|\Sigma^{1 /2} \theta^*\|_2^2 /2$. Then, for every $t \in [0,1]$, we have
\begin{equation}
    \underline R(t) \leq R(t) \leq \overline R(t),
\end{equation}
where $R(t)$ is defined in Definition \ref{def:deteq} after taking $\gamma_n \to \gamma$.
\end{proposition}
\begin{proof}
Until the end of the proof, we will define more auxiliary ODEs, such that the RHS is uniformly Lipschitz in the dependent variable at all times $t \in [0 , 1]$. Then, by the extension of the Picard–Lindel\"{o}f theorem (see Corollary 2.6 in \cite{teschl2012}),
we have that their solutions exist and are unique, and therefore also have continuous derivatives.
Then, defining
\begin{equation}\label{eq:ODEbarsi}
\begin{aligned}
    \diff \overline D_i &= - 2 \lambda_{\min} \tilde{\eta}(t) \mu_c(R(t)) \overline D_i \diff t + \lambda_{i}\tilde{\eta}^2(t) \nu_c(R(t)) (R(t) + \zeta^2/2) \gamma \diff t + 2 c^2 \tilde \sigma^2(t) \gamma^2 \diff t, \\
    \diff \underline D_i &= - 2 \lambda_{\max} \tilde{\eta}(t) \mu_c(R(t)) \underline D_i \diff t + \lambda_{i} \tilde{\eta}^2(t) \nu_c(R(t)) (R(t) + \zeta^2/2) \gamma \diff t + 2 c^2 \tilde \sigma^2(t) \gamma^2 \diff t,
\end{aligned}
\end{equation}
by standard ODE comparison arguments (see Theorem 1.3 in \cite{teschl2012})
we have that 
$\overline D_i(t) \geq D_i(t) \geq \underline D_i(t)$ for all $t \in [0, 1]$.
Then, averaging over $i$ (weighting by $\lambda_i$) the equations in \eqref{eq:ODEbarsi} and \eqref{eq:ODEi}, and defining $\overline R'(t) = \frac{1}{d} \sum_{i=1}^d \lambda_i \overline D_i(t)$ and $\underline R'(t) = \frac{1}{d} \sum_{i=1}^d \lambda_i \underline D_i(t)$, we get $\overline R'(t) \geq R(t) \geq \underline R'(t)$ for all $t \in [0, 1]$, with
\begin{equation}\label{eq:ODEprimed}
\begin{aligned}
    \diff \overline R' &= - 2 \lambda_{\min} \tilde{\eta}(t) \mu_c(R(t)) \overline R' \diff t + \frac{\tr(\Sigma^2)}{d} \tilde{\eta}^2(t) \nu_c(R(t)) (R(t) + \zeta^2/2) \gamma \diff t + 2 c^2 \tilde \sigma^2(t) \gamma^2 \diff t, \\
    \diff \underline R' &= - 2 \lambda_{\max} \tilde{\eta}(t) \mu_c(R(t)) \underline R' \diff t + \frac{\tr(\Sigma^2)}{d} \, \tilde{\eta}^2(t) \nu_c(R(t)) (R(t) + \zeta^2/2) \gamma \diff t + 2 c^2 \tilde \sigma^2(t) \gamma^2 \diff t.
\end{aligned}
\end{equation}
Furthermore, for a fixed value of $c$, we have that the functions $\mu_c(R)$ and $\nu_c(R) (R + \zeta^2/2)$ are respectively monotonically decreasing and increasing with respect to $R$ (see Lemma \ref{lemma:munu}). Then, since by Jensen inequality we have that $\tr(\Sigma^2) \geq \tr(\Sigma)=d$ (where the last step holds due to Assumption \ref{ass:data}), and since we also have $\tr(\Sigma^2) \geq \lambda_{\max} \tr(\Sigma) = \lambda_{\max} d$,
by defining
\begin{equation}\label{eq:ODEnoprimed}
\begin{aligned}
    \diff \overline R &= - 2 \lambda_{\min} \tilde{\eta}(t) \mu_c(\overline R) \overline R \diff t + \lambda_{\max} \tilde{\eta}^2(t) \nu_c(\overline R) (\overline R + \zeta^2/2) \gamma \diff t + 2 c^2 \tilde \sigma^2(t) \gamma^2 \diff t, \\
    \diff \underline R &= - 2 \lambda_{\max} \tilde{\eta}(t) \mu_c(\underline R) \underline R \diff t + \tilde{\eta}^2(t) \nu_c(\underline R)(\underline R + \zeta^2/2) \gamma \diff t + 2 c^2 \tilde \sigma^2(t) \gamma^2 \diff t,
\end{aligned}
\end{equation}
again due to Theorem 1.3 in \cite{teschl2012} we get that $\overline R(t) \geq \overline R'(t) \geq R(t) \geq \underline R'(t) \geq \underline R(t)$ for all $t \in [0, 1]$, which gives the desired result.    
\end{proof}

\subsection{Proofs on output perturbation and constant private noise}\label{app:proofalpha012}

All the ODEs defined in this (and the following) section will be such that their RHS is uniformly Lipschitz in the dependent variable at all times $t \in [0 , 1]$, which in turn guarantees they have a unique solution. Furthermore, the RHSs will also be uniformly Lipschitz with respect to the variable $t$ due to Assumption \ref{ass:learning_rate_schedules}. Thus, if $R(0) = R'(0)$, and
\begin{equation}
    \diff R = f(t , R), \qquad \diff R' = f'(t, R'),
\end{equation}
with
\begin{equation}
    f'(t, R'(t)) \geq f(t, R'(t)),
\end{equation}
for all $t \in [0, 1]$, we have that
\begin{equation}\label{eq:chaplygin}
    R'(t) \geq R(t) \, \text{ for all } \, t \in [0, 1].
\end{equation}
The same statement holds also for the opposite inequality, and will be used extensively to bound the solutions of $R(t)$ for different schedules. This is a direct application of Theorem 1.3 in \cite{teschl2012}.

To ease the presentation, we introduce the notation $v = \tilde \eta(0)$. We will provide the proof separately for $\alpha = 0$ and $\alpha = 1/2$, and the proof for $\alpha \geq 1$ in the next section. We will also denote the test risk at initialization $R(0) = \norm{\Sigma^{1/2} \theta^*}_2^2 / 2$, and all asymptotic notations will be with respect to the limit $\gamma \to 0$.

The proof of Proposition \ref{prop:alpha012} is given separately for the case $\alpha = 0$ and $\alpha = 1/2$, in the general case where $\lambda_{\max}$ and $\lambda_{\min}$ are allowed to depend on $\gamma$. The results are stated separately in Propositions \ref{prop:onlyalpha0} and \ref{prop:onlyalpha12}, and Proposition \ref{prop:alpha012} is a direct consequence of them. At the end of this section we will discuss the sub-optimality of $c = \omega(1)$, to prove \eqref{eq:suboptimalitycbody}.

\begin{lemma}\label{lemma:integral}
Let $\beta < 1$ and $E_\beta(x) : \R^+ \to \R^+$ be the exponential integral function defined as
\begin{equation}\label{eq:expintegral}
    E_\beta(x) := \int_1^{ + \infty} \frac{e^{-xt}}{t^\beta} \, \diff t.
\end{equation}
Then, we have that
\begin{equation}\label{eq:s1}
    \lim_{x \to 0^+} x^{1 - \beta} E_\beta(x) = \Gamma(1 - \beta),
\end{equation}
where $\Gamma(\cdot)$ denotes the Euler Gamma function
\begin{equation}\label{eq:Gamma}
    \Gamma(s) := \int_0^{\infty} e^{-t} t^{s-1} \, \diff t.
\end{equation}
Furthermore, we have that, for all $x\ge 0$,
\begin{equation}\label{eq:ubEI}
    E_{\beta}(x)\le \Gamma(1-\beta)x^{-1+\beta}.
\end{equation}
Finally, if either $\beta \geq 0$ or $x \geq - 2 \beta$, we also have that 
\begin{equation}\label{eq:s2}
E_\beta(x) \leq \frac{2 e^{-x}}{x}.   
\end{equation}
\end{lemma}
\begin{proof}
The change of variable $u = xt$ in the definition of $E_\beta(x)$ yields
\begin{equation}\label{eq:exptogamma}
\begin{aligned}
    E_\beta(x) &= x^{\beta - 1} \int_x^{+ \infty} \frac{e^{-u}}{u^\beta} \, \diff u \\
    &= x^{\beta - 1} \left( \int_0^{ + \infty} e^{-u} u^{-\beta} \, \diff u - \int_0^{x} e^{-u} u^{-\beta} \, \diff u \right) \\
    &= x^{\beta - 1} \left( \Gamma(1 - \beta) - \int_0^x e^{-u} u^{-\beta} \, \diff u \right).
\end{aligned}
\end{equation}
Then, \eqref{eq:s1} and \eqref{eq:ubEI} readily follow from the fact that the last term in the equation above is bounded by
\begin{equation}
 0\le    \int_0^x e^{-u} u^{-\beta} \, \diff u \leq \int_0^x u^{-\beta} \, \diff u = \frac{x^{1 - \beta}}{1 - \beta}.
\end{equation}

For the upper bound, denoting with
\begin{equation}
    \Gamma(s, x) := \int_x^{\infty} e^{-t} t^{s-1} \, \diff t
\end{equation}
the upper incomplete Euler gamma function, \eqref{eq:exptogamma} allows us to write
\begin{equation}\label{eq:expincomplgamma}
    E_\beta(x) = \frac{1}{x^{1 - \beta}} \Gamma(1 - \beta, x).
\end{equation}
Via an integration by parts, we have
\begin{equation}\label{eq:gammaparts}
    \Gamma(1 - \beta, x) = \int_x^{\infty} e^{-t} t^{-\beta} \, \diff t = e^{-x} x^{-\beta} - \beta \int_x^{\infty} e^{-t} t^{-\beta - 1} \, \diff t.
\end{equation}
If $\beta \geq 0$, we have $\Gamma(1 - \beta, x) \leq e^{-x} x^{-\beta}$, which together with \eqref{eq:expincomplgamma} gives \eqref{eq:s2}. If $\beta < 0$, the second term in the equation above is positive, and since $t \geq x$, it can be upper bounded as
\begin{equation}
    - \beta \int_x^{\infty} e^{-t} t^{-\beta - 1} \, \diff t \leq - \frac{\beta}{x} \int_x^{\infty} e^{-t} t^{-\beta} \, \diff t = - \frac{\beta}{x} \, \Gamma(1 - \beta, x),
\end{equation}
which, if plugged in \eqref{eq:gammaparts}, for $x \geq - 2 \beta$ gives
\begin{equation}
    \Gamma(1 - \beta, x) \leq \frac{e^{-x} x^{-\beta}}{1 + \beta / x} \leq 2 e^{-x} x^{-\beta},
\end{equation}
and the thesis again follows from \eqref{eq:expincomplgamma}.
\end{proof}


\paragraph{$\alpha = 0$: output perturbation.} Recall that in the setting $\alpha = 0$, we have
\begin{equation}\label{eq:upperlowerODEsalpha0}
\begin{aligned}
    \diff \overline {R}(t) &= - 2 \lambda_{\min} v c \frac{\mu_c(\overline R)}{c} \overline{R} \diff t +  \lambda_{\max} (v c)^2 \frac{\nu_c(\overline R) (\overline {R} + \zeta^2/2)}{c^2} \gamma \diff t, \\
    \diff \underline {R}(t) &= - 2 \lambda_{\max} vc \frac{\mu_c(\underline R)}{c} \underline{R} \diff t + (vc)^2 \frac{\nu_c(\underline R) (\underline {R} + \zeta^2/2)}{c^2} \gamma \diff t.
\end{aligned}
\end{equation}
Importantly, recall that the risk $\mathcal R(\theta^p)$ in this setting is not well approximated by $R(1)$, due to Proposition \ref{prop:laststeprisk}.

\begin{proposition}\label{prop:onlyalpha0}
    Let Assumptions \ref{ass:data} and \ref{ass:learning_rate_schedules} hold, and let $\theta^p$ be the solution obtained with Algorithm \ref{alg:dp-sgd}, with the schedule defined in \eqref{eq:polynomialsched} for $\alpha = 0$, in the setting $\gamma = d / n \to 0$. Furthermore, assume
    \begin{equation}\label{eq:conditionsfirstpropalpha0}
        \frac{\lambda_{\max}}{\lambda_{\min}^2} \ln (1 / \gamma) \gamma = o(1).
    \end{equation}
    Then, by setting 
    \begin{equation}
        c = O(1), \qquad vc = \frac{C \ln(1/\gamma)}{\lambda_{\min}}, \qquad v \le 2 / \gamma, 
    \end{equation}
    for a large enough constant $C$ which does not depend on $\gamma, \rho, \Sigma$, we have that, with overwhelming probability,
    \begin{equation}
        \mathcal R(\theta^p) = O \left( \frac{\lambda_{\max} \gamma \ln(1 / \gamma)}{\lambda_{\min}^2} + \frac{\gamma^2 \ln^2(1 / \gamma)}{\rho^2 \lambda_{\min}^2} \right). 
    \end{equation}
    Furthermore, suppose there exists $h >0$ such that $\rho = \Omega \left( \gamma^{1 - h} \right)$. Then, for any choice of the hyper-parameters $c$ and $v$ such that $v \le 2 / \gamma$, we have that
    \begin{equation}\label{eq:upperboundalg1}
         \mathcal R(\theta^p) = \Omega \left( \frac{\gamma \ln(1 / \gamma)}{\lambda_{\max}^2} + \frac{\gamma^2 \ln^2(1 / \gamma)}{\rho^2 \lambda_{\max}^2} \right).
    \end{equation}
\end{proposition}

\begin{proof}
Let us introduce the shorthands
\begin{equation}
\begin{aligned}
    \overline f(t, \overline R) &= - 2 \lambda_{\min} v c \frac{\mu_c(\overline R)}{c} \overline{R} \diff t +  \lambda_{\max} (v c)^2 \frac{\nu_c(\overline R) (\overline {R} + \zeta^2/2)}{c^2} \gamma \diff t,  \\
    \underline f(t, \underline R) &= - 2 \lambda_{\max} vc \frac{\mu_c(\underline R)}{c} \underline{R} \diff t + (vc)^2 \frac{\nu_c(\underline R) (\underline {R} + \zeta^2/2)}{c^2} \gamma \diff t,
\end{aligned}
\end{equation}
corresponding to the RHSs of the ODEs of interest. Then, consider the auxiliary ODEs
\begin{equation}\label{eq:auxiliaries}
    \diff \overline R'  = - \overline a \overline R' \diff t + \overline b \diff t =: \overline f'(t, \overline R') \diff t, \qquad \diff \underline R'  = - \underline a \,\underline R' \diff t + \underline b \diff t=: \underline f'(t, \underline R')\diff t,
    \qquad \overline a, \overline b, \underline a, \underline b > 0,
\end{equation}
with initial conditions $\overline R'(0) = \overline R(0) = R(0) = \underline R(0) = \underline R'(0)$.

Notice that $\overline R'(t)$ admits the closed form solution
\begin{equation}\label{eq:solaux}
    \overline R'(t) = \left( R(0) - \overline b / \overline a \right) e^{-\overline at} + \overline b / \overline a. 
\end{equation}
Similarly, we have that
\begin{equation}\label{eq:solauxlb}
    \underline R'(t) = \left( R(0) - \underline b / \underline a \right) e^{-\underline at} + \underline b / \underline a. 
\end{equation}

Since $c = O(1)$, by Lemma \ref{lemma:munubounds}, we have that 
\begin{equation}\label{eq:boundsmunu}
    \frac{\underline c_\mu(c, \zeta)}{\sqrt { 2 R + \zeta^2}} < \frac{\mu_c(R)}{c} < \frac{1}{\sqrt {\pi \left( R + \zeta^2 /2 \right)}}, \qquad \underline c_\nu(c, \zeta) < \frac{2 \nu_c(R) \left( R + \zeta^2 / 2 \right)}{c^2} < 1.
\end{equation}
Then, let us set
\begin{equation}\label{eq:ab1}
    \overline a = 2 \lambda_{\min} vc \frac{\underline c_\mu(c, \zeta)}{\sqrt{R(0) + 1 + \zeta^2 }}, \quad \overline b = \lambda_{\max} \frac{v^2 c^2 \gamma}{2}, \quad \underline a= 2 vc \frac{\lambda_{\max}}{ \sqrt{\pi \zeta^2 / 2}}, \quad \underline b=\frac{v^2 c^2 \gamma\underline c_\nu(c, \zeta) }{2}.
\end{equation}
As $cv = C \ln(1/ \gamma) / \lambda_{\min}$, the choice in \eqref{eq:ab1} ensures that $\underline b / \underline a \leq 1$ and $\overline b / \overline a \leq 1$ as long as $\lambda_{\max} / \lambda_{\min}^2 \ln (1 / \gamma) \gamma = o(1)$, which in turn guarantees that 
\begin{equation}
    \overline R'(t) \in [0, R(0) + 1], \qquad \underline R'(t) \in [0, R(0) + 1].
\end{equation}
Thus, \eqref{eq:ab1} guarantees
\begin{equation}\label{ineq:fg}
    \overline f(t, \overline R'(t)) < \overline f'(t, \overline R'(t)), \qquad \underline f(t, \underline R'(t)) > \underline f'(t, \underline R'(t)), \qquad \text{ for all } \, t \in [0, 1].
\end{equation}
Thus, by \eqref{eq:chaplygin}, we have
\begin{equation}\label{eq:chap1}
  \underline R'(t) < \underline R(t), \qquad  \overline R'(t) > \overline R(t) \qquad \text{ for all } \, t \in (0,1].
\end{equation}

In particular, plugging $vc = C \ln(1 / \gamma) / \lambda_{\min}$ and \eqref{eq:ab1} in \eqref{eq:solaux}, we have that 
\begin{equation}
    \overline R(1) < \overline R'(1)<e^{-\overline a} + \frac{\overline b}{\overline a} = O \left(\frac{\lambda_{\max}}{\lambda_{\min}^2} \ln (1 / \gamma) \gamma\right),
\end{equation}
as long as $C$ is chosen to be large enough. Then, the upper bounds follows from Theorem \ref{thm:deteq} and Propositions \ref{prop:laststeprisk} and \ref{prop:conditionsandwhich},
as the risk at the last iterate roughly increases by $2 c^2 v^2 \gamma^2 / \rho^2 = O (\gamma^2 \ln^2(1/\gamma) / (\rho^2 \lambda_{\min}^2))$.

For the lower bound, let us first suppose $c \leq 1$, which implies that \eqref{eq:boundsmunu} holds with $\underline c_\mu(c, \zeta)$ and $\underline c_\nu(c, \zeta)$ lower bounded by constants independent of $\gamma$. Pick $\underline a, \underline b$ as in \eqref{eq:ab1}. 

First, let us consider the case $\underline a \leq 1$. We have that \eqref{eq:chap1} gives
\begin{equation}\label{eq:lbR1asmall}
    \underline R(1) > \underline R'(1)= R(0) e^{-\underline a}+\frac{\underline b}{\underline a}-\frac{\underline b}{\underline a}e^{-\underline a}> R(0) e^{-\underline a} \geq R(0) / e = \Omega(1).
\end{equation}
In the case $\underline a > 1$, \eqref{eq:chap1} gives
\begin{equation}\label{eq:lbR1}
    \underline R(1) > \underline R'(1)= R(0) e^{-\underline a}+\frac{\underline b}{\underline a}-\frac{\underline b}{\underline a}e^{-\underline a}> R(0) e^{-\underline a}+\frac{\underline b}{\underline a}(1 - e^{-1}) = R(0) \left(e^{-v c \lambda_{\max} a} + \frac{v c b \gamma}{\lambda_{\max} a}\right),
\end{equation}
where $ a, b$ are positive constants which do not depend on $\gamma, v, c$ and the spectrum of $\Sigma$. 
Then, let us consider the following quantity
\begin{equation}\label{eq:lower01}
    \underline{\mathcal R} := R(0) \left(e^{-v c \lambda_{\max} a} + \frac{v c b \gamma}{\lambda_{\max} a}\right) + 2 v^2 c^2 \frac{\gamma^2}{\rho^2}.
\end{equation}
We have that
\begin{equation}
    \arg \min_{vc} e^{-v c \lambda_{\max} a} + \frac{v c b \gamma}{\lambda_{\max} a} = \frac{1}{\lambda_{\max} a} \ln \left( \frac{(\lambda_{\max} a)^2}{b \gamma} \right),
\end{equation}
which implies
\begin{equation}\label{eq:lower02}
    \underline{\mathcal R} > R(0) \left(e^{-v c \lambda_{\max} a} + \frac{v c b \gamma}{\lambda_{\max} a}\right) = \Omega\left( \frac{\gamma \ln ( 1/ \gamma)}{\lambda_{\max}^2}\right).
\end{equation}

Note that, assuming $\rho = \Omega(\gamma^{1 - h})$, we have
\begin{equation}\label{eq:lower03}
\begin{split}    
    \min_{vc \ge \frac{\ln(\rho^2/\gamma^2)}{2a \lambda_{\max}}}& e^{-vc \lambda_{\max} a}+v^2 c^2 \frac{\gamma^2}{\rho^2} \ge \min_{vc \ge \frac{\ln(\rho^2/\gamma^2)}{2a \lambda_{\max}}} v^2 c^2 \frac{\gamma^2}{\rho^2} =  \Omega\left(\frac{\gamma^2 \ln^2 \left( 1 / \gamma \right)}{\rho^2 \lambda^2_{\max}}\right),\\
    \min_{0 \le vc \le \frac{\ln(\rho^2/\gamma^2)}{2a \lambda_{\max}}}& e^{-vc \lambda_{\max} a}+v^2 c^2 \frac{\gamma^2}{\rho^2} \ge\ \frac{\gamma}{\rho} = \Omega\left(\frac{\gamma^2 \ln^2 \left( 1 / \gamma \right)}{\rho^2}\right) = \Omega\left(\frac{\gamma^2 \ln^2 \left( 1 / \gamma \right)}{\rho^2 \lambda^2_{\max}}\right).
\end{split}
\end{equation}

Then, merging \eqref{eq:lower01}, \eqref{eq:lower02} and \eqref{eq:lower03} yields
\begin{equation}\label{eq:lowerbound0part1}
    \underline{\mathcal R} = \Omega \left(\frac{\gamma \ln ( 1/ \gamma)}{\lambda_{\max}^2} + \frac{\gamma^2 \ln^2 \left( 1 / \gamma \right)}{\rho^2 \lambda^2_{\max}}\right).
\end{equation}
The result in \eqref{eq:lbR1} together with Theorem \ref{thm:deteq} and Propositions \ref{prop:laststeprisk} and \ref{prop:conditionsandwhich} imply that, with overwhelming probability,
\begin{equation}
    \mathcal R(\theta^p) = \Omega \left(\frac{\gamma \ln ( 1/ \gamma)}{\lambda_{\max}^2} + \frac{\gamma^2 \ln^2 \left( 1 / \gamma \right)}{\rho^2 \lambda^2_{\max}}\right).
\end{equation}

It remains to prove the lower bound for $c > 1$. From \eqref{eq:mucdef}, one readily has that
    $0<\mu_c(R) < 1$. Note that
\begin{equation}\label{eq:munubigc}
    \nu_c(R) (R + \zeta^2 / 2) \geq \max \left( \frac{\nu_c(R) (R + \zeta^2 / 2)}{c^2}, \nu_c(R) \zeta^2 / 2 \right),
\end{equation}
which combined with \eqref{eq:bdsnu} gives that $\nu_c(R) (R + \zeta^2 / 2) > b_1$, for a positive constant $b_1$ independent of $\gamma, v, c$ and the spectrum of $\Sigma$. Furthermore, \eqref{eq:bdcnu} implies that  $\nu_c(R) (R + \zeta^2 / 2)<c^2/2$. Thus, the solution of $\underline R$ is lower bounded by that of the ODE below:
\begin{equation}\label{eq:forremark1}
    \diff \underline R'' = - 2 \lambda_{\max} v \underline R'' \diff t + v^2 b_1 \gamma \diff t,
\end{equation}
with initial condition $\underline R''(0)=R(0)$.
Thus, following the same steps we used to achieve \eqref{eq:lowerbound0part1}, it can be shown that, as $c \geq 1$, we have, with overwhelming probability,
\begin{equation}
        \mathcal R(\theta^p) = \Omega \left(\frac{\gamma \ln ( 1/ \gamma)}{\lambda_{\max}^2} + \frac{\gamma^2 \ln^2 \left( 1 / \gamma \right)}{\rho^2 \lambda^2_{\max}}\right),
\end{equation}
which gives the desired result.
\end{proof}

\paragraph{$\alpha = 1/2$: constant private noise.}
In the setting $\alpha = 1/2$, \eqref{eq:polynomialsched} gives
\begin{equation}\label{eq:upperlowerODEsalpha12}
\begin{aligned}
    \diff \overline {R}(t) &= - 2 \lambda_{\min} v c \sqrt{1 - t} \frac{\mu_c(\overline R)}{c} \overline{R} \diff t +  \lambda_{\max} (v c)^2 (1 - t) \frac{\nu_c(\overline R) (\overline {R} + \zeta^2/2)}{c^2} \gamma \diff t + 2 (vc)^2 \frac{\gamma^2}{\rho^2} \diff t, \\
    \diff \underline {R}(t) &= - 2 \lambda_{\max} vc \sqrt{1 - t} \frac{\mu_c(\underline R)}{c} \underline{R} \diff t + (vc)^2 (1 - t) \frac{\nu_c(\underline R) (\underline {R} + \zeta^2/2)}{c^2} \gamma \diff t + 2 (vc)^2 \frac{\gamma^2}{\rho^2} \diff t.
\end{aligned}
\end{equation}
Importantly, recall that the risk $\mathcal R(\theta^p)$ in this setting is well approximated by $R(1)$, due to Proposition \ref{prop:laststeprisk}, since $\tilde \eta(1) = 0$.

\begin{proposition}\label{prop:onlyalpha12}
    Let Assumptions \ref{ass:data} and \ref{ass:learning_rate_schedules} hold, and let $\theta^p$ be the solution obtained with Algorithm \ref{alg:dp-sgd}, with the schedule defined in \eqref{eq:polynomialsched} for $\alpha = 1/2$, in the setting $\gamma = d / n \to 0$. Furthermore, assume
    \begin{equation}\label{eq:condapp}
        \frac{\ln^2(1 / \gamma) \gamma}{\lambda_{\min}^2} \left( \lambda_{\max} + \frac{\gamma}{\rho^2}\right) = o(1).
    \end{equation}
    Then, by setting 
    \begin{equation}
        c = O(1), \qquad vc = \frac{C \ln(1/\gamma)}{\lambda_{\min}}, \qquad v \le 2 / \gamma, 
    \end{equation}
    for a large enough constant $C$ which does not depend on $\gamma, \rho, \Sigma$, we have that, with overwhelming probability,
    \begin{equation}
        \mathcal R(\theta^p) = O \left( \frac{\lambda_{\max} \gamma \ln^{2/3}(1 / \gamma)}{\lambda_{\min}^2} + \frac{\gamma^2 \ln^{4 / 3}(1 / \gamma)}{\rho^2 \lambda_{\min}^2} \right). 
    \end{equation}
    Furthermore, suppose there exists $h >0$ such that $\rho = \Omega \left( \gamma^{1 - h} \right)$. Then, for any choice of the hyper-parameters $c$ and $v$ such that $v \le 2 / \gamma$, we have that
    \begin{equation}
         \mathcal R(\theta^p) = \Omega \left( \frac{\gamma \ln^{2 / 3}(1 / \gamma)}{\lambda_{\max}^2} + \frac{\gamma^2 \ln^{4/3}(1 / \gamma)}{\rho^2 \lambda_{\max}^2} \right).
    \end{equation}
\end{proposition}

\begin{proof}
As before, let us introduce the shorthands
\begin{equation}
\begin{aligned}
    \overline f(t, \overline R) &= - 2 \lambda_{\min} v c \sqrt{1 - t} \frac{\mu_c(\overline R)}{c} \overline{R} +  \lambda_{\max} (v c)^2 (1 - t) \frac{\nu_c(\overline R) (\overline {R} + \zeta^2/2)}{c^2} \gamma + 2 (vc)^2 \frac{\gamma^2}{\rho^2}, \\
    \underline f(t, \underline R) &= - 2 \lambda_{\max} vc \sqrt{1 - t} \frac{\mu_c(\underline R)}{c} \underline{R} + (vc)^2 (1 - t) \frac{\nu_c(\underline R) (\underline {R} + \zeta^2/2)}{c^2} \gamma  + 2 (vc)^2 \frac{\gamma^2}{\rho^2}.
\end{aligned}
\end{equation}
Then, consider the auxiliary ODEs
\begin{equation}\label{eq:auxiliaries2}
\begin{split}    
    \diff \overline R'  &=  - \overline a \sqrt{1-t} \overline R' \diff t + \overline b_1 (1-t) \diff t + \overline b_2 \diff t=: \overline f'(t, \overline R)\diff t,\qquad\overline a, \overline b_1,\overline b_2>0, \\
    \diff \underline R'  &=  - \underline a \sqrt{1-t}\underline R' \diff t + \underline b_1 (1-t) \diff t+\underline b_2\diff t=: \underline f'(t, \underline R)\diff t,\qquad\underline a, \underline b_1,\underline b_2>0, 
\end{split}
\end{equation}
with initial conditions $\overline R'(0) = \overline R(0) = R(0) = \underline R(0) = \underline R'(0)$, which admit the closed-form solutions
\begin{equation}\label{eq:closedformODEalpha12}
\begin{split}    
    \overline R' (t)=\frac{R(0)}{3} & e^{-2\overline a(1-(1-t)^{3/2})/3}\biggl(3 \, +\\
    &- \frac{2}{R(0)} \overline b_1e^{2\overline a/3}\left(E_{-1/3}\left(\frac{2\overline a}{3}\right)-(1-t)^2 E_{-1/3}\left(\frac{2\overline a(1-t)^{3/2}}{3}\right)\right)\\
    &- \frac{2}{R(0)} \overline b_2e^{2\overline a/3}\left(E_{1/3}\left(\frac{2\overline a}{3}\right)-(1-t)E_{1/3}\left(\frac{2\overline a(1-t)^{3/2}}{3}\right)\right)\biggr),
    \end{split}
\end{equation}
\begin{equation}
\begin{split}    
    \underline R'(t)=\frac{R(0)}{3}& e^{-2\underline a(1-(1-t)^{3/2})/3}\biggl(3  \, + \\
    & - \frac{2}{R(0)} \underline b_1e^{2\underline a/3}\left(E_{-1/3}\left(\frac{2\underline a}{3}\right)-(1-t)^2 E_{-1/3}\left(\frac{2\underline a(1-t)^{3/2}}{3}\right)\right)\\
    &- \frac{2}{R(0)} \underline b_2e^{2\underline a/3}\left(E_{1/3}\left(\frac{2\underline a}{3}\right)-(1-t)E_{1/3}\left(\frac{2\underline a(1-t)^{3/2}}{3}\right)\right)\biggr),
    \end{split}
\end{equation}
expressed in terms of the exponential integral functions $E_{-1/3}(\cdot)$, $E_{1/3}(\cdot)$ defined in \eqref{eq:expintegral}.

Note that, for $t\in [0, 1]$, $\overline R'(t)\ge 0$. In fact, if this is not the case, by continuity of $\overline R'$, there exists an interval $(t^*, t^*+\delta)\subseteq [0, 1]$ s.t.\  $\overline R'(t^*)=0$ and $\overline R'(t)<0$ for all $t\in (t^*, t^*+\delta)$. However, if $\overline R'(t^*)=0$, then the derivative of $\overline R'$ evaluated at $t^*$ is $\ge b_2>0$ by \eqref{eq:auxiliaries2}, which is a contradiction. A similar argument gives that, for $t\in [0, 1]$, $\underline R'(t)\ge 0$.

Next, we upper bound $\overline R'(t)$ as
\begin{equation}
\begin{split}
    \overline R'(t)\le R(0)&+\frac{2}{3} \overline b_1e^{2\overline a(1-t)^{3/2}/3}(1-t)^2 E_{-1/3}\left(\frac{2\overline a(1-t)^{3/2}}{3}\right)\\
    &+\frac{2}{3} \overline b_2e^{2\overline a(1-t)^{3/2}/3}(1-t) E_{1/3}\left(\frac{2\overline a(1-t)^{3/2}}{3}\right).
\end{split}
\end{equation}
If $2\overline a(1-t)^{3/2}/3\ge 2/3$, an application of \eqref{eq:s2} gives that
\begin{equation}
\begin{split}    
    e^{2\overline a(1-t)^{3/2}/3}(1-t)^2 & E_{-1/3}\left(\frac{2\overline a(1-t)^{3/2}}{3}\right)\le \frac{3}{\overline a}\sqrt{1-t}\le \frac{3}{\overline a},\\
    e^{2\overline a(1-t)^{3/2}/3}(1-t) & E_{1/3}\left(\frac{2\overline a(1-t)^{3/2}}{3}\right)\le \frac{3}{\overline a}(1-t)^{-1/2}\le \frac{3}{\overline a^{2/3}}.
    \end{split}
\end{equation}
Instead, if $2\overline a(1-t)^{3/2}/3< 2/3$, by using \eqref{eq:ubEI} we have
\begin{equation}
\begin{split}    
    e^{2\overline a(1-t)^{3/2}/3}(1-t)^2& E_{-1/3}\left(\frac{2\overline a(1-t)^{3/2}}{3}\right)\le e^{2/3}\Gamma\left(\frac{4}{3}\right)\left(\frac{2\overline a}{3}\right)^{-4/3},\\
    e^{2\overline a(1-t)^{3/2}/3}(1-t)& E_{1/3}\left(\frac{2\overline a(1-t)^{3/2}}{3}\right)\le e^{2/3}\Gamma\left(\frac{2}{3}\right)\left(\frac{2\overline a}{3}\right)^{-2/3}.
\end{split}
\end{equation}
Thus, the following upper bound holds for $t\in [0, 1]$:
\begin{equation}
\begin{split}
    \overline R'(t)\le  R(0) &+ \frac{2\overline b_1}{\overline a}+\frac{2\overline b_2}{\overline a^{2/3}}+e^{2/3}\left(\frac{2}{3}\right)^{-1/3}\Gamma\left(\frac{4}{3}\right)\frac{\overline b_1}{\overline a^{4/3}}+e^{2/3}\left(\frac{2}{3}\right)^{1/3}\Gamma\left(\frac{2}{3}\right)\frac{\overline b_2}{\overline a^{2/3}}.
\end{split}
\end{equation}
Following the same passages, we have that, for $t\in [0, 1]$,
\begin{equation}
\begin{split}
    \underline R'(t)\le R(0) &+ \frac{2\underline b_1}{\underline a}+\frac{2\underline b_2}{\underline a^{2/3}}+e^{2/3}\left(\frac{2}{3}\right)^{-1/3}\Gamma\left(\frac{4}{3}\right)\frac{\underline b_1}{\underline a^{4/3}}+e^{2/3}\left(\frac{2}{3}\right)^{1/3}\Gamma\left(\frac{2}{3}\right)\frac{\underline b_2}{\underline a^{2/3}}.
\end{split}
\end{equation}
Let us now pick
\begin{equation}\label{eq:setabc}
\begin{split}
    \overline a &= 2 vc \lambda_{\min} \frac{\underline c_\mu(c, \zeta)}{\sqrt{R(0) + 1 + \zeta^2 / 2}}, \qquad \overline b_1 = \lambda_{\max}  \frac{v^2 c^2 \gamma}{2}, \qquad \overline b_2 = 4 v^2 c^2\frac{\gamma^2}{\rho^2},\\
    \underline a &= 2 vc \lambda_{\max} \frac{1}{ \sqrt{\pi \zeta^2 / 2}}, \qquad \underline b_1 = \frac{v^2 c^2 \gamma\underline c_\nu(c, \zeta) }{2}, \qquad \underline b_2 = v^2 c^2\frac{\gamma^2}{\rho^2}.    
\end{split}
\end{equation}
Since $c = O(1)$, by Lemma \ref{lemma:munubounds}, we have that \eqref{eq:boundsmunu} holds. 
As $cv = C \ln(1/\gamma) / \lambda_{\min}$, the choice in \eqref{eq:setabc} ensures that $0\le \underline R'(t)\le \overline R'(t)\le 1 + R(0)$ for $t\in [0, 1]$, as long as we have $\overline b_1 + \overline b_2 = o(1)$, which holds due to \eqref{eq:condapp}.

Thus, \eqref{eq:setabc} guarantees
\begin{equation}\label{ineq:fg2}
    \overline f(t, \overline R'(t)) < \overline f'(t, \overline R'(t)), \qquad \underline f(t, \underline R'(t)) > \underline f'(t, \underline R'(t)),\qquad \text{ for all } \, t \in [0, 1].
\end{equation}
Note that $\overline f(t, R)$ and $\underline f(t ,R)$ are continuous in both variables in the intervals $t \in [0, 1]$ and $R \in [0, 1]$. Furthermore, they are also Lipschitz in $R$ in these same intervals. Thus, \eqref{eq:chaplygin} gives
\begin{equation}\label{eq:sand}
  \underline R'(t) < \underline R(t), \qquad \overline R(t) < \overline R'(t) ,\qquad \text{ for all } \, t \in (0,1].
\end{equation}

To prove the upper bound, due to Theorem \ref{thm:deteq} and Propositions \ref{prop:laststeprisk} and \ref{prop:conditionsandwhich}, it suffices to give an upper bound on $\overline R'(1)$. To this aim, note that Lemma \ref{lemma:integral} yields
\begin{equation}
    \begin{split}
\lim_{x\to 0}& \,x^2 \, E_{-1/3}\left(\frac{2ax^{3/2}}{3}\right)=\left(\frac{3}{2}\right)^{4/3}\Gamma\left(\frac{4}{3}\right)a^{-4/3},\\
\lim_{x\to 0}& \, x \, E_{1/3}\left(\frac{2ax^{3/2}}{3}\right)=\left(\frac{3}{2}\right)^{2/3}\Gamma\left(\frac{2}{3}\right)a^{-2/3},
    \end{split}
\end{equation}
where $\Gamma(\cdot)$ denotes the Euler Gamma function (defined in \eqref{eq:Gamma}). Thus,
\begin{equation}\label{eq:solaux2}
\begin{split}
    \overline R'(1)&=\frac{R(0)}{3}e^{-2\overline a/3}\biggl(3-\frac{2}{R(0)}\overline b_1e^{2\overline a/3}\left(E_{-1/3}\left(\frac{2\overline a}{3}\right)-\left(\frac{3}{2}\right)^{4/3}\Gamma\left(\frac{4}{3}\right)\overline a^{-4/3}\right)\\
    & \hspace{6em}- \frac{2}{R(0)} \overline b_2e^{2\overline a/3}\left(E_{1/3}\left(\frac{2\overline a}{3}\right)-\left(\frac{3}{2}\right)^{2/3}\Gamma\left(\frac{2}{3}\right)\overline a^{-2/3}\right)
    \biggr)\\
    &\le R(0) e^{-2\overline a/3}+\left(\frac{3}{2}\right)^{1/3}\Gamma\left(\frac{4}{3}\right)\overline b_1\overline a^{-4/3}+\left(\frac{3}{2}\right)^{-1/3}\Gamma\left(\frac{2}{3}\right)\overline b_2\overline a^{-2/3},
\end{split}
\end{equation}
where in the last line we have used the non-negativity of the exponential integral functions.
For $C$ sufficiently large, due to \eqref{eq:setabc}, the first term in the RHS is $o(\gamma)$ (recall that $vc = C \ln(1/ \gamma) / \lambda_{\min}$). The other two terms are $O( \lambda_{\max} \gamma \ln^{2 / 3} (1 / \gamma) / \lambda_{\min}^2)$ and $O( \gamma^2 \ln^{4 / 3} (1 / \gamma) / (\rho^2 \lambda_{\min}^2)$ respectively. Note that 
$\overline R(1) < \overline R'(1)$ and $\tilde \eta(1) = 0$. Thus, the desired result follows from Theorem \ref{thm:deteq} and Propositions \ref{prop:laststeprisk} and \ref{prop:conditionsandwhich}. 

To prove the lower bound, due to Theorem \ref{thm:deteq} and Propositions \ref{prop:laststeprisk} and \ref{prop:conditionsandwhich},
it suffices to show that the inequality in the thesis holds for $\underline R(1)$. Let us first suppose $c\le 1$, which implies that \eqref{eq:boundsmunu} holds. Pick $\underline a, \underline b_1, \underline b_2$ as in \eqref{eq:setabc}, and consider the ODE $\underline R'(t)$ defined in \eqref{eq:auxiliaries2} with the initial condition $\underline R'(0)=R(0)$, which is a lower bound on $\underline R(t)$ due to \eqref{eq:sand}, i.e.,
\begin{equation}
\begin{split}
    \underline R(1) > \underline R'(1)&=\frac{R(0)}{3}e^{-2\underline a/3}\biggl(3-\frac{2}{R(0)}\underline b_1e^{2\underline a/3}\left(E_{-1/3}\left(\frac{2\underline a}{3}\right)-\left(\frac{3}{2}\right)^{4/3}\Gamma\left(\frac{4}{3}\right)\underline a^{-4/3}\right)\\
    & \hspace{6em}- \frac{2}{R(0)} \underline b_2e^{2\underline a/3}\left(E_{1/3}\left(\frac{2\underline a}{3}\right)-\left(\frac{3}{2}\right)^{2/3}\Gamma\left(\frac{2}{3}\right)\underline a^{-2/3}\right)
    \biggr).
\end{split}
\end{equation}
We consider two additional cases depending on the value of $\underline a$. If $\underline a\ge 2$, then \eqref{eq:s2} gives that
\begin{equation}
    E_{-1/3}\left(\frac{2\underline a}{3}\right)\le \frac{2e^{-2\underline a/3}}{\frac{2\underline a}{3}},\qquad E_{1/3}\left(\frac{2\underline a}{3}\right)\le \frac{2e^{-2\underline a/3}}{\frac{2\underline a}{3}},
\end{equation}
which implies that 
\begin{equation}
\begin{split}    
    \underline R(1)\ge R(0) e^{-2\underline a/3}&+\left(\frac{3}{2}\right)^{1/3}\Gamma\left(\frac{4}{3}\right)\underline b_1\underline a^{-4/3}-2e^{-2\underline a/3}\underline b_1\underline a^{-1}\\
    &+\left(\frac{3}{2}\right)^{-1/3}\Gamma\left(\frac{2}{3}\right)\underline b_2\underline a^{-2/3}-2e^{-2\underline a/3}\underline b_2\underline a^{-1}.
\end{split}
\end{equation}
Note that 
\begin{equation}
\begin{split}    
 &   \left(\frac{3}{2}\right)^{1/3}\Gamma\left(\frac{4}{3}\right)\underline b_1\underline a^{-4/3}-2e^{-2\underline a/3}\underline b_1\underline a^{-1}\ge \underline b_1\underline a^{-4/3}-2e^{-2\underline a/3}\underline b_1\underline a^{-1}\ge \frac{1}{3} \underline b_1\underline a^{-4/3},\\
  &   \left(\frac{3}{2}\right)^{-1/3}\Gamma\left(\frac{2}{3}\right)\underline b_2\underline a^{-2/3}-2e^{-2\underline a/3}\underline b_2\underline a^{-1}\ge \underline b_2\underline a^{-2/3}-2e^{-2\underline a/3}\underline b_2\underline a^{-1}\ge \frac{1}{3} \underline b_2\underline a^{-2/3},\\
\end{split}
\end{equation}
where the inequalities on the right hold for $\underline a\ge 2$. Thus,
\begin{equation}\label{eq:lbR1int}
\begin{aligned}
    \underline R(1) \ge & \left( \frac{R(0)}{2} e^{-avc \lambda_{\max}} + \frac{b_1}{3} \lambda_{\max} ^{- 4 / 3} a^{-4/3}(vc)^{2/3}\gamma\right) \\
    & \qquad + \left( \frac{R(0)}{2} e^{-avc \lambda_{\max} }+ \frac{b_2}{3} \lambda_{\max} ^{- 2 / 3} a^{-2/3}(vc)^{4/3}\frac{\gamma^2}{\rho^2}\right),
\end{aligned}
\end{equation}
where $a, b_1, b_2$ are positive constants which do not depend on $\gamma, v, c, \rho$ or the spectrum of $\Sigma$.

Denoting with $a' = a \lambda_{\max}$, $\tau = vc$, we have that for a fixed $\beta \in \{2/3, 4/3 \}$, and for any $\delta = o(1)$,
\begin{equation}
    \min_{\tau \ge 0} e^{-\tau a'}+\frac{\tau^{2-\beta}\delta}{(a')^\beta}=\Omega \left(\frac{\delta \ln^{2-\beta}(1/\delta)}{(a')^2} \right),
\end{equation}
which follows from the following calculations:
\begin{equation}
\begin{split}    
    \min_{\tau \ge \ln(1/\delta)/(2a')}& e^{-\tau a'}+\frac{\tau^{2-\beta}\delta}{(a')^\beta} \ge \min_{\tau \ge \ln(1/\delta)/(2a')} \frac{\tau^{2-\beta} \delta}{(a')^\beta}=\Omega \left(\frac{\delta\ln^{2-\beta}(1/\delta)}{(a')^2} \right),\\
    \min_{0 \le \tau \le \ln(1/\delta)/(2a')}& e^{-\tau a'}+\frac{\tau^{2-\beta}\delta}{(a')^\beta} \ge \delta^{1/2} = \Omega(\delta\ln^{2-\beta}(1/\delta)) = \Omega \left(\frac{\delta\ln^{2-\beta}(1/\delta)}{(a')^2} \right).
\end{split}
\end{equation}
Then, since $\rho = \Omega(\gamma^{1 - h})$, we can set $\delta = \gamma$ and $\delta = \gamma^2 / \rho^2$ to obtain
\begin{equation}
    \underline R(1)=\Omega\left(\frac{\gamma\ln^{2/3}(1/\gamma)}{\lambda_{\max}^{2}} + \frac{\gamma^2\ln^{4/3}(\rho^2/\gamma^2)}{\rho^2 \lambda_{\max}^{2}}\right),
\end{equation}
which implies the desired result (due to Theorem \ref{thm:deteq} and Propositions \ref{prop:laststeprisk} and \ref{prop:conditionsandwhich}).

If $\underline a\le 2$, then \eqref{eq:ubEI} gives that
\begin{equation}
    E_{-1/3}\left(\frac{2\underline a}{3}\right)\le \Gamma\left(\frac{4}{3}\right)\left(\frac{2\underline a}{3}\right)^{-4/3},\qquad E_{1/3}\left(\frac{2\underline a}{3}\right)\le \Gamma\left(\frac{2}{3}\right)\left(\frac{2\underline a}{3}\right)^{-2/3},
\end{equation}
which implies that 
\begin{equation}
\begin{split}    
    \underline R(1)\ge R(0) e^{-2\underline a/3}=\Omega(1),
\end{split}
\end{equation}
thus again proving the desired claim (due to Theorem \ref{thm:deteq} and Propositions \ref{prop:laststeprisk} and \ref{prop:conditionsandwhich}).

Finally, for $c>1$, due to the same argument used to show \eqref{eq:munubigc}, the solution of the original ODE is lower bounded by that of the ODE below:
\begin{equation}\label{eq:forremark2}
    \diff \underline R'' = - 2 \lambda_{\max} v\sqrt{1-t} \underline R'' \diff t + v^2 b_1 \gamma(1-t) \diff t+2 v^2 c^2\frac{\gamma^2}{\rho^2}\diff t,
\end{equation}
with initial condition $\underline R''(0)=R(0)$, where $b_1$ is a positive constant independent of $\gamma, v, c, \rho$ and the spectrum of $\Sigma$. Thus, following the same steps above with $\underline a=2v \lambda_{\max}, \underline b_1=v^2b_1\gamma, \underline b_2=2v^2c^2\gamma^2/\rho^2$, one readily shows that 
\begin{equation}
    \underline R''(1) = \Omega \left(\frac{\gamma \ln^{2/3} ( 1/ \gamma)}{\lambda_{\max}^2} + \frac{\gamma^2 \ln^{4/3} \left( 1 / \gamma \right)}{\rho^2 \lambda^2_{\max}}\right),
\end{equation}
concluding the proof (due to Theorem \ref{thm:deteq} and Propositions \ref{prop:laststeprisk} and \ref{prop:conditionsandwhich}).
\end{proof}


\paragraph{Sub-optimality of $c = \omega_\gamma(1)$.} Note that, if $c > 1$, the ODE in \eqref{eq:forremark1} maps to the one in \eqref{eq:solauxlb}, with $\underline a, \underline b$ defined as in \eqref{eq:ab1} (except for absolute constants), with $vc \mapsto v$. This mapping also holds when defining \eqref{eq:lower01}, via $\rho \mapsto \rho / c$. Then, if we consider the further condition $\rho / c = \Omega(\gamma^{1 - h})$, for some $h > 0$, the lower bound in Proposition \ref{prop:alpha012} becomes
\begin{equation}\label{eq:suboptimalitycapp}
    \mathcal R(\theta^p_0) = \Omega \left(\frac{\gamma \ln ( 1/ \gamma)}{\lambda_{\max}^2} + \frac{\max(1, c^2) \gamma^2 \ln^2 \left( 1 / \gamma \right)}{\rho^2 \lambda^2_{\max}}\right).
\end{equation}
The same argument holds also for the ODE in \eqref{eq:forremark2}, providing analogous expression for $\mathcal R(\theta^p_{1/2})$. This quantifies the negative effects of using a large clipping constant $c = \omega_\gamma(1)$.

\subsection{Proofs on decaying private noise:  \texorpdfstring{$\alpha \geq 1$}{alpha gtr 1}}\label{sec:bigalpha}

In this section, we prove Proposition \ref{prop:alphabigger1body}. 
The result is stated in the general case where $\lambda_{\max}$ and $\lambda_{\min}$ are allowed to depend on $\gamma$. 

\begin{customthm}{Proposition~\ref{prop:alphabigger1body} (general version)}
    Let Assumptions \ref{ass:data} and \ref{ass:learning_rate_schedules} hold, and let $\theta^p_\alpha$ be the solution obtained with Algorithm \ref{alg:dp-sgd}, with $\tilde \eta(t)$ given by \eqref{eq:polynomialsched} for $\alpha \geq 1$. Consider the setting
\begin{equation}\label{eq:boundalpha12notexplode}
\gamma=\frac{d}{n}=o_\gamma(1),\qquad         \frac{\ln^2(1 / \gamma) \gamma}{\lambda_{\min}^2} \left( \lambda_{\max} \alpha + \frac{\gamma}{\rho^2}\right) = o_\gamma(1),
    \end{equation}
    and pick 
\begin{equation}
        c = O_\gamma(1), \qquad \tilde \eta(0)c = \frac{C \alpha \ln(1/\gamma)}{\lambda_{\min}}, \qquad \tilde \eta(0) \le \frac{2}{\gamma}, 
    \end{equation}
    for a large enough constant $C$ which does not depend on $\gamma, \rho, \Sigma, \alpha$. Then, we have that, with overwhelming probability,
   \begin{equation}
        \mathcal R(\theta^p_\alpha) =O_\gamma\left(\frac{\lambda_{\max}}{\lambda_{\min}^2} \alpha \gamma \ln^{\frac{1}{1+\alpha}}(1/\gamma)  +\frac{1}{\lambda_{\min}^2}\frac{\alpha^{2} \gamma^2 \ln^{\frac{2}{1+\alpha}}(1/\gamma)}{\rho^2} \right). 
    \end{equation}
\end{customthm}

\begin{proof}

Let $\overline{R}(t)$ be defined as in \eqref{eq:upperlowerODEsbody}, with the learning rate schedule in \eqref{eq:polynomialsched} for a generic $\alpha \geq 1$. Then, we have
\begin{equation}
\begin{aligned}
    \diff \overline R &= - 2 vc \lambda_{\min}\left(1 - t\right)^\alpha \frac{\mu_c(\overline R)}{c} \overline R \diff t + (v c)^2 \lambda_{\max} \left(1 - t\right)^{2 \alpha} \frac{\nu_c(\overline R) (\overline R + \zeta^2 / 2)}{c^2} \gamma \diff t \\
    &\qquad + 4 (v c)^2 \alpha \left(1 - t\right)^{2 \alpha - 1} \frac{\gamma^2}{\rho^2} \diff t.
\end{aligned}
\end{equation}
    The argument is similar to the one used to obtain Proposition \ref{prop:alpha012}, so we only highlight differences. Using \eqref{eq:chaplygin}, we have that $\overline R(t)$ is strictly bounded by the auxiliary ODEs
\begin{equation}\label{eq:auxiliaries3}
\begin{split}    
    \diff \overline R'  &=  - \overline a (1-t)^\alpha\overline R' \diff t + \overline b_1(1-t)^{2\alpha} \diff t+\overline b_2(1-t)^{2\alpha-1}\diff t=: \overline f'(t, \overline R' )\diff t,\qquad\overline a, \overline b_1,\overline b_2>0, \\
    \diff \underline R'  &=  - \underline a (1-t)^{\alpha}\underline R' \diff t + \underline b_1(1-t)^{2\alpha} \diff t+\underline b_2(1-t)^{2\alpha-1}\diff t=: \underline f'(t, \underline R')\diff t,\qquad\underline a, \underline b_1,\underline b_2>0, 
\end{split}
\end{equation}
with initial conditions $\overline R'(0) = \overline R(0) = R(0)$.
    
To establish a closed form solution for the ODEs in \eqref{eq:auxiliaries3}, we start by analyzing the ODE
\begin{equation}\label{eq:ODEnew}
    \diff \tilde R  = - \overline a (1-t)^\alpha \tilde R \diff t + \overline b_1 (1-t)^{2\alpha}\diff t,
\end{equation}
with initial condition $\tilde R(0)=R(0)$, which admits the closed form solution
\begin{equation}
\begin{split}    
    \tilde R(t)=\frac{R(0)}{(1+\alpha)}&e^{-\frac{\overline a(1-(1-t)^{1+\alpha})}{1+\alpha}}\biggl(1+\alpha-\frac{\overline b_1}{R(0)}e^{\frac{\overline a}{1+\alpha}}E_{-1+\frac{1}{1+\alpha}}\left(\frac{\overline a}{1+\alpha}\right)\\
    &\hspace{2em}+ \frac{\overline b_1}{R(0)} e^{\frac{\overline a}{1+\alpha}}(1-t)^{1+2\alpha} E_{-1+\frac{1}{1+\alpha}}\left(\frac{\overline a(1-t)^{1+\alpha}}{1+\alpha}\right)\biggr).
    \end{split}
\end{equation}
Let 
\begin{equation}\label{eq:Gronw}
    w(t)=\overline R'(t)-\tilde R(t),
\end{equation}
and note that
\begin{equation}
\begin{aligned}
    \frac{\diff w}{\diff t} &= \frac{\diff \overline R'}{\diff t}-\frac{\diff \tilde R}{\diff t} \\
    &= -\overline a(1-t)^\alpha \overline R+\overline a(1-t)^\alpha \tilde R +\overline b_2(1-t)^{2\alpha-1} \\
    &= -\overline a(1-t)^\alpha w+\overline b_2(1-t)^{2\alpha-1}. 
\end{aligned}
\end{equation}

Thus, by using the initial condition $w(0)=\overline R'(0)-\tilde R(0)=0$, we have
\begin{equation}
    w(t)= \frac{\overline b_2}{1+\alpha} e^{\frac{\overline a(1-t)^{1+\alpha}}{1+\alpha}}\left((1-t)^{2\alpha}E_{\frac{1-\alpha}{1+\alpha}}\left(\frac{\overline a(1-t)^{1+\alpha}}{1+\alpha}\right)-E_{\frac{1-\alpha}{1+\alpha}}\left(\frac{\overline a}{1+\alpha}\right)\right),
\end{equation}
which implies that
\begin{equation}
\begin{split}    
    \overline R'(t)=\frac{R(0)}{(1+\alpha)}&e^{-\frac{\overline a(1-(1-t)^{1+\alpha})}{1+\alpha}}\biggl(1+\alpha-\frac{\overline b_1}{R(0)}e^{\frac{\overline a}{1+\alpha}}E_{-1+\frac{1}{1+\alpha}}\left(\frac{\overline a}{1+\alpha}\right)\\
    &\hspace{2em}+ \frac{\overline b_1}{R(0)}e^{\frac{\overline a}{1+\alpha}}(1-t)^{1+2\alpha} E_{-1+\frac{1}{1+\alpha}}\left(\frac{\overline a(1-t)^{1+\alpha}}{1+\alpha}\right)\\
    +\frac{\overline b_2}{R(0)} &e^{\frac{\overline a}{1+\alpha}}\left((1-t)^{2\alpha}E_{\frac{1-\alpha}{1+\alpha}}\left(\frac{\overline a(1-t)^{1+\alpha}}{1+\alpha}\right)-E_{\frac{1-\alpha}{1+\alpha}}\left(\frac{\overline a}{1+\alpha}\right)\right)\biggr).
    \end{split}
\end{equation}
Similarly, we have that 
\begin{equation}
\begin{split}    
    \underline R'(t)=\frac{1}{(1+\alpha)}&e^{-\frac{\underline a(1-(1-t)^{1+\alpha})}{1+\alpha}}\biggl(R(0) (1+\alpha) -\underline b_1e^{\frac{\underline a}{1+\alpha}}E_{-1+\frac{1}{1+\alpha}}\left(\frac{\underline a}{1+\alpha}\right)\\
    &\hspace{2em}+ \underline b_1e^{\frac{\underline a}{1+\alpha}}(1-t)^{1+2\alpha} E_{-1+\frac{1}{1+\alpha}}\left(\frac{\underline a(1-t)^{1+\alpha}}{1+\alpha}\right)\\
    +\underline b_2 &e^{\frac{\underline a}{1+\alpha}}\left((1-t)^{2\alpha}E_{\frac{1-\alpha}{1+\alpha}}\left(\frac{\underline a(1-t)^{1+\alpha}}{1+\alpha}\right)-E_{\frac{1-\alpha}{1+\alpha}}\left(\frac{\underline a}{1+\alpha}\right)\right)\biggr).
    \end{split}
\end{equation}
For $t\in [0, 1]$, $\overline R(t), \underline R(t)\ge 0$. Furthermore, the following upper bounds hold for $t\in [0, 1]$:
\begin{equation}
\begin{split}    
\overline R'(t)\le R(0) +\frac{2\overline b_1}{\overline a}+\frac{\overline b_2}{\alpha}&+\overline b_1\left(\frac{1}{1+\alpha}\right)^{-1+\frac{1}{1+\alpha}}\Gamma\left(2-\frac{1}{1+\alpha}\right) \overline a^{-2+\frac{1}{1+\alpha}}\\
&+\overline b_2\Gamma\left(\frac{2\alpha}{1+\alpha}\right)(1+\alpha)^{\frac{\alpha-1}{1+\alpha}} \overline a^{-\frac{2\alpha}{1+\alpha}}.
\end{split}
\end{equation}
Let us take
\begin{equation}
    \overline a=C_1 vc \lambda_{\min}, \qquad \overline b_1=C_2 \gamma v^2 c^2 \lambda_{\max}, \qquad \overline b_2=C_3 v^2c^2\frac{\gamma^2}{\rho^2}\alpha,
\end{equation}
with $C_1, C_2, C_3$ constants independent from $\gamma, \rho, \alpha, v, c$ or the spectrum of $\Sigma$. Then, we have that $\overline R'(t)\le 1 + R(0)$ 
for $t\in [0, 1]$, as $\overline a = \Omega(\ln(1 / \gamma))$ and
\begin{equation}
    \frac{\ln^2(1 / \gamma) \gamma}{\lambda_{\min}^2} \left( \lambda_{\max} \alpha + \frac{\gamma}{\rho^2}\right) = o(1).
\end{equation}

As a result, we can apply the argument in \eqref{eq:chaplygin} to obtain that
\begin{equation}\label{eq:dec}
    \overline R (1) \le \overline R'(1)= \tilde R(1) + w(1).
\end{equation}

Note that Lemma \ref{lemma:integral} yields 
\begin{equation}
    \begin{split}
\lim_{x\to 0}&x^{1+2\alpha} E_{-1+\frac{1}{1+\alpha}}\left(\frac{\overline ax^{1+\alpha}}{1+\alpha}\right)=\left(\frac{1}{1+\alpha}\right)^{-2+\frac{1}{1+\alpha}}\Gamma\left(2-\frac{1}{1+\alpha}\right)\overline a^{-2+\frac{1}{1+\alpha}}.
    \end{split}
\end{equation}
Thus,
\begin{equation}\label{eq:ubtildeR}
    \begin{split}
        \tilde R(1)=\frac{1}{(1+\alpha)}e^{-\frac{\overline a}{1+\alpha}}\biggl(R(0) &(1+\alpha) -\overline b_1e^{\frac{\overline a}{1+\alpha}}E_{-1+\frac{1}{1+\alpha}}\left(\frac{\overline a}{1+\alpha}\right)\\
        &+\overline b_1e^{\frac{\overline a}{1+\alpha}}\left(\frac{1}{1+\alpha}\right)^{-2+\frac{1}{1+\alpha}}\Gamma\left(2-\frac{1}{1+\alpha}\right)\overline a^{-2+\frac{1}{1+\alpha}}\biggr)\\
        &\hspace{-11em}\le R(0) e^{-\frac{\overline a}{1+\alpha}}+\left(\frac{1}{1+\alpha}\right)^{-1+\frac{1}{1+\alpha}} \Gamma\left(2-\frac{1}{1+\alpha}\right)\overline b_1\overline a^{-2+\frac{1}{1+\alpha}}\\
        &\hspace{-11em}\le R(0) e^{-\frac{\overline a}{1+\alpha}}+(1+\alpha) \overline b_1\overline a^{-2+\frac{1}{1+\alpha}},
    \end{split}
\end{equation}
where in the second line we have used the non-negativity of the exponential integral function and in the third line we have used that $\Gamma\left(2-\frac{1}{1+\alpha}\right)\le \Gamma(2)= 1$ 
for all $\alpha \ge 1$. Next, we bound $w(1)$ as
\begin{equation}\label{eq:ubw}
\begin{split}    
    w(1)&=\frac{\overline b_2}{1+\alpha}\left(\Gamma\left(\frac{2\alpha}{1+\alpha}\right)(1+\alpha)^{\frac{2\alpha}{1+\alpha}}\overline a^{-\frac{2\alpha}{1+\alpha}}-E_{\frac{1-\alpha}{1+\alpha}}\left(\frac{\overline a}{1+\alpha}\right)\right)\\
    &\le \frac{\overline b_2}{1+\alpha}\Gamma\left(\frac{2\alpha}{1+\alpha}\right)(1+\alpha)^{\frac{2\alpha}{1+\alpha}}\overline a^{-\frac{2\alpha}{1+\alpha}}\\        &\le (1+\alpha)\overline b_2\overline a^{-\frac{2\alpha}{1+\alpha}},
    \end{split}
\end{equation}
where in the last line we have used that $\Gamma\left(\frac{2\alpha}{1+\alpha}\right)\le 1$ for $\alpha\ge 1$.
By combining \eqref{eq:dec}, \eqref{eq:ubtildeR} and \eqref{eq:ubw}, we conclude that 
\begin{equation}
\begin{split}    
     R(1)& \le R(0)e^{-\frac{\overline a}{1+\alpha}}+(1+\alpha)^{1/3} \overline b_1\overline a^{-\frac{2\alpha+1}{1+\alpha}}+(1+\alpha)\overline b_2\overline a^{-\frac{2\alpha}{1+\alpha}}\\
     &=O\left(\frac{\lambda_{\max} \alpha \gamma \ln^{\frac{1}{1+\alpha}}(1/\gamma)}{\lambda_{\min}^2}   +\frac{\alpha^{2} \gamma^2 \ln^{\frac{2}{1+\alpha}}(1/\gamma)}{\rho^2 \lambda_{\min}^2} \right),
\end{split}
\end{equation}
which gives the thesis, after applying Theorem \ref{thm:deteq} and Propositions \ref{prop:laststeprisk} and \ref{prop:conditionsandwhich}, since $\tilde \eta(1) = 0$.
\end{proof}

\paragraph{Proof of \eqref{eq:boundopts2}.}
By taking $\alpha = \ln\ln(1/\gamma)$, we have
\begin{equation}
\begin{aligned}
    & \alpha \gamma \ln^{\frac{1}{1+\alpha}}(1/\gamma) + \frac{\alpha^{2} \gamma^2 \ln^{\frac{2}{1+\alpha}}(1/\gamma)}{\rho^2} \\
    & \qquad = \gamma \ln\ln(1/\gamma) e^{\frac{\ln\ln(1/\gamma)}{1 + \ln\ln(1/\gamma)}} + \frac{\gamma^2}{\rho^2}(\ln\ln(1/\gamma))^{2} e^{\frac{2 \ln \ln (1 / \gamma)}{1 + \ln \ln (1 / \gamma)}} \\
    & \qquad \le  \gamma (\ln\ln(1/\gamma)) e + \frac{\gamma^2}{\rho^2}(\ln\ln(1/\gamma))^{2} e^2,
\end{aligned}
\end{equation}
which readily gives \eqref{eq:boundopts2}, as $\lambda_{\max}, \lambda_{\min}= \Theta_\gamma(1)$. \qed

\subsection{Proofs on the harmonic schedule}\label{app:harmonic}

In this section, we prove Theorem \ref{thm:harmonicbody}, providing a more general result that allows for $\lambda_{\max}$ and $\lambda_{\min}$ to depend on $\gamma$.

\begin{lemma}\label{lemma:ODEclosed}
    Consider the ODE
\begin{equation}\label{eq:aux-lemma}
\begin{split}    
    \diff \tilde R  &=  - a \tilde\eta(t) \tilde R \diff t + b_1 \tilde\eta^2(t) \diff t -  b_2 \frac{\diff \tilde\eta^2(t)}{\diff t}\diff t,\qquad a, b_1, b_2>0, 
\end{split}
\end{equation}
with initial condition $\tilde R(0)$. Then,
\begin{equation}
    \tilde R (t)= e^{- a F(t)}\left(\tilde R(0) + \int_0^t e^{ a F(s)}\bigl( b_1\tilde\eta^2(s)-2  b_2\tilde\eta(s)\tilde\eta'(s)\bigr)\diff s\right),
\end{equation}
where $\tilde\eta'(t)$ denotes the derivative of $\tilde\eta(t)$ and $F(t):=\int_0^t \tilde\eta(s)\diff s$. 
\end{lemma}

\begin{proof}
    Define
    \begin{equation*}
    u(t):=\tilde R(t)+b_2\tilde\eta^2(t).
    \end{equation*}
    Then,
    \begin{equation*}
        \diff u(t)=-a\tilde\eta(t)u(t)\diff t +b_1\tilde\eta^2(t)\diff t+ab_2\tilde\eta^3(t)\diff t,
    \end{equation*}
which implies 
    \begin{equation*}
        \diff \left(u(t)e^{aF(t)}\right)=b_1\tilde\eta^2(t)e^{aF(t)}\diff t+ab_2\tilde\eta^3(t)e^{aF(t)}\diff t.
    \end{equation*}
    By integrating the equation above, we obtain
    \begin{equation*}
        u(t)e^{aF(t)}=u(0)e^{aF(0)}+\int_0^t e^{aF(s)}\bigl(b_1\tilde\eta^2(s)+ab_2\tilde\eta^3(s)\bigr)\diff s.
    \end{equation*}
    Noting that $u(0)=\tilde R(0)+b_2\tilde\eta^2(0)$ and $F(0)=0$ gives 
    \begin{equation}\label{eq:neweq1}
    \tilde R (t)=- b_2\tilde\eta^2(t)  +e^{- a F(t)}\tilde R(0) + e^{- a F(t)} b_2\tilde\eta^2(0)+e^{- a F(t)}\int_0^t e^{ a F(s)}\bigl( b_1\tilde\eta^2(s)+ a b_2\tilde\eta^3(s)\bigr)\diff s.
\end{equation}
Furthermore, we have
\begin{equation*}
    \frac{\diff}{\diff s}\left(e^{aF(s)}\tilde\eta^2(s)\right)=ae^{aF(s)}\tilde\eta^3(s)+2e^{aF(s)}\tilde\eta(s)\tilde\eta'(s),
\end{equation*}
which implies that
\begin{equation}\label{eq:neweq2}
\int_0^t a e^{aF(s)}\tilde\eta^3(s)\diff s = e^{aF(t)}\tilde\eta^2(t)-\tilde\eta^2(0)-2\int_0^te^{aF(s)}\tilde\eta(s)\tilde\eta'(s)\diff s.    
\end{equation}
Combining \eqref{eq:neweq1} and \eqref{eq:neweq2} gives
    the desired result.
\end{proof}

\begin{customthm}{Theorem~\ref{thm:harmonicbody} (general version)} 
    Let Assumptions \ref{ass:data} and \ref{ass:learning_rate_schedules} hold, and let $\theta^p$ be the solution obtained with Algorithm \ref{alg:dp-sgd}. Assume that 
    \begin{equation}
        \frac{\lambda_{\max}}{\lambda_{\min}^2}\gamma+\frac{1}{\lambda_{\min}^2}\frac{\gamma^2}{\rho^2}=o_\gamma(1).
    \end{equation}
Consider the schedule
\begin{equation}\label{eq:optsched}
    \tilde \eta(t) = \frac{\alpha}{t+\tau}, \qquad \mbox{with }\alpha:=\frac{3\sqrt{A+90}}{\lambda_{\min}}, \quad \tau:=\max\left(\frac{9(A+90)\lambda_{\max}}{\lambda_{\min}^2}\gamma, \frac{12\sqrt{2(A+90)}}{\lambda_{\min}}\frac{\gamma}{\rho}\right),
\end{equation}
where $A$ is a universal constant s.t.\ $R(0)+\zeta^2/2\le A$.  
    Then, by setting $c=3\sqrt{2(A+90)}/\lambda_{\min}$, we have that, with overwhelming probability,
    \begin{equation}
        \mathcal R(\theta^p) = O_\gamma \left( \frac{\lambda_{\max} }{\lambda_{\min}^4}\gamma + \frac{1}{\lambda_{\min}^4}\frac{\gamma^2}{\rho^2} \right). 
    \end{equation}
\end{customthm}


\begin{proof}
    Let us introduce the shorthands
\begin{equation*}
\begin{aligned}
    \overline f(t, \overline R) &= - 2 \lambda_{\min}  \tilde\eta(t) \mu_c(\overline R) \overline{R} +  \lambda_{\max} c^2 \tilde\eta^2(t) \frac{\nu_c(\overline R) (\overline {R} + \zeta^2/2)}{c^2} \gamma - 2 c^2 \frac{\gamma^2}{\rho^2}\frac{\diff \tilde\eta^2(t)}{\diff t}, \\
    \underline f(t, \underline R) &= - 2 \lambda_{\max} \tilde\eta(t) \mu_c(\underline R) \underline{R} + c^2 \tilde\eta^2(t) \frac{\nu_c(\underline R) (\underline {R} + \zeta^2/2)}{c^2} \gamma  - 2 c^2 \frac{\gamma^2}{\rho^2}\frac{\diff \tilde\eta^2(t)}{\diff t}.
\end{aligned}
\end{equation*}
Then, consider the auxiliary ODEs 
\begin{equation}\label{eq:aux-opt}
\begin{split}    
    \diff \overline R'  &=  - \overline a \tilde\eta(t) \overline R' \diff t + \overline b_1 \tilde\eta^2(t) \diff t - \overline b_2 \frac{\diff \tilde\eta^2(t)}{\diff t}\diff t=: \overline f'(t, \overline R')\diff t,\qquad\overline a, \overline b_1,\overline b_2>0, \\
    \diff \underline R'  &=  - \underline a \tilde\eta(t)\underline R' \diff t + \underline b_1 \tilde\eta^2(t) \diff t-\underline b_2\frac{\diff \tilde\eta^2(t)}{\diff t}\diff t=: \underline f'(t, \underline R')\diff t,\qquad\underline a, \underline b_1,\underline b_2>0, 
\end{split}
\end{equation}
with initial conditions $\overline R'(0) = \overline R(0) = R(0) = \underline R(0) = \underline R'(0)$. By Lemma \ref{lemma:ODEclosed}, we have
\begin{equation*}
\begin{split}    
    \overline R' (t)=e^{-\overline a F(t)}\left(R(0)+\int_0^t e^{\overline a F(s)}\bigl(\overline b_1\tilde\eta^2(s)-2\overline b_2\tilde\eta(s)\tilde\eta'(s)\bigr)\diff s\right),  \\
    \underline R' (t)=e^{-\underline a F(t)}\left(R(0)+\int_0^t e^{\underline a F(s)}\bigl(\underline b_1\tilde\eta^2(s)-2\underline b_2\tilde\eta(s)\tilde\eta'(s)\bigr)\diff s\right),
    \end{split}
\end{equation*}
where
\begin{equation*}
   F(t):=\int_0^t \tilde\eta(s)\diff s. 
\end{equation*}
By plugging in the schedule defined in \eqref{eq:optsched}, we obtain
\begin{equation*}
F(t)=\alpha\log\frac{t+\tau}{\tau}, \qquad \tilde\eta'(t)=-\frac{\alpha}{(t+\tau)^2}.
\end{equation*}
Thus, assuming that $\overline a\alpha>2$ (which will be verified later), we have that
\begin{equation*}
    \begin{split}    
    \overline R' (t)=\left(\frac{\tau}{t+\tau}\right)^{\overline a\alpha}\biggl(R(0)&+\overline b_1\frac{\alpha^2}{\tau^{\overline a\alpha}(\overline a\alpha-1)}\left((t+\tau)^{\overline a\alpha-1}-\tau^{\overline a\alpha-1}\right)\\
    &+2\overline b_2\frac{\alpha^2}{\tau^{\overline a\alpha}(\overline a\alpha-2)}\left((t+\tau)^{\overline a\alpha-2}-\tau^{\overline a\alpha-2}\right)\biggr).
    \end{split}
\end{equation*}
This readily gives the upper bound
\begin{equation}\label{eq:ubR1}
\begin{split}    
    \overline R'(t)&\le R(0) +\overline b_1\frac{\alpha^2}{(t+\tau)(\overline a\alpha-1)}+2\overline b_2\frac{\alpha^2}{(t+\tau)^2(\overline a\alpha-2)}\\
    &\le R(0) +\overline b_1\frac{\alpha^2}{\tau(\overline a\alpha-1)}+2\overline b_2\frac{\alpha^2}{\tau^2(\overline a\alpha-2)}
\end{split}
\end{equation}
valid for all $t\in [0, 1]$. Furthermore, for $t=1$, we also obtain the upper bound
\begin{equation}\label{eq:ubR2}
\begin{split}    
    \overline R'(1)&\le \left(\frac{\tau}{1+\tau}\right)^{\overline a\alpha}R(0) +\overline b_1\frac{\alpha^2}{(1+\tau)(\overline a\alpha-1)}+2\overline b_2\frac{\alpha^2}{(1+\tau)^2(\overline a\alpha-2)}\\
    &\le \tau^{\overline a\alpha}R(0) +\overline b_1\frac{\alpha^2}{\overline a\alpha-1}+2\overline b_2\frac{\alpha^2}{\overline a\alpha-2}.
\end{split}
\end{equation}
Let us now pick
\begin{equation}\label{eq:setabc-u}
    \overline a = \frac{\lambda_{\min}}{\sqrt{A+90}}, \qquad \overline b_1 = \lambda_{\max}  \frac{c^2 \gamma}{2}, \qquad \overline b_2 = 4 c^2\frac{\gamma^2}{\rho^2}.
\end{equation}
This implies that
\begin{equation}
\tau=\max\left(\overline b_1, 2\sqrt{\overline b_2}\right),\qquad \overline a\alpha=3,     
\end{equation}
which combined with \eqref{eq:ubR1} gives
\begin{equation}\label{eq:bdtraj}
    \overline R'(t)\le R(0) +\alpha^2=R(0)+\frac{9(A+90)}{\lambda_{\min}^2}.
\end{equation}
We now verify that, for all $t\in [0, 1]$,
\begin{equation}\label{eq:verconda}
    \overline a\le 2\lambda_{\min}\mu_c(\overline R(t))=2\lambda_{\min}\erf\left(\frac{c}{2\sqrt{\overline R(t)+\zeta^2/2}}\right).
\end{equation}
For \eqref{eq:verconda} to hold, it suffices that
\begin{equation}\label{eq:verconda1}
    \overline a\le2\lambda_{\min}\erf\left(\frac{c}{2\sqrt{A+9(A+90)/\lambda_{\min}^2}}\right),
\end{equation}
due to \eqref{eq:bdtraj} and to the fact that $R(0)+\zeta^2/2\le A$. Note that 
\begin{equation*}
    \erf(z)=\frac{2}{\sqrt{\pi}}\int_0^z e^{-t^2}\diff t\ge \frac{2z}{\sqrt{\pi e}}\ge \frac{2z}{9},
\end{equation*}
where the inequalities holds for all $z\in [0, 1/\sqrt{2}]$. Furthermore, 
as $c=3\sqrt{2(A+90)}/\lambda_{\min}$, one can readily verify that 
\begin{equation*}
    \frac{c}{2\sqrt{A+9(A+90)/\lambda_{\min}^2}}\le \frac{1}{\sqrt{2}}.
\end{equation*}
Thus, \eqref{eq:verconda1} is implied by
\begin{equation}
    \overline a \le \frac{2c\lambda_{\min}}{9\sqrt{A+9(A+90)/\lambda_{\min}^2}},
\end{equation}
which holds due to the definitions of $c, \overline a$ and to the fact that $\lambda_{\min}\le 1$.

By Lemma \ref{lemma:munubounds}, we have that 
\begin{equation*}
   \frac{2 \nu_c(R) \left( R + \zeta^2 / 2 \right)}{c^2} < 1,
\end{equation*}
which implies that
\begin{equation}
    \overline b_1>\lambda_{\max}c^2   \frac{\nu_c(R) \left( \overline R(t) + \zeta^2 / 2 \right)}{c^2}\gamma,
\end{equation}
for all $t\in [0, 1]$.
Now, note that $\diff\tilde\eta^2(t)/\diff t\le -2\alpha^2/(1+\tau)^3< 0$ for $t\in [0, 1]$. Thus, $\overline R'(t)\ge 0$. In fact, if this is not the case, by continuity of $\overline R'$, there exists an interval $(t^*, t^*+\delta)\subseteq [0, 1]$ s.t.\  $\overline R'(t^*)=0$ and $\overline R'(t)<0$ for all $t\in (t^*, t^*+\delta)$. However, if $\overline R'(t^*)=0$, then the derivative of $\overline R'$ evaluated at $t^*$ is $>0$, which is a contradiction. 

As a result, we have
\begin{equation}\label{ineq:fg2-new}
    \overline f(t, \overline R'(t)) < \overline f'(t, \overline R'(t)),\qquad \text{ for all } \, t \in [0, 1].
\end{equation}
Again, by Lemma \ref{lemma:munubounds}, we have that 
\begin{equation*}
     \frac{\mu_c(R)}{c} < \frac{1}{\sqrt {\pi \left( R + \zeta^2 /2 \right)}}, \qquad \underline c_\nu(c, \zeta) < \frac{2 \nu_c(R) \left( R + \zeta^2 / 2 \right)}{c^2}.
\end{equation*}
Thus, by setting 
\begin{equation*}
\begin{split}
    \underline a &= 2 c \lambda_{\max} \frac{1}{ \sqrt{\pi \zeta^2 / 2}}, \qquad \underline b_1 = \frac{ c^2 \gamma\underline c_\nu(c, \zeta) }{2}, \qquad \underline b_2 =  c^2\frac{\gamma^2}{\rho^2},    
\end{split}
\end{equation*}
we have that 
\begin{equation}\label{ineq:fg2-new2}
\underline f(t, \underline R'(t)) > \underline f'(t, \underline R'(t)),\qquad \text{ for all } \, t \in [0, 1].
\end{equation}
Note that $\overline f(t, R)$ and $\underline f(t ,R)$ are continuous in both variables in the intervals $t \in [0, 1]$ and $R \in [0, 1]$. Furthermore, they are also Lipschitz in $R$ in these same intervals. Thus, by Theorem 1.3 in \cite{teschl2012}, 
\begin{equation} 
  \underline R'(t) < \underline R(t), \qquad \overline R(t) < \overline R'(t) ,\qquad \text{ for all } \, t \in (0,1].
\end{equation}
Note that 
\begin{equation}
    \sup_{t\in [0, 1]}\tilde\eta(t)=\frac{\alpha}{\tau}<\frac{2}{\gamma},
\end{equation}
which follows from
\begin{equation}
    \frac{\alpha}{\tau}\le \frac{3\sqrt{A+90}}{\lambda_{\min}}\frac{\lambda_{\min}^2}{9\gamma\lambda_{\max}(A+90)}\le \frac{1}{3\gamma\sqrt{A+90}}\le \frac{1}{3\gamma\sqrt{90}}.
\end{equation}
Thus, noting that $$\overline R(1)+\frac{c^2\tilde\eta^2(1)\gamma^2}{\rho^2}= O_\gamma \left( \frac{\lambda_{\max} }{\lambda_{\min}^4}\gamma + \frac{1}{\lambda_{\min}^4}\frac{\gamma^2}{\rho^2} \right),$$ the desired result follows from Theorem \ref{thm:deteq} and Propositions \ref{prop:laststeprisk} and \ref{prop:conditionsandwhich}.
\end{proof}

\subsection{Proofs on the minimax rates}\label{app:lowerbound}

In this section, we consider the notation $D = (X, Y)$ and $D'_i = (X'_i, Y'_i)$, with $X, X'_i \in \R^{n \times d}$ and $Y, Y'_i \in \R^{n}$, where the two neighboring datasets differ in their $i$-th row $(x_i, y_i)$ and $(x_i', y_i')$. Assume that all data satisfies Assumption \ref{ass:data}, i.e., $(x_i, y_i)$ are sampled i.i.d. such that $x_i \sim \mathcal N(0, \Sigma)$, $y_i = x_i^\top \theta^* + z_i$, and $(x'_i, y'_i)$ is sampled from this same distribution independently. 
In this section, the asymptotic notation will hide only numerical absolute constants, independent of $\zeta$ and $\gamma$.

Denote by $\mathcal M$ a generic $\rho^2 / 2$-zCDP algorithm. $\mathcal M$ applied to $D$ (or $D'$) outputs a random variable in the probability space of the data and the private mechanism.
Let $p_{\mathcal{M}(D)}$ ($p_{\mathcal{M}(D'_i)}$) be the probability density function of the random variable $\mathcal{M}(D)$ ($\mathcal{M}(D'_i)$) when taking into account the randomness of $\mathcal M$, $X$ and $Y$ ($X'_i$ and $Y'_i$).
Define $L_i$ to be the random variable
\begin{align}\label{eq:Li}
    L_i = \frac{p_{\mathcal{M}(D)} ({\theta}')}{p_{\mathcal{M}(D'_i)} ({\theta}')},
\end{align}
where ${\theta}'$ is set to be the random variable ${\theta}' = \mathcal{M}(D'_i)$. 
Then, due to the definition of zCDP (see Definition \ref{def:zcdp}), we have that, for all $D$, $D'_i$ and $\alpha > 1$,
\begin{align}
    \E\left[ L^\alpha_i \, | \, D, D'_i \right] \leq \exp \left( \frac{\alpha (\alpha - 1) \rho^2 }{2}\right).
\end{align}
Thus, taking the expectation with respect to $D$ and $D'_i$, we have
\begin{align}
    \E\left[ L^\alpha_i \right] \leq \exp \left( \frac{\alpha (\alpha - 1) \rho^2 }{2}\right).
\end{align}
As in \cite{cai21}, let us define the tracing attack
\begin{align}\label{eq:tracingattack}
    A_i(\theta) = z_i x_i^\top \left(\theta - \theta^* \right).
\end{align}

Unless differently stated, in this section, when we use the notation $\E$, we refer to expectations on the randomness of the data $(X, Y)$ and of the algorithm $\mathcal M$.

\begin{lemma}\label{lemma:omegad}
Let Assumption \ref{ass:data} hold and let $\zeta^2$ be an absolute positive constant.
Then, there exists a prior $\pi$ on $\theta^*$ with support in $\norm{\theta^*}_2 < 1$ and an absolute constant $C > 0$ such that, for any 
algorithm $\mathcal M$ with $\sup_{\norm{\theta^*}_2 < 1} \E \left[ \norm{\mathcal M}_2^2 \right] < \infty$, either one of the following holds
\begin{equation}\label{eq:twoconditions}
    \E_\pi \left[ \sum_{i \in [n]} \E \left[ A_i (\mathcal M(D)) \right] \right] \geq C d, \qquad \textup{or} \qquad \E_{\pi, X, Y, \mathcal M} \left[ \norm{ \mathcal M(D) - \theta^* }_2^2 \right] > C.
\end{equation}
\end{lemma}
\begin{proof}


Consider a generic algorithm $\mathcal M$. We have
\begin{align}\label{eq:sumoverj}
    \sum_{i = 1}^n \E \left[ A_i(\mathcal M(D)) \right] =  \sum_{i = 1}^n \E \left[ z_i x_i^\top  \mathcal M(D) \right] = \sum_{j \in [d]} \E \left[ \mathcal M(D)_j \sum_{i \in [n]} x_{ij} (y_i - x_i^\top \theta^*) \right],
\end{align}
since the term $\theta^*$ vanishes in expectation from the definition of $A_i$ in \eqref{eq:tracingattack} and where we replace $z_i = y_i - x_i^\top \theta^*$ and introduce the notation $x_{ij} \in \R$ to denote the $j$-th entry of the data vector $x_i$.
By Assumption \ref{ass:data}, we have
\begin{align}
    p(y_i | x_i, \theta^*) = \mathcal N(x_i^\top \theta^*, \zeta^2), \qquad \frac{\partial}{\partial \theta^*_j} p(y_i | x_i, \theta^*) = \frac{x_{ij} (y_i - x_i^\top \theta^*)}{\zeta^2} p(y_i | x_i, \theta^*),
\end{align}
which give
\begin{equation}
\begin{aligned}
    \zeta^2 \frac{\partial}{\partial \theta^*_j} \log p(y_i | x_i, \theta^*) &= x_{ij} (y_i - x_i^\top \theta^*), \\
    \zeta^2 \frac{\partial}{\partial \theta^*_j} \log p(Y | X, \theta^*) &= \sum_{i \in [n]} x_{ij} (y_i - x_i^\top \theta^*),
\end{aligned}
\end{equation}
as $p(Y | X, \theta^*) = \Pi_i p(y_i | x_i, \theta^*)$. Thus, we have (dropping the dependence of $\mathcal M$ on $X$ and $Y$)
\begin{equation}\label{eq:swapintder}
\begin{split}    \E & \left[ \mathcal M_j \sum_{i \in [n]} x_{ij} (y_i - x_i^\top \theta^*) \right] = \E \left[ \mathcal M_j \zeta^2 \frac{\partial}{\partial \theta^*_j} \log p(Y | X, \theta^*) \right] \\
    &= \int \mathcal M_j \zeta^2 \left(\frac{1}{p(Y | X, \theta^*)} \frac{\partial}{\partial \theta^*_j} p(Y | X, \theta^*)\right) p(X) p(Y|X, \theta^*) p(\mathcal M_j | X, Y) \diff X \diff Y \diff \mathcal M_j  \\
    &= \int \mathcal M_j \zeta^2 \left(\frac{\partial}{\partial \theta^*_j} p(Y | X, \theta^*)\right) p(X)  p(\mathcal M_j | X, Y) \diff X \diff Y \diff \mathcal M_j \\
    &= \frac{\partial}{\partial \theta^*_j} \int \mathcal M_j \zeta^2 p(X) p(Y | X, \theta^*) p(\mathcal M_j | X, Y) \diff X \diff Y \diff \mathcal M_j \\
    &= \zeta^2 \frac{\partial}{\partial \theta^*_j} \E \left[ \mathcal M_j \right],
\end{split}
\end{equation}
where the fourth line holds since $p(Y | X)$ is the only term that depends on $\theta^*$, and the derivative is swapped with the integration due to dominated convergence, as in the open ball $\norm{\theta^*}_2 < 1$ we uniformly have
\begin{equation}\label{eq:regularityM}
    \E \left[ \left( \frac{\partial}{\partial \theta^*_j} \log p(Y | X, \theta^*) \right)^2\right] < \infty, \qquad \E \left[ \mathcal M_j^2 \right] < \infty.
\end{equation}
Here, the first condition follows from Assumption \ref{ass:data}, and the second one by hypothesis. Notice that the previous equation, by Cauchy-Schwartz inequality, implies the condition for completeness in Theorem 2.1 in \cite{cai2025scoreattacklowerbound}, and that this condition also implies the function $g_j(\theta^*) := \E \left[ \mathcal M_j \right]$ being differentiable for all $\theta^*$ in $\norm{\theta^*}_2 < 1$. Then, plugging \eqref{eq:swapintder} in \eqref{eq:sumoverj} and taking the expectation over the prior $\pi$ of $\theta^*$, we get
\begin{align}\label{eq:fromntod}
    \E_\pi \left[ \sum_{i \in [n]} \E \left[ A_i (\mathcal M(D)) \right] \right] =  \zeta^2 \E_\pi \left[ \sum_{j \in [d]} \frac{\partial}{\partial \theta^*_j}  g_j(\theta^*) \right].
\end{align}

Consider the prior $\pi$ on $\theta^*$ such that all the entries $\theta^*_j$ are independent and distributed according to 
\begin{equation}
    \pi_j(z) = \frac{15 d^{5 / 2}}{16} \left( \frac{1}{d} - z^2 \right)^2,
\end{equation}
with support in $z \in (- d^{- 1 / 2}, d^{- 1 / 2})$. This prior is such that the support is $\norm{\theta^*}_2 < 1$, it vanishes at the boundary and 
\begin{equation}\label{eq:sortofvariance}
    - \E_{z \sim \pi_j} \left[ - z \frac{\pi'_j(z)}{\pi_j(z)} \right] = 1, \qquad \E_{z \sim \pi_j} \left[ \left( \frac{\pi'_j(z)}{\pi_j(z)} \right) ^2 \right] = 10 d.
\end{equation}

Let us define the one dimensional function
\begin{align}
    h_j(\theta^*_j) := \E_{\theta^*_{-j}}\left[ g_j(\theta^*) \right], 
\end{align}
where the notation $\E_{\theta^*_{-j}}$ takes expectation over all the (independent) entries of $\theta^*$, except for the $j$-th. Then, due to \eqref{eq:regularityM}, we can apply Stein's Lemma as stated in Lemma 2.1 in \cite{cai2025scoreattacklowerbound}, which yields
\begin{align}\label{eq:steinlemma}
    \E_\pi \left[ \frac{\partial}{\partial \theta^*_j} g_j(\theta^*) \right] = \E_{\theta^*_j} \left[ \frac{\partial}{\partial \theta^*_j} h_j(\theta^*_j) \right] = \E_{\theta^*_j} \left[ - h_j(\theta^*_j) \frac{\pi_j'(\theta^*_j)}{\pi_j(\theta^*_j)} \right] = \E_\pi \left[ - g_j(\theta^*) \frac{\pi_j'(\theta^*_j)}{\pi_j(\theta^*_j)} \right].
\end{align}
By triangle inequality, we have
\begin{align}\label{eq:trianglethetastar}
    - g_j(\theta^*) \frac{\pi_j'(\theta^*_j)}{\pi_j(\theta^*_j)} \geq - \theta^*_j \frac{\pi_j'(\theta^*_j)}{\pi_j(\theta^*_j)} - \left| (g_j(\theta^*) - \theta^*_j)  \frac{\pi_j'(\theta^*_j)}{\pi_j(\theta^*_j)} \right|.
\end{align}

Plugging in \eqref{eq:steinlemma} and using \eqref{eq:sortofvariance}, we get
\begin{align}
    \E_\pi \left[ \frac{\partial}{\partial \theta^*_j} g_j(\theta^*) \right] \geq 1 - \E_\pi \left[ \left| (g_j(\theta^*) - \theta^*_j)  \frac{\pi_j'(\theta^*_j)}{\pi_j(\theta^*_j)} \right| \right].
\end{align}

Summing over $j$ we get 
\begin{equation}\label{eq:bunchofCS}
    \begin{aligned}
    \sum_{j \in [d]} \E_\pi \left[ \frac{\partial}{\partial \theta^*_j} g_j(\theta^*) \right] &\geq d - \E_\pi \left[ \sum_{j \in [d]} \left| (g_j(\theta^*) - \theta^*_j)  \frac{\pi_j'(\theta^*_j)}{\pi_j(\theta^*_j)} \right| \right] \\
    &\geq d - \E_\pi \left[ \sqrt{ \sum_{j \in [d]} \left(g_j(\theta^*) - \theta^*_j \right)^2} \sqrt{ \sum_{j \in [d]} \left( \frac{\pi_j'(\theta^*_j)}{\pi_j(\theta^*_j)} \right)^2 } \right] \\
    &\geq d - \sqrt{ \E_\pi \left[ \sum_{j \in [d]} \left(g_j(\theta^*) - \theta^*_j \right)^2 \right]} \sqrt{ \E_\pi \left[ \sum_{j \in [d]} \left( \frac{\pi_j'(\theta^*_j)}{\pi_j(\theta^*_j)} \right)^2 \right] } \\
    & \geq d - \sqrt{ \E_\pi \left[ \norm{\E \left[ \mathcal M(D) \right] - \theta^* }_2^2 \right]} \sqrt{10 d^2} \\
    & \geq d \left( 1 - \sqrt{10 \, \E_{\pi, X, Y, \mathcal M} \left[ \norm{ \mathcal M(D) - \theta^* }_2^2 \right]} \right).
    \end{aligned}
\end{equation}
Plugging this inequality in \eqref{eq:fromntod} we get
\begin{equation}
    \E_\pi \left[ \sum_{i \in [n]} \E \left[ A_i (\mathcal M(D)) \right] \right] \geq \zeta^2 d \left( 1 - \sqrt{10 \, \E_{\pi, X, Y, \mathcal M} \left[ \norm{ \mathcal M(D) - \theta^* }_2^2 \right]} \right).
\end{equation}
Suppose $\E_{\pi, X, Y, \mathcal M} \left[ \norm{ \mathcal M(D) - \theta^* }_2^2 \right] \leq 1 / 40$. Then, we have $\E_\pi \left[ \sum_{i \in [n]} \E \left[ A_i (\mathcal M(D)) \right] \right] \geq \zeta^2 d / 2$. Thus, the desired result follows by taking $C = \min(\zeta^2 / 2, 1 / 40)$.
\end{proof}

\begin{lemma}\label{lemma:privateminimax}
    Let Assumption \ref{ass:data} hold, $\zeta^2$ be an absolute positive constant and $d^2 / (\rho^2 n^2 \lambda_{\min}) = O(1)$. 
    Define $\mathbb M$ to be the set of all $\rho^2 / 2$-zCDP algorithms such that $\sup_{\norm{\theta^*}_2 < 1} \E \left[ \norm{\mathcal M}_2^2 \right] < \infty$.
    Then, there exists a prior $\pi$ on $\theta^*$ with support in $\norm{\theta^*}_2 < 1$ such that 
    \begin{equation}
        \inf_{\mathcal M \in \mathbb M} \E_{\pi, \mathcal M, X, Y} \left[ \mathcal R(\mathcal M(X, Y)) \right] = \Omega \left( \frac{d^2}{n^2 \left( e^{\rho^2} - 1 \right)} \right).
    \end{equation}
\end{lemma}
\begin{proof}
    
Note that, according to \eqref{eq:tracingattack} and \eqref{eq:Li}, we have
\begin{align}
    \E \left[ A_i(\mathcal M(D'_i)) \right] = 0, \qquad  \E \left[ (L_i - 1) A_i(\mathcal M(D'_i)) \right] =  \E \left[ L_i A_i(\mathcal M(D'_i)) \right] = \E \left[ A_i(\mathcal M(D)) \right].
\end{align}
The first one holds because $z_i$, $x_i$ and $\left(\mathcal M(D_i') - \theta^* \right)$ are independent, and the first one is mean-0. The last equality follows from
\begin{equation}
    \begin{aligned}
    \E \left[ L_i A_i(\mathcal M(D'_i))  \right] &=  \int \frac{p_{\mathcal{M}(D)} (\theta)}{p_{\mathcal{M}(D'_i)} (\theta)} A_i(\theta) p_{\mathcal{M}(D'_i)} (\theta)  \diff \theta  \\
    &= \int A_i (\theta) p_{\mathcal{M}(D)} (\theta)  \diff \theta  \\
    &= \E \left[ A_i(\mathcal M(D))  \right].
    \end{aligned}
\end{equation}

Then, we have
\begin{equation}\label{eq:cslowerbound}
\begin{aligned}
    \sum_{i = 1}^n \E \left[ A_i(\mathcal M(D)) \right] &= \sum_{i = 1}^n \E \left[ (L_i - 1) A_i( \mathcal M(D'_i) )\right] \\
    &\le \sum_{i = 1}^n \sqrt{\E \left[ (L_i - 1)^2 \right]} \sqrt{\E \left[ A_i( \mathcal M(D'_i))^2 \right]} \\
    &\leq n \sqrt{\E \left[ (L_1 - 1)^2 \right]} \sqrt{\E \left[ A_1( \mathcal M(D'_1))^2 \right]},
\end{aligned}
\end{equation}
where we used Cauchy-Schwartz in the second line and the fact that the samples are equally distributed in the last line.

Since $\E[L_1] = 1$, we have
\begin{equation}\label{eq:boundsrhoandattack}
\begin{aligned}
    \E \left[ (L_1 - 1)^2 \right] &= \E \left[ L_1^2 \right] - 1 \leq \exp \left( \rho^2 \right) - 1, \\
    \E \left[ A_1( \mathcal M(D'_1))^2 \right] &= \E \left[ z_1^2 \left(x_1^\top \left(\mathcal M(D'_1) - \theta^* \right)\right)^2 \right] \\
    &= \zeta^2 \E \left[ \norm{\Sigma^{1/2} \left(\mathcal M(D'_1) - \theta^* \right)}_2^2 \right] \\
    &= 2 \zeta^2 \E \left[ \mathcal R \left( \mathcal M(D) \right) \right],
\end{aligned}
\end{equation}
where the steps in the third line are again motivated by the fact that $z_1$, $x_1$ and $\left(\mathcal M(D'_1) - \theta^* \right)$ are independent.


Let us now focus on the LHS of \eqref{eq:cslowerbound}. Suppose
\begin{equation}
    \E_{\pi, \mathcal M, X, Y} \left[ \norm{\mathcal M(D) - \theta^*}_2^2 \right] \leq \frac{2 \, \E_{\pi, \mathcal M, X, Y} \left[ \mathcal R( \mathcal M(D)) \right]}{\lambda_{\min}} < C_1.
\end{equation}
Setting $C_1$ to be the same constant as in Lemma \ref{lemma:omegad}, by \eqref{eq:twoconditions} we have that
\begin{align}
    \E_\pi \left[ \sum_{i \in [n]} \E \left[ A_i (\mathcal M(D)) \right] \right] \geq C_1 d.
\end{align}
This, together with \eqref{eq:boundsrhoandattack} and \eqref{eq:cslowerbound}, after taking the expectation with respect to $\pi$ yields
\begin{align}
    \E_{\pi, \mathcal M, X, Y} \left[ \mathcal R \left( \mathcal M(D) \right) \right] = \Omega \left( \frac{d^2}{n^2 \left( e^{\rho^2} - 1 \right)} \right),
\end{align}
which yields the desired result.

Suppose instead that
\begin{equation}
    \frac{2 \, \E_{\pi, \mathcal M, X, Y} \left[ \mathcal R( \mathcal M(D)) \right]}{\lambda_{\min}} \geq C_1.
\end{equation}
Then, by hypothesis, we would have
\begin{equation}
    \E_{\pi, \mathcal M, X, Y} \left[ \mathcal R( \mathcal M(D)) \right] \geq \frac{\lambda_{\min} C_1}{2} = \Omega \left( \frac{d^2}{n^2 \rho^2} \right) = \Omega \left( \frac{d^2}{n^2 \left( e^{\rho^2} - 1 \right)}  \right),
\end{equation}
which again yields the desired result.


\end{proof}



\begin{lemma}\label{lemma:nonprivateminimax}
    Let Assumption \ref{ass:data} hold, 
    $\zeta^2$ be an absolute positive constant, $d < n = O(d^2)$, $\lambda_{\min} > 0$ and $\lambda_{\max} = O(1)$. Define $\mathbb M$ to be the set of all algorithms such that $\sup_{\norm{\theta^*}_2 < 1} \E \left[ \norm{\mathcal M}_2^2 \right] < \infty$. 
    Then, we have
    \begin{equation}\label{eq:minmaxnonprivate}
    \inf_{\mathcal M \in \mathbb M} \sup_{\norm{\theta^*}_2 < 1} \E \left[ \mathcal R(\mathcal M(X, Y)) \right] = \Omega \left( \frac{d}{ n} \right).
    \end{equation}
\end{lemma}
\begin{proof}
Let us denote $\pi_B$ as the Gaussian distribution in $d$ dimensions $\mathcal N(0, \zeta^2 / (\lambda n) I)$, with $\lambda = 2 \lambda_{\max}$, truncated at $\norm{\theta^*}_2 \geq 1$. Then, we have
\begin{equation}\label{eq:maxtoexp}
    \inf_{\mathcal M \in \mathbb M} \sup_{\norm{\theta^*}_2 < 1} \E \left[ \mathcal R(\mathcal M(X, Y)) \right] \geq \inf_{\mathcal M \in \mathbb M} \E_{\theta^* \sim \pi_B} \E \left[ \mathcal R(\mathcal M(X, Y)) \right].
\end{equation}
Let $B$ be the closed unit ball, and let $\mathcal B$ be the ellipsoid $\Sigma^{1 / 2} B$. Given any algorithm $\mathcal M$, define the composition $\mathcal M_B =  \Sigma^{-1/2} \operatorname{Proj}_{\mathcal B} \Sigma^{1/2} \circ \mathcal M$. This implies that, for all $(X, Y)$, $\mathcal M_B(X, Y) \in B$. Furthermore, for all $(X, Y)$ and $\theta^* \in B$, we have 
\begin{equation}
\begin{aligned}
    \norm{\Sigma^{1/2}(\mathcal M_B(X, Y) - \theta^*)}_2^2 &= \norm{\operatorname{Proj}_{\mathcal B} \Sigma^{1/2} \mathcal M(X, Y) - \Sigma^{1/2} \theta^*}_2^2 \\
    &= \norm{\operatorname{Proj}_{\mathcal B} \Sigma^{1/2} \mathcal M(X, Y) - \operatorname{Proj}_{\mathcal B} \Sigma^{1/2}  \theta^* }_2^2 \\
    &\leq \norm{\Sigma^{1/2} \left( \mathcal M(X, Y) - \theta^* \right)}_2^2,
\end{aligned}
\end{equation}
where the second step holds since $\Sigma^{1/2} \theta^* \in \mathcal B$, and the last step holds since projections on closed convex sets are non-expansive. Thus, we have $\mathcal R(\mathcal M_B(X, Y)) \leq \mathcal R(\mathcal M(X, Y))$. Then, since the support of $\pi_B$ is included in $B$, it holds that
\begin{equation}\label{eq:project}
     \inf_{\mathcal M \in \mathbb M} \E_{\theta^* \sim \pi_B} \E \left[ \mathcal R(\mathcal M(X, Y)) \right] = \inf_{\mathcal M\in \mathbb M} \E_{\theta^* \sim \pi_B} \E \left[ \mathcal R(\mathcal M_B(X, Y)) \right],
\end{equation}
where the equality follows from the fact that the set of algorithms $\{ \mathcal M_B \, | \, \mathcal M \in \mathbb M \}$ is a subset of $\mathbb M$.

Let us denote with $\pi = \mathcal N(0, \zeta^2 / (\lambda n) I)$ the non-truncated Gaussian distribution and with $A$ the event that $\theta^* \sim \pi$ is such that $\norm{\theta^*}_2 < 1$. Then, dropping the notation $\mathbb M$ and the dependence of $\mathcal M$ from $(X, Y)$, one has 
\begin{equation}\label{eq:changeofprior}
\begin{aligned}
    \inf_{\mathcal M} \E_{\theta^* \sim \pi_B} \E \left[ \mathcal R(\mathcal M_B) \right] &= \frac{1}{\mathbb P(A)} \inf_{\mathcal M} \E_{\theta^* \sim \pi} \E \left[ \mathcal R(\mathcal M_B) \mathbf{1}_A \right] \\
    &= \frac{1}{\mathbb P(A)} \inf_{\mathcal M} \E_{\theta^* \sim \pi} \E \left[ \mathcal R(\mathcal M_B) (1 - \mathbf{1}_{A^c}) \right] \\
    &\geq \frac{1}{\mathbb P(A)} \inf_{\mathcal M} \E_{\theta^* \sim \pi} \E \left[ \mathcal R(\mathcal M_B) \right] - \frac{1}{\mathbb P(A)} \sup_{\mathcal M} \E_{\theta^* \sim \pi} \E \left[ \mathcal R(\mathcal M_B) \mathbf{1}_{A^c} \right].
\end{aligned}
\end{equation}

Note that 
\begin{equation}\label{eq:smallPcomple}
    \sup_{\mathcal M} \E_{\theta^* \sim \pi} \E \left[ \mathcal R(\mathcal M_B) \mathbf{1}_{A^c} \right] \leq \sup_{\mathcal M} \sqrt{ \E_{\theta^* \sim \pi} \E \left[ \mathcal R(\mathcal M_B)^2 \right]} \sqrt{\mathbb P(A^c)} \leq 2 \lambda_{\max} \sqrt{\mathbb P(A^c)},
\end{equation}
where the first step follows from Cauchy-Schwartz inequality and the second step holds since $\norm{\mathcal M_B}_2 \leq 1$. Furthermore, since $\lambda > 2$ and $n > d$, due to concentration of the norm of Gaussian vectors we have that $\mathbb P(A) \geq 1 - 2 e^{-c_1 d}$ for some absolute constant $c_1$. Then, we also have
\begin{equation}\label{eq:frommourta}
\begin{aligned}
    \inf_{\mathcal M} \E_{\theta^* \sim \pi} \E \left[ \mathcal R(\mathcal M_B) \right] &\geq \inf_{\mathcal M} \E_{\theta^* \sim \pi} \E \left[ \mathcal R(\mathcal M) \right] \\
    &= \frac{\zeta^2}{n} \E \left[ \tr \left( (X^\top X / n + \lambda I)^{-1} \Sigma \right) \right] \\
    &\geq \frac{\zeta^2}{n}  \tr \left( \left( \E \left[ \Sigma^{- 1/2}(X^\top X / n + \lambda I) \Sigma^{-1/2} \right] \right)^{-1} \right) \\
    &= \frac{\zeta^2}{n}  \tr \left( \left( I + \lambda \Sigma^{-1} \right)^{-1} \right) \\
    &\geq \frac{\zeta^2}{2 \lambda n} \tr \left(\Sigma \right) \\
    &= \frac{\zeta^2 d}{4 \lambda_{\max} n},
\end{aligned}
\end{equation}
where the second step follows from the closed form of the Bayes estimator under Gaussian prior (see, e.g., Eq. (63) in \cite{mourtada2022}), the third step follows from Jensen inequality since the mapping $A \mapsto \tr(A^{-1})$ is convex if $A$ is positive definite (see Eq. (1.33) in \cite{bhadia2007}), and the last two steps follow from the fact that $\lambda$ has been set to $2 \lambda_{\max}$ and $\tr(\Sigma) = d$.

Then, putting \eqref{eq:changeofprior}, \eqref{eq:smallPcomple} and \eqref{eq:frommourta} together, we get
\begin{equation}
    \inf_{\mathcal M} \E_{\theta^* \sim \pi_B} \E \left[ \mathcal R(\mathcal M_B) \right] \geq \frac{\zeta^2 d}{4 \lambda_{\max} n} - 4 \lambda_{\max} e^{-c_1 d / 2}.
\end{equation}
Since $n = O(d^2)$, merging with \eqref{eq:maxtoexp}, \eqref{eq:project} yields the desired result.
\end{proof}


\paragraph{Proof of Theorem \ref{thm:lowerbound}.}
Note that, if $\sup_{\norm{\theta^*}_2 < 1} \E \left[ \norm{\mathcal M}_2^2 \right]$ diverges, then $\sup_{\norm{\theta^*}_2 < 1} \E \left[ \mathcal R(\mathcal M(X, Y)) \right]$ diverges as well.
Thus, let us focus on the class of $\rho^2 / 2$-zCDP algorithms such that $\sup_{\norm{\theta^*}_2 < 1} \E \left[ \norm{\mathcal M}_2^2 \right] < \infty$. By Lemmas \ref{lemma:privateminimax} and \ref{lemma:nonprivateminimax}, and using the same argument as in \eqref{eq:maxtoexp}, we have
\begin{equation}\label{eq:minimaxiwithexp}
    \inf_{\mathcal M \in \mathbb M} \sup_{\norm{\theta^*}_2 < 1} \E \left[ \mathcal R(\mathcal M(X, Y)) \right] = \Omega \left( \frac{d}{n} + \frac{d^2}{(e^{\rho^2} - 1)n^2} \right).
\end{equation}
Notice that, if $\rho > 1$, 
\begin{equation}
    \frac{d^2}{(e^{\rho^2} - 1)n^2} = O \left( \frac{d}{n} \right), \qquad \frac{d^2}{\rho^2 n^2} = O \left( \frac{d}{n} \right),
\end{equation}
i.e., the first terms in the RHS of \eqref{eq:thesislowerboundapp} and \eqref{eq:minimaxiwithexp} dominate. On the other hand, if $\rho < 1$, we have
\begin{equation}
    \frac{d^2}{(e^{\rho^2} - 1)n^2} = \Theta \left( \frac{d^2}{\rho^2 n^2} \right).
\end{equation}
Thus, for any value of $\rho > 0$, we have that
\begin{equation}\label{eq:thesislowerboundapp}
    \inf_{\mathcal M \in \mathbb M} \sup_{\norm{\theta^*}_2 < 1} \E \left[ \mathcal R(\mathcal M(X, Y)) \right] = \Omega \left( \frac{d}{n} + \frac{d^2}{n^2 \rho^2} \right),
\end{equation}
which gives the desired result when taking sequentially the limits $d / n \to \gamma$ and $\gamma \to 0$. \qed

\section{Proofs for Section \ref{sec:scalinglaws}}\label{app:scalinglaws}


In this section, we always assume Assumptions \ref{ass:data}, \ref{ass:learning_rate_schedules} and \ref{ass:power_law_1} to hold. We consider the limit $\gamma \to 0$, and use the notation $f(\gamma) \asymp g(\gamma)$ to indicate two quantities such that $f(\gamma) = \Theta_\gamma(g(\gamma))$, where the notation $\Theta_\gamma$ does not track constants depending on $\phi$, $\psi$, $\zeta$, $\norm{\theta^*}_2^2$, and $\alpha$. Similarly, we will consider the notations $\gtrsim$ and $\lesssim$ for $\Omega_\gamma(\cdot)$ and $O_\gamma(\cdot)$ respectively.
Furthermore, unless differently specified, we will always consider schedules such that $\tilde \eta(t) \leq 2 / \gamma$ uniformly in $t$, so that $\bar \eta (t) = \tilde \eta(t)$.

We have
\begin{equation}\label{eq:GammaF}
\Gamma(t) = \int_0^t \tilde \eta(s) \mu_c(R(s)) \diff s, \qquad F(t) :=  \int_0^t \tilde \eta(s) \diff s,
\end{equation}
where we defined the shorthand $F(t)$. Then, \eqref{eq:implicitRt} reads
\begin{equation}\label{eq:Rtapp}
     R (t) = \cF(\Gamma(t)) + \int_0^t \left( \tilde \eta^2(s) \nu_c(R(s)) (R(s) + \zeta^2/2) \gamma \mathcal{K}(\Gamma(t)-\Gamma (s))  + 2 c^2 \tilde \sigma^2(s) \gamma^2 \mathcal J (\Gamma(t)-\Gamma (s)) \right) \diff s.
\end{equation}
Within this section, for convenience, we also introduce the shorthand $v = c \tilde \eta(0)$.

\begin{proposition}\label{prop:FK}
    Let $\cF(x)$, $\cK(x)$, $\mathcal J(x)$ be defined according to \eqref{eq:cFcK}. Consider the shorthands
    \begin{equation}
        \kappa_1 = 2 - \phi - \psi, \qquad  \kappa_2 = 3 - \phi, \qquad \kappa_3 = 2 - \phi,
    \end{equation}
    so that $\kappa_1 > 1$, $\kappa_2 > 2$, and $\kappa_3 > 1$ due to Assumption \ref{ass:power_law_1}.
    Then, we have that $\cF(x) \asymp x^{-\kappa_1}$, $\mathcal{K}(x) \asymp x^{-\kappa_2}$, and $\mathcal{J}(x) \asymp x^{-\kappa_3}$ uniformly for $x \geq 1$, and $\cF(x) \asymp \mathcal{K}(x) \asymp \mathcal J(x) \asymp 1$ uniformly for $x < 1$.
\end{proposition}
\begin{proof}
    Due to Assumption \ref{ass:power_law_1}, we have that, for $d \to \infty$,
    \begin{equation}
    \cF(x) \to \int p(\lambda) \lambda^{-\psi + 1} e^{- 2 \lambda x}\diff \lambda, \qquad \mathcal{K}(x) \to  \int p(\lambda) \lambda^2 e^{- 2 \lambda x}\diff \lambda, \qquad \mathcal{J}(x) \to  \int p(\lambda) \lambda e^{- 2 \lambda x}\diff \lambda.
    \end{equation}
    Then, for the second part of the statement, notice that $\cF(x)$, $\mathcal K(x)$ and $\mathcal J(x)$ are decreasing in $x$, and a direct computation for $x \in  \{0 , 1\}$ yields the desired result. For $x \geq 1$, the desired result is given by the same approach used in Lemma D.5 in \cite{collins2024high}.
\end{proof}




\begin{lemma}\label{lemma:approxforODElaw}
    Assume $c \lesssim 1$ and $\tilde \eta(t) \leq 2 / \gamma$ uniformly in $t \in [0, 1]$. Then, we have
    \begin{equation}\label{eq:boundsfromGammatocF}
    \begin{aligned}
        R (t) &\asymp \cF(\Gamma(t)) + \int_0^t \left(\tilde \eta^2(s) c^2 \gamma  \mathcal{K}(\Gamma(t)-\Gamma (s)) + c^2 \tilde \sigma^2(s) \gamma^2  \mathcal{J}(\Gamma(t)-\Gamma (s)) \right)\diff s \\
        &\gtrsim \cF(c F(t)) + \int_0^t \left(\tilde \eta^2(s) c^2 \gamma \mathcal{K} \left(c F(t) - c F(s) \right) + c^2 \tilde \sigma^2(s) \gamma^2  \mathcal{J} \left (c F(t)- c F(s) \right) \right)\diff s \\
        &\gtrsim \cF(c F(t)).
    \end{aligned}
    \end{equation}
    
    Furthermore, if $R(t)$ is uniformly upper bounded by a constant not dependent on $\gamma$, we have
    \begin{equation}\label{eq:exactfromGammatocF}
        R (t) \asymp \cF(c F(t)) + \int_0^t \left(\tilde \eta^2(s) c^2 \gamma \mathcal{K} \left(c F(t) - c F(s) \right) + c^2 \tilde \sigma^2(s) \gamma^2  \mathcal{J} \left (c F(t)- c F(s) \right) \right)\diff s.
    \end{equation}
    
\end{lemma}
\begin{proof}
By Lemma \ref{lemma:munubounds}, we have that $\mu_c(R) \lesssim c$. Then, multiplying both sides by $\tilde \eta(s)$ and integrating from $s$ to $t$, we get 
\begin{equation}\label{eq:GammatoF}
    \Gamma(t) - \Gamma(s) \lesssim c \left( F(t) - F(s) \right),
\end{equation}
and $\Gamma(t) \lesssim c F(t)$ when setting $s = 0$. Then, since $\cF$, $\mathcal K$ and $\mathcal J$ are all uniformly $\asymp$ to decreasing functions, we have
\begin{equation}\label{eq:ineqfromGammatoF}
\begin{aligned}
    \cF(\Gamma(t)) &\gtrsim \cF(c F(t)) \\
    \mathcal{K} \left ( \Gamma (t)- \Gamma(s) \right) &\gtrsim \mathcal{K} \left ( c F(t)- c F(s) \right) \\
    \mathcal{J} \left (\Gamma (t)- \Gamma(s)  \right) &\gtrsim \mathcal{J} \left (c F(t)- c F(s) \right).
\end{aligned}
\end{equation}
Furthermore, again due to Lemma \ref{lemma:munubounds}, since $c \lesssim 1$, we have $c^2 \asymp \nu_c(R(s)) (R(s) + \zeta^2/2)$. Thus, since all terms on the RHS of \eqref{eq:Rtapp} are positive, and dropping numerical constants, we get the first part of the thesis.

For the second part of the thesis, notice that Lemma \ref{lemma:munubounds} gives
\begin{equation}
    \frac{c}{\sqrt{\zeta^2 + 2 R}} \lesssim \mu_c(R) \lesssim c,
\end{equation}
which implies that, if $\sup_t R(t) \lesssim 1$, the relations $\lesssim$ in \eqref{eq:GammatoF} can be replaced by $\asymp$. Thus, the thesis follows from Proposition \ref{prop:FK}, since we have $\cF(x) \asymp \cF(ax)$ if $a \asymp 1$, with the same relation holding for $\mathcal K$ and $\mathcal J$.
\end{proof}



\begin{lemma}\label{lemma:systemcomparison}
    Assume $c \lesssim 1$ and $\tilde \eta(t) \leq 2 / \gamma$ uniformly in $t \in [0, 1]$. Then, if
    \begin{equation}\label{eq:Rprimestatement}
    \begin{aligned}
         \cF(c F(t)) + \int_0^t \left(\tilde \eta^2(s) c^2 \gamma \mathcal{K} \left(c F(t) - c F(s) \right) + c^2 \tilde \sigma^2(s) \gamma^2  \mathcal{J} \left (c F(t)- c F(s) \right) \right)\diff s \lesssim 1,
    \end{aligned}
    \end{equation}
    uniformly in $t \in [0, 1]$, we have that
    \begin{equation}
    \begin{aligned}
         R(t) \asymp \cF(c F(t)) + \int_0^t \left(\tilde \eta^2(s) c^2 \gamma \mathcal{K} \left(c F(t) - c F(s) \right) + c^2 \tilde \sigma^2(s) \gamma^2  \mathcal{J} \left (c F(t)- c F(s) \right) \right)\diff s.
    \end{aligned}
    \end{equation}
\end{lemma}
\begin{proof}
    Define the following system of ODEs
    \begin{equation}\label{eq:ODEiprime}
        \diff D_i'(t) = - \lambda_i \tilde {\eta}(t) c D'_i(t) \diff t + \lambda_i \tilde {\eta}^2(t) c^2 \gamma \diff t + c^2 \tilde \sigma^2(t) \gamma^2 \diff t,
    \end{equation}
    with $D'_i(0) = D_i(0) = d \left(\omega_i^\top \theta^* \right)^2 / 2$, and $R'(t) = \frac{1}{d} \sum_{i=1}^d \lambda_i D'_i(t)$. Then, following the same steps as in \eqref{eq:ODEi_sol} and \eqref{eq:implicitRt} we have that $R'(t)$ is equal to the LHS in \eqref{eq:Rprimestatement}.

    By hypothesis, there exists a large enough constant $C_1$ such that $R'(t) \leq C_1$ for all $t \in [0, 1]$. Then, consider the new system of ODEs
    \begin{equation}\label{eq:ODEiprimebar}
        \diff \underline D_i'(t) = - \lambda_i \tilde {\eta}(t) \sqrt{2 C_1 + \zeta^2} \frac{c}{\sqrt{2 \underline R'(t) + \zeta^2}} \underline D'_i(t) \diff t + \lambda_i \tilde {\eta}^2(t) c^2 \gamma \diff t + c^2 \tilde \sigma^2(t) \gamma^2 \diff t,
    \end{equation}
    with $\underline D'_i(0) = D'_i(0)$, and $\underline R'(t) = \frac{1}{d} \sum_{i=1}^d \lambda_i \underline D'_i(t)$. Then, by ODE comparison, we have $\underline R'(t) \leq R'(t)$ for all $t \in [0, 1]$. This follows from the same argument as in Theorem 1.3 in \cite{teschl2012}, since $R'(t)$ ($\underline R'(t)$) is non-decreasing in $D_i'(t)$ ($\underline D_i'(t)$) for all $i$.

    Due to Lemma \ref{lemma:munubounds}, since $c \lesssim 1$, we have that there exists a sufficiently large constant $C_2$ such that $\mu_c(R) C_2 \geq c / \sqrt{2 R + \zeta^2}$. Then, by ODE comparison, integrating as in \eqref{eq:ODEi_sol} and \eqref{eq:implicitRt}, and denoting $C_3 = C_2 \sqrt{2 C_1 + \zeta^2}$, we get
    \begin{equation}
    \begin{aligned}
        \underline R'(t) &\geq \cF(C_3 \Gamma(t)) + \int_0^t \left(\tilde \eta^2(s) c^2 \gamma \mathcal{K} \left(C_3 \Gamma(t) - C_3 \Gamma(s) \right) + c^2 \tilde \sigma^2(s) \gamma^2  \mathcal{J} \left (C_3 \Gamma(t)- C_3 \Gamma(s) \right) \right)\diff s \\
        &\asymp R(t),
    \end{aligned}
    \end{equation}
    where the last step holds since $c^2 \asymp \nu_c(R(s)) (R(s) + \zeta^2/2)$ (due to Lemma \ref{lemma:munubounds} since $c \lesssim 1$), and since $\cF(x) \asymp \cF(ax)$ if $a \asymp 1$ (with the same relation holding for $\mathcal K$ and $\mathcal J$) due to Proposition \ref{prop:FK}.

    Thus, we have
    \begin{equation}
        R(t) \gtrsim R'(t) \geq \underline R'(t) \gtrsim R(t),
    \end{equation}
    where the first inequality is a consequence of Lemma \ref{lemma:approxforODElaw}, which directly yields $R(t) \asymp R'(t)$, proving the thesis.
\end{proof}

\begin{lemma}\label{lemma:scalinglawpolyalpha}
    Let $\tilde \eta(t) = \tilde \eta(0) (1 - t)^\alpha$, with $\alpha \in \{0 , 1/2\}$ or $\alpha \geq 1$. Assume $c \lesssim 1$, $\tilde \eta(0) \leq 2 / \gamma$, $v = \tilde \eta(0) c = \omega_\gamma(1)$.
    Then, we have that 
    \begin{equation}\label{eq:riskwithgv}
        R(1) + \frac{2 c^2 \tilde \eta^2(1) \gamma^2}{\rho^2} \gtrsim v^{-\kappa_1} + \gamma v^{\frac{1}{\alpha + 1}} + \frac{\gamma^2}{\rho^2} g(v),
    \end{equation}
    where we introduced
    \begin{equation}\label{eq:casescostofprivacystatement}
      g(v) = \left\{\begin{array}{cc}
      \displaystyle  v^{\frac{2}{\alpha + 1}} & \mbox{if } \; \displaystyle  \phi < \frac{2}{\alpha + 1}, \vspace{0.3cm} \\
      \displaystyle  v^\phi & \mbox{if } \; \displaystyle  \phi > \frac{2}{\alpha + 1}, \vspace{0.3cm} \\
      \displaystyle  v^\phi \ln v & \mbox{if } \; \displaystyle  \phi = \frac{2}{\alpha + 1}.
    \end{array}\right.
    \end{equation}

    Furthermore, the relation $\gtrsim$ in \eqref{eq:riskwithgv} becomes $\asymp$ in the following cases: \emph{(i)} $\alpha = 0$, or \emph{(ii)} $\alpha = 1/2$ and $\gamma^2 v^{4 / 3} / \rho^2 \lesssim 1$, or \emph{(iii)} $\alpha \geq 1$ and $\gamma^2v / \rho^2 \lesssim 1$.
\end{lemma}
\begin{proof}
By Lemma \ref{lemma:approxforODElaw}, we have that
\begin{equation}\label{eq:Rtappnew}
    R (t) \gtrsim \cF(c F(t)) + \int_0^t \left(\tilde \eta^2(s) c^2 \gamma \mathcal{K} \left(c F(t) - c F(s) \right) + c^2 \tilde \sigma^2(s) \gamma^2  \mathcal{J} \left (c F(t)- c F(s) \right) \right)\diff s,
\end{equation}
where $\gtrsim$ can be replaced by $\asymp$ if the RHS is shown to be $\lesssim 1$ for all $t \in [0, 1]$, due to Lemma \ref{lemma:systemcomparison}.

Let us focus on the first term of the RHS of \eqref{eq:Rtappnew}. Due to \eqref{eq:polynomialsched}, \eqref{eq:GammaF} gives
\begin{equation}\label{eq:integralschedule}
    c F(t) = \frac{v}{(\alpha + 1)} (1 - (1 - t)^{\alpha + 1}).
\end{equation}
Then, due to Proposition \ref{prop:FK}, we have
\begin{equation}\label{eq:N1scaling}
    \cF(c F(1)) \asymp  v^{- \kappa_1},
\end{equation}
with $\cF(c F(t)) \lesssim 1$ for all $t \in [0, 1]$.

Let us now focus on the the second term of the RHS of \eqref{eq:Rtappnew}, i.e.,
\begin{equation}\label{eq:N1}
    N_2(t) :=  \gamma c^2 \int_{0}^t \tilde \eta^2(s) \cK (cF(t) - cF(s)) \diff s.
\end{equation}
Consider the change of variable $u = cF(t) - cF(s)$. This yields
\begin{equation}\label{eq:changeofvariable1}
    N_2(t) = \gamma c \int_{0}^{c F(t)} \tilde \eta(s(u)) \cK(u) \diff u.
\end{equation}
Furthermore, \eqref{eq:integralschedule} yields
\begin{equation}\label{eq:Cu}
    u = \frac{v}{\alpha + 1} \left( (1 - s)^{\alpha + 1} - (1 - t)^{\alpha + 1} \right), \qquad s = 1 - \left( \frac{u (\alpha + 1)}{v} + (1 - t)^{\alpha + 1} \right)^{\frac{1}{\alpha + 1}}.
\end{equation}
Then, plugging the last relation in \eqref{eq:polynomialsched} we get
\begin{equation}
    c \tilde \eta(s(u)) = v \left( \frac{u (\alpha + 1)}{v} + (1 - t)^{\alpha + 1} \right)^{\frac{\alpha}{\alpha + 1}}.
\end{equation}
Plugging this in \eqref{eq:changeofvariable1} we obtain
\begin{equation}
    N_2(t) = \gamma v^{\frac{1}{\alpha + 1}} \int_{0}^{c F(t)} \left(u (\alpha + 1) + v (1 - t)^{\alpha + 1} \right)^{\frac{\alpha}{\alpha + 1}} \cK(u) \diff u.
\end{equation}
Using Proposition \ref{prop:FK}, and the fact that $c F(1) \asymp v = \omega_\gamma(1)$, we have
\begin{equation}\label{eq:N2scaling}
\begin{aligned}
    N_2(1) &\asymp \gamma v^{\frac{1}{\alpha + 1}} \left( \int_{0}^{1} u^{\frac{\alpha}{\alpha + 1}} \cK(u) \diff u + \int_{1}^{c F(1)} u^{\frac{\alpha}{\alpha + 1}} \cK(u) \diff u \right) \\
    &\asymp \gamma v^{\frac{1}{\alpha + 1}} \left( \int_{0}^{1} u^{\frac{\alpha}{\alpha + 1}} \diff u + \int_{1}^{c F(1)} u^{\frac{\alpha}{\alpha + 1}} u^{- \kappa_2} \diff u \right) \\
    &\asymp \gamma v^{\frac{1}{\alpha + 1}},
\end{aligned}
\end{equation}
where the last line holds since $\frac{\alpha}{\alpha + 1} - \kappa_2 < -1$ for any $\alpha \geq 0$, since $\kappa_2 > 2$. Notice that, with a similar computation, we also have
\begin{equation}\label{eq:N2bound}
    N_2(t) \lesssim \gamma v \lesssim 1,
\end{equation}
uniformly in $t \in [0, 1]$, since $c \lesssim 1$ and $\tilde \eta(0) \leq 2 / \gamma$ by hypothesis.

Let us now focus on the third term in the RHS of \eqref{eq:Rtappnew} for $\alpha > 0$:
\begin{equation}
    N_3(t) := \gamma^2 c^2 \int_{0}^t \tilde \sigma^2(s) \mathcal J (c F(t) - c F(s)) \diff s,
\end{equation}
where we have
\begin{equation}
    \rho^2 \tilde \sigma^2(s) = 2 \alpha \tilde \eta(0)^2 (1 - s)^{2 \alpha - 1}.
\end{equation}
The change of variable $u = cF(t) - cF(s)$ yields
\begin{equation}\label{eq:changeofvariable1N3}
    N_3(t) = \gamma^2 c \int_{0}^{c F(t)} \frac{\tilde \sigma^2(s(u))}{\tilde \eta(s(u))} \mathcal J (u) \diff u \asymp  \frac{\gamma^2}{\rho^2} v \int_{0}^{c F(t)} (1 - s(u))^{\alpha - 1} \mathcal J (u) \diff u,
\end{equation}
which gives, plugging in \eqref{eq:Cu},
\begin{equation}\label{eq:N3scalingt}
    N_3(t) \asymp \frac{\gamma^2}{\rho^2} v \int_{0}^{c F(t)} \left( \frac{u}{v} + (1 - t)^{\alpha + 1} \right)^{\frac{\alpha - 1}{\alpha + 1}} \mathcal J (u) \diff u.
\end{equation}

Using Proposition \ref{prop:FK}, we have
\begin{equation}\label{eq:N3scaling}
\begin{aligned}
    N_3(1) &\asymp \frac{\gamma^2}{\rho^2} v^{\frac{2}{\alpha + 1}} \left( \int_{0}^{1} u^{\frac{\alpha - 1}{\alpha + 1}} \mathcal J(u) \diff u + \int_{1}^{c F(1)} u^{\frac{\alpha - 1}{\alpha + 1}} \mathcal J(u) \diff u \right) \\
    &\asymp \frac{\gamma^2}{\rho^2} v^{\frac{2}{\alpha + 1}} \left( \int_{0}^{1} u^{\frac{\alpha - 1}{\alpha + 1}} \diff u + \int_{1}^{c F(1)} u^{\frac{\alpha - 1}{\alpha + 1}} u^{- \kappa_3} \diff u \right).
\end{aligned}
\end{equation}
Notice that, since $c F(1) \asymp v = \omega_\gamma(1)$ and $\kappa_3 = 2 - \phi$, we have
\begin{equation}
    \int_{1}^{c F(1)} u^{\frac{\alpha - 1}{\alpha + 1}} u^{- \kappa_3} \diff u \asymp \left\{\begin{array}{cc}
      \displaystyle 1 & \mbox{if } \; \displaystyle \frac{\alpha - 1}{\alpha + 1} + \phi < 1, \vspace{0.3cm} \\
      \displaystyle v^{\frac{\alpha - 1}{\alpha + 1} + \phi - 1} & \mbox{if } \; \displaystyle \frac{\alpha - 1}{\alpha + 1} + \phi > 1, \vspace{0.3cm} \\
      \displaystyle \ln v & \mbox{if } \; \displaystyle \frac{\alpha - 1}{\alpha + 1} + \phi = 1.
    \end{array}\right.
\end{equation}
This yields, since $\int_{0}^{1} u^{\frac{\alpha - 1}{\alpha + 1}} \diff u \asymp 1$ for $\alpha > 0$,
\begin{equation}\label{eq:casescostofprivacy}
      N_3(1) \asymp \left\{\begin{array}{cc}
      \displaystyle \frac{\gamma^2}{\rho^2} v^{\frac{2}{\alpha + 1}} & \mbox{if } \; \displaystyle \frac{\alpha - 1}{\alpha + 1} + \phi < 1, \vspace{0.3cm} \\
      \displaystyle \frac{\gamma^2}{\rho^2} v^\phi & \mbox{if } \; \displaystyle \frac{\alpha - 1}{\alpha + 1} + \phi > 1, \vspace{0.3cm} \\
      \displaystyle \frac{\gamma^2}{\rho^2} v^{\phi} \ln v & \mbox{if } \; \displaystyle \frac{\alpha - 1}{\alpha + 1} + \phi = 1.
    \end{array}\right.
\end{equation}
Thus, if $\alpha \geq 1$, it is sufficient to have $\gamma^2 v / \rho^2 \lesssim 1$ to guarantee that $N_3(1) \lesssim 1$ (recall that $\phi < 1)$. Notice that, a direct computation from \eqref{eq:N3scalingt} shows that this condition is sufficient to guarantee $N_3(t) \lesssim 1$ uniformly in $t \in [0, 1]$. 

In the case $\alpha = 1/2$, only the first case in \eqref{eq:casescostofprivacy} is possible, and this gives the condition $\gamma^2 v^{4/3} / \rho^2 \lesssim 1$ to guarantee that $N_3(1) \lesssim 1$ (and this is also sufficient to guarantee $N_3(t) \lesssim 1$ uniformly in $t \in [0, 1]$).

Thus, the first part of the thesis follows from plugging \eqref{eq:N1scaling}, \eqref{eq:N2scaling} and \eqref{eq:casescostofprivacy} in \eqref{eq:Rtappnew}, since for $\alpha = 0$ we have $v = c \tilde \eta(1)$, and $\tilde \eta(1) = 0$ for $\alpha > 0$.

The second part of the thesis follows from the fact that the RHS of \eqref{eq:Rtappnew} is uniformly $\lesssim 1$ due to \eqref{eq:N1scaling} and \eqref{eq:N2bound} for \emph{(i)} $\alpha = 0$, or \emph{(ii)} $\alpha = 1/2$ and $\gamma^2 v^{4/3} / \rho^2 \lesssim 1$, or \emph{(iii)} $\alpha \geq 1$ and $\gamma^2 v / \rho^2 \lesssim 1$, due to the discussion following \eqref{eq:casescostofprivacy}.
\end{proof}

\begin{lemma}\label{lemma:optimalscalinglaws}
    Let $\tilde \eta(t) = \tilde \eta(0) (1 - t)^\alpha$, with $\alpha = \{0 , 1/2\}$ and $\alpha \geq 1$, such that $\tilde \eta(0) \leq 2 / \gamma$, and denote $v = \tilde \eta(0) c$.
    Assume $\rho \asymp \gamma^b$ with $b < 1$ independent of $\gamma$.

    Suppose $\phi (\alpha + 1) < 2$. Set $c \lesssim 1$ and $c \tilde \eta(0) = \gamma^a$, with
    \begin{equation}\label{eq:hyperparamscalinglawslowphilemma}
        a = \left\{\begin{array}{cc}
          \displaystyle - \frac{\alpha+1}{K+1} & \mbox{if } \; \displaystyle b\le\frac{K}{2 (K+1)}, \vspace{0.3cm} \\
          \displaystyle - \frac{2(1 - b)(\alpha + 1)}{K + 2} & \mbox{if } \; \displaystyle b > \frac{K}{2 (K+1)},
        \end{array}\right.
    \end{equation}
    where $K = (2 - \phi - \psi) (\alpha + 1)$. Then, we have that $R(1) + 2 c^2 \tilde \eta^2(1) \gamma^2 / \rho^2 \asymp \gamma^h$, with
    \begin{equation}\label{eq:scalinglawslowphilemma}
        h = \left\{\begin{array}{cc}
         \displaystyle \frac{K}{K+1}  & \mbox{if } \; \displaystyle b\le\frac{K}{2 (K+1)}, \vspace{0.3cm} \\
         \displaystyle \frac{2 K (1 - b)}{K + 2} & \mbox{if } \; \displaystyle b > \frac{K}{2 (K+1)}.
          \end{array}\right.
    \end{equation}

    Suppose $\phi (\alpha + 1) \geq 2$. Set $c \lesssim 1$, and $c \tilde \eta(0) = \gamma^a$. If
    \begin{equation}\label{eq:thresholdbhighphilemma}
         b \leq 1 - \frac{(\kappa_1 + \phi)(\alpha + 1)}{2 (K + 1)},
    \end{equation}
    then set $a$ as in the first case of \eqref{eq:hyperparamscalinglawslowphilemma} to achieve the same $h$ as in the first case of \eqref{eq:scalinglawslowphilemma}. If \eqref{eq:thresholdbhighphilemma} does not hold, then set
    \begin{equation}
        a = - \frac{2 (1 - b)}{\kappa_1 + \phi},
    \end{equation}
    which gives $R(1) + 2 c^2 \tilde \eta^2(1) \gamma^2 / \rho^2 \asymp \gamma^h$, with
    \begin{equation}
        h = \left\{\begin{array}{cc}
         \displaystyle \frac{2 \kappa_1 (1 - b)}{\kappa_1 + \phi} & \mbox{if } \; \displaystyle \phi (\alpha + 1) > 2, \vspace{0.3cm} \\
         \displaystyle \frac{2 \kappa_1 (1 - b)}{\kappa_1 + \phi} + \frac{\ln(a \ln \gamma)}{\ln \gamma} & \mbox{if } \; \displaystyle \phi (\alpha + 1) = 2.
          \end{array}\right.
    \end{equation}

    Furthermore, for each value of $\alpha$, we have $R(1) + 2 c^2 \tilde \eta^2(1) \gamma^2 / \rho^2 \gtrsim \gamma^h$ for all choices of $c$ and all choices of $a$ independent of $\gamma$.
\end{lemma}
\begin{proof}
Consider the case $c = \omega_\gamma(1)$. Due to Lemma \ref{lemma:munubounds} and the same argument used in \eqref{eq:munubigc}, we have that
\begin{equation}
\begin{aligned}
    R (t) &\gtrsim \cF(\Gamma(t)) + \int_0^t \left(\tilde \eta^2(s) \gamma \mathcal{K}(\Gamma(t) - \Gamma (s))  + c^2 \tilde \sigma^2(s) \gamma^2 \mathcal{J}(\Gamma(t) - \Gamma (s))  \right) \diff s \\
    &\gtrsim \cF(F(t)) + \int_0^t \left(\tilde \eta^2(s) \gamma \mathcal{K}(F(t) - F(s))  + c^2 \tilde \sigma^2(s) \gamma^2 \mathcal{J}(F(t) - F(s))  \right) \diff s \\
    &\gtrsim \cF(F(t)) + \int_0^t \left(\tilde \eta^2(s) \gamma \mathcal{K}(F(t) - F(s))  + \tilde \sigma^2(s) \gamma^2 \mathcal{J}(F(t) - F(s))  \right) \diff s,
\end{aligned}
\end{equation}
where the the second step holds for an argument similar to the one in the proof of Lemma \ref{lemma:approxforODElaw}, as $\mu_c(R) < 1$. Then, the risk at all times is lower bounded by an expression which is identical in form to the setting $c \asymp 1$, which implies that the smallest rates for $R(1) + 2 c^2 \tilde \eta^2(1) \gamma^2 / \rho^2$ can be achieved for $c \lesssim 1$. Furthermore, if $c \lesssim 1$ and $v \lesssim 1$, \eqref{eq:N1scaling} gives $R(1) \gtrsim 1$, which is a vacuous rate.

Then, let us focus on $c \lesssim 1$ and $v = \omega_\gamma(1)$. Then, Lemma \ref{lemma:scalinglawpolyalpha} gives the lower bound in \eqref{eq:riskwithgv}. During the proof, we will suppose this bound to be tight, and we will find the value of $a$ (with $v = \gamma^a$) that minimizes this lower bound. Later, we will verify that $a \in [-1, 0)$ (such that it is possible to consider $c \lesssim 1$ and $\tilde \eta(0) < 2 / \gamma$), $2 - 2b + 4a/3 \geq 0$ for $\alpha = 1/2$, and $2 - 2b + a \geq 0$ for $\alpha \geq 1$, guaranteeing that indeed the lower bound in \eqref{eq:riskwithgv} is tight and that the choice of $a$ is optimal.

\begin{itemize}
    \item Suppose $\phi (\alpha + 1) < 2$. Then, Lemma \ref{lemma:scalinglawpolyalpha} gives
    \begin{equation}\label{eq:firstcasees}
         R(1) + \frac{2 c^2 \tilde \eta^2(1) \gamma^2}{\rho^2} \gtrsim \gamma^{e_1} + \gamma^{e_2} + \gamma^{e_3},
    \end{equation}
    with
    \begin{equation}\label{eq:e1e2e3}
        e_1 = - a \kappa_1, \qquad e_2 = \frac{a}{\alpha + 1} + 1, \qquad e_3 = 2 - 2b + \frac{2 a}{\alpha + 1}.
    \end{equation}
    Then, the problem of hyper-parameter optimization reduces to finding $a$ 
    such that $\min (e_1, e_2, e_3)$ is maximized, giving the minimum lower bound on the risk.
    Consider the value of $a$ for which $e_1 = e_2$. This gives
    \begin{equation}\label{eq:optimalanonprivate}
        a = - \frac{\alpha + 1}{K + 1} \in [-1, 0], 
    \end{equation}
    where the inclusion follows from $\alpha \geq 0$ and $K > \alpha$ since $\kappa_1 > 1$. Note that, for this value of $a$, if $b \le \frac{K}{2 (K+1)}$, then we have
\begin{equation}\label{eq:optimalrisknonprivate}
        e_1 = e_2 = \frac{K}{K+1} \leq e_3.
    \end{equation}
    Then, since $e_1$ is monotonically decreasing in $a$ and $e_2$ and $e_3$ are increasing in $a$ (see Figures \ref{fig:maxmin-nonprivatecost}, \ref{fig:maxmin-nonprivatecostbigkappa} and \ref{fig:maxmin-nonprivatecostbigalpha} for a graphical presentation of this case), \eqref{eq:optimalanonprivate} gives the optimal choice of $a$ and \eqref{eq:optimalrisknonprivate} the corresponding lower bound on the risk. Furthermore, note that, in this case, we have
    \begin{equation}
        2 - 2b \geq \frac{K + 2}{K + 1} = 1 + \frac{1}{K + 1} = 1 - \frac{a}{\alpha + 1}.
    \end{equation}
    Then, for $\alpha \geq 1$ this gives $2 - 2b + a \geq 0$, and for $\alpha = 1/2$ this gives $2 - 2b + 4a/3 \geq 0$, which implies that \eqref{eq:firstcasees} is tight, due to the second part of the thesis of Lemma \ref{lemma:scalinglawpolyalpha}.
    
    If $b > \frac{K}{2 (K+1)}$, the value of $a$ in \eqref{eq:optimalanonprivate} is such that $e_1 = e_2 \geq e_3$. Then, set $e_1 = e_3$, which yields
    \begin{equation}\label{eq:optimalaprivate}
        a = - \frac{2(1 - b)(\alpha + 1)}{K + 2} \in [-1, 0],
    \end{equation}
    where the inclusion follows from $b \leq 1$ and $\kappa_1 > 1$. Note that, for this value of $a$, since $b > \frac{K}{2 (K+1)}$ we have
    \begin{equation}\label{eq:optimalriskprivate}
        e_1 = e_3 = \frac{2 K (1 - b)}{K + 2} \leq e_2.
    \end{equation}
    Then, since $e_1$ is monotonically decreasing in $a$ and $e_2$ and $e_3$ are increasing in $a$ (see Figure \ref{fig:maxmin-privatecost} 
    for a graphical presentation of this case), \eqref{eq:optimalaprivate} gives the optimal choice of $a$ and \eqref{eq:optimalriskprivate} the corresponding lower bound on the risk. 
    Furthermore, in this case, we have that
    \begin{equation}
        2 - 2b + a = \frac{2(1 - b) (K + 1 - \alpha) }{K + 2} > 0,
    \end{equation}
    where the last step holds since $K > \alpha$. If $\alpha = 1/2$, we have that $2 - 2b + 4a / 3 > 0$ since $3 \kappa_1 / 2 + 2 > 3 / 2$, which implies that \eqref{eq:firstcasees} is tight due to the second part of the thesis of Lemma \ref{lemma:scalinglawpolyalpha}. This concludes the argument for the case $\phi (\alpha + 1) < 2$.

\begin{figure}[t]
\centering
\begin{minipage}{0.48\textwidth}
  \centering
  \includegraphics[width=\linewidth]{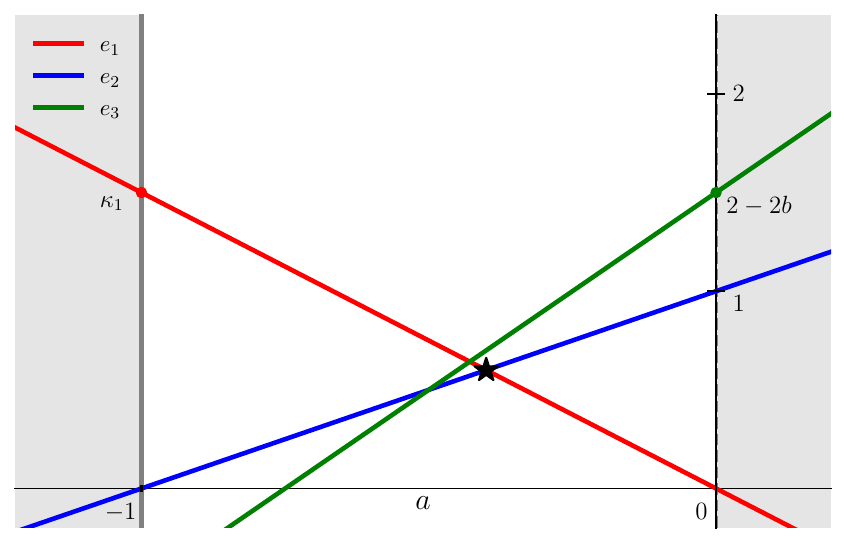}
  \caption{Plots of $e_1$, $e_2$ and $e_3$ as in \eqref{eq:e1e2e3} as a function of $a$ for $\kappa_1=1.5$, $\alpha=0$, $b=0.25$. Privacy comes for free, since $\argmax (\min(e_1, e_2, e_3)) = \argmax (\min(e_1, e_2))$.}
  \label{fig:maxmin-nonprivatecost}
\end{minipage}\hfill
\begin{minipage}{0.48\textwidth}
  \centering
  \includegraphics[width=\linewidth]{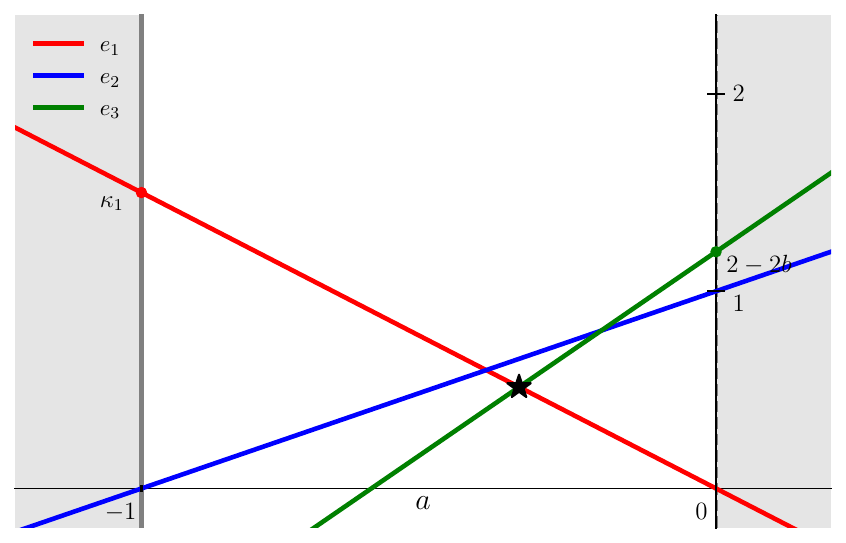}
  \caption{Plots of $e_1$, $e_2$ and $e_3$ as in \eqref{eq:e1e2e3} as a function of $a$ for $\kappa_1=1.5$, $\alpha=0$, $b=0.4$. Privacy comes with a performance cost, since $\argmax (\min(e_1, e_2, e_3)) = \argmax (\min(e_1, e_3))$.}
  \label{fig:maxmin-privatecost}
\end{minipage}
\end{figure}

    \item Suppose $\phi (\alpha + 1) > 2$. Then, Lemma \ref{lemma:scalinglawpolyalpha} gives that \eqref{eq:firstcasees} holds with
    \begin{equation}
        e_1 = - a \kappa_1, \qquad e_2 = \frac{a}{\alpha + 1} + 1, \qquad e_3 = 2 - 2b + a \phi.
    \end{equation}
    As before, set $a$ for which $e_1 = e_2$. This gives the result in \eqref{eq:optimalanonprivate}, and if
    \begin{equation}\label{eq:criticalbphi}
        b \leq 1 - \frac{(\kappa_1 + \phi)(\alpha + 1)}{2 (K + 1)},
    \end{equation}
    we have that \eqref{eq:optimalrisknonprivate} holds. Notice that the case $\phi(\alpha+1)>2$ cannot happen for either $\phi \leq 0$ or $\alpha \leq 1$, and for $\alpha > 1$, we have
    \begin{equation}
        2 - 2b + a \geq \frac{(\kappa_1 + \phi)(\alpha + 1)}{K + 1} + a = \frac{(\kappa_1 + \phi - 1)(\alpha + 1)}{K + 1} = \frac{(1 - \psi)(\alpha + 1)}{K + 1} > 0,
    \end{equation}
    where the last step holds since $1 - \psi > \phi$ by Assumption \ref{ass:power_law_1}. Then, this implies that \eqref{eq:firstcasees} is tight due to the second part of the thesis of Lemma \ref{lemma:scalinglawpolyalpha}.
    
    If \eqref{eq:criticalbphi} does not hold, the value of $a$ in \eqref{eq:optimalanonprivate} is such that $e_1 = e_2 \geq e_3$. Then, set $e_1 = e_3$, which yields
    \begin{equation}\label{eq:optimalaprivatephi}
        a = - \frac{2(1 - b)}{\kappa_1 + \phi} \in [-1, 0],
    \end{equation}
    where the inclusion follows from $b \leq 1$ and the lower bound on $b$. Notice that, for this value of $a$ and interval of $b$ we have
    \begin{equation}\label{eq:optimalriskprivatephi}
        e_1 = e_3 = \frac{2 \kappa_1 (1 - b)}{\kappa_1 + \phi} \leq e_2.
    \end{equation}
    Then, since $e_1$ is monotonically decreasing in $a$ and $e_2$ and $e_3$ are increasing in $a$,
    \eqref{eq:optimalaprivatephi} gives the optimal choice of $a$ and \eqref{eq:optimalriskprivatephi} the corresponding lower bound on the risk. 
    Furthermore, in this case, we have that
    \begin{equation}
        2 - 2b + a = \frac{2(1 - b) (\kappa_1 + \phi - 1) }{\kappa_1 + \phi} > 0,
    \end{equation}
    where the last step holds since $\phi > 0$, which implies that \eqref{eq:firstcasees} is tight due to the second part of the thesis of Lemma \ref{lemma:scalinglawpolyalpha}. This proves the thesis for the case $\phi (\alpha + 1) > 2$.

    \item Suppose $\phi (\alpha + 1) = 2$. Then, Lemma \ref{lemma:scalinglawpolyalpha} gives that \eqref{eq:firstcasees} holds with
    \begin{equation}
        e_1 = - a \kappa_1, \qquad e_2 = \frac{a}{\alpha + 1} + 1, \qquad e_3 = 2 - 2b + a \phi - \frac{\ln( | a| \ln (1/ \gamma))}{\ln (1 / \gamma)}.
    \end{equation}
    Since the last term on the RHS  of the last expression is $o_\gamma(1)$, the case given by \eqref{eq:criticalbphi} is identical to the previous one.
    
    When instead \eqref{eq:criticalbphi} does not hold, setting $a$ as in \eqref{eq:optimalaprivatephi}
    is also optimal as $\gamma \to 0$, since $\ln (1 / \gamma) = \omega_\gamma \left( \ln( | a | \ln( 1 / \gamma) \right)$. Thus, the desired result follows.
\end{itemize}
\end{proof}


\begin{figure}[t]
\centering
\begin{minipage}{0.48\textwidth}
  \centering
  \includegraphics[width=\linewidth]{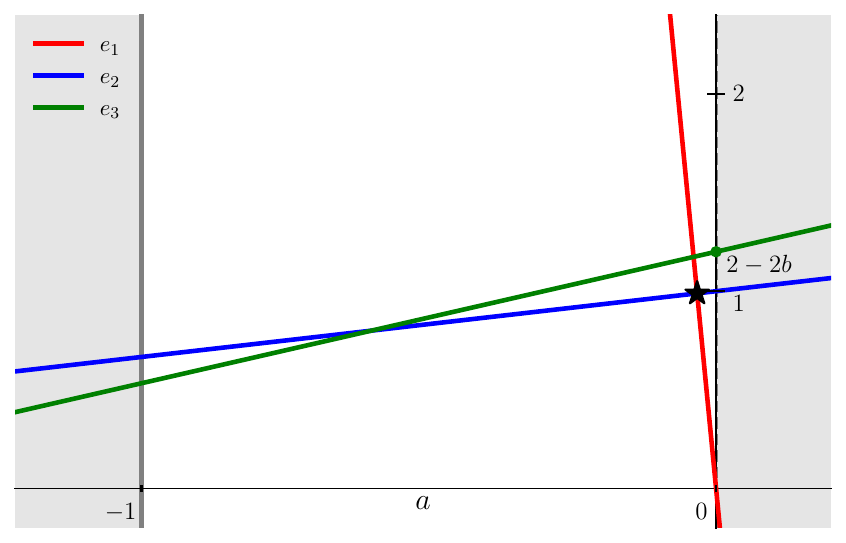}
  \caption{Plots of $e_1$, $e_2$ and $e_3$ as in \eqref{eq:e1e2e3} as a function of $a$ for $\kappa_1=30$, $\alpha=2$, $b=0.4$. Privacy comes for free, since $\argmax (\min(e_1, e_2, e_3)) = \argmax (\min(e_1, e_2))$.}
  \label{fig:maxmin-nonprivatecostbigkappa}
\end{minipage}\hfill
\begin{minipage}{0.48\textwidth}
  \centering
  \includegraphics[width=\linewidth]{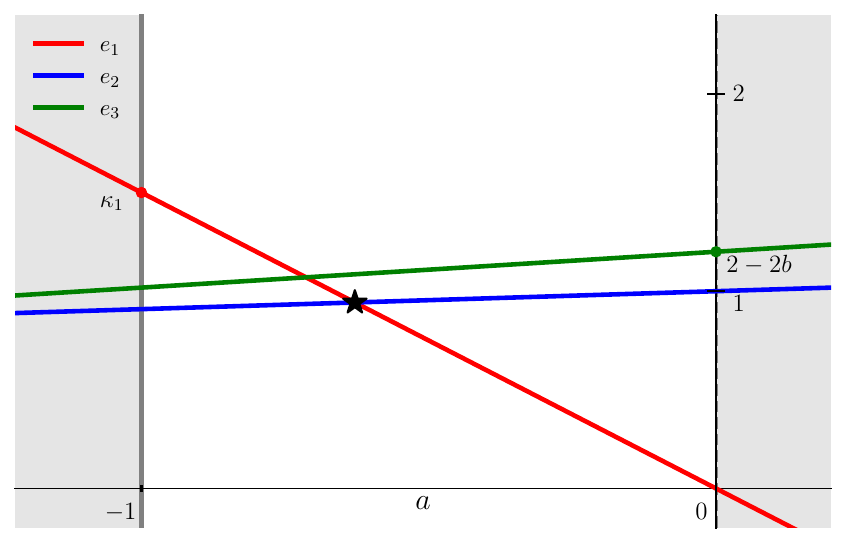}
  \caption{Plots of $e_1$, $e_2$ and $e_3$ as in \eqref{eq:e1e2e3} as a function of $a$ for $\kappa_1=1.5$, $\alpha=10$, $b=0.4$. Privacy comes for free, since $\argmax (\min(e_1, e_2, e_3)) = \argmax (\min(e_1, e_2))$.}
  \label{fig:maxmin-nonprivatecostbigalpha}
\end{minipage}
\end{figure}

\paragraph{Proof of Theorem \ref{thm:scalinglawsbody}.} The proof is a direct consequence of Lemma \ref{lemma:optimalscalinglaws} together with Theorem \ref{thm:deteq} and Proposition \ref{prop:laststeprisk}, after replacing $\kappa_1 = 2 - \phi - \psi$. \qed

\paragraph{Proof of Corollary \ref{cor:scalinglaws}.} If $\phi \leq 0$, for any $\alpha \geq 0$ we have that $\phi (\alpha + 1) \leq 0 < 2$. Then, taking $\alpha \to \infty$ (and consequently $K \to \infty$ as defined in Lemma \ref{lemma:optimalscalinglaws}), $h$ as defined in \eqref{eq:scalinglawslowphilemma} converges (from below) to $\min(1, 2 (1 - b))$, which gives the first part of the thesis.

If $\phi > 0$, it is still desirable to consider the limit $\alpha \to \infty$, as $h$ is non-decreasing in $\alpha$ (this can be seen from Lemma \ref{lemma:scalinglawpolyalpha}). Then, this setting reduces to the case $\phi (\alpha + 1) > 2$. As $\alpha\to\infty$, we have that the threshold in \eqref{eq:criticalbphi} approaches
\begin{equation}
    b^* = 1 - \frac{\kappa_1 + \phi}{2 \kappa_1} = \frac{1}{2} - \frac{\phi}{2 (2 - \phi - \psi)}
\end{equation}
from above. Furthermore, the exponent $h$ approaches 1 from below if $b \leq b^*$, and it is equal to $2 (1 - b) \frac{2 - \phi - \psi}{2 - \psi}$ otherwise, which concludes the proof. \qed

\end{document}